 \documentclass[final,3p,times]{elsarticle}

\usepackage{amsfonts}
\usepackage{graphicx}
\usepackage{times}
\usepackage{color}
\usepackage{amssymb}
\usepackage{natbib}
\usepackage{amsmath}
\usepackage{bm}
\usepackage{bbold}
\usepackage[bbgreekl]{mathbbol}
\usepackage{breqn}
\usepackage{algorithm}
\usepackage{algpseudocode}
\usepackage{accents}
\usepackage{multicol}
\usepackage{subcaption}

\font\ppppppcarac=ptmr8y at 4pt
\font\pppppcarac=ptmr8y at 5pt
\font\ppppcarac=ptmr8y at 6pt
\font\pppcarac=ptmr8y at 7pt
\font\ppcarac=ptmr8y at 8pt
\font\pcarac=ptmr8y at 9pt

\font\Ppcarac=ptmr8y at 10pt

\newcommand{\bfH}{{\bm{H}}}

\newcommand{\bfQ}{{\bm{Q}}}

\newcommand{\bfU}{{\bm{U}}}

\newcommand{\bfW}{{\bm{W}}}
\newcommand{\bfX}{{\bm{X}}}
\newcommand{\bfY}{{\bm{Y}}}
\newcommand{\bfZ}{{\bm{Z}}}

\newcommand{\bfzero}{{ \hbox{\bf 0} }}

\newcommand{\bfa}{{\bm{a}}}
\newcommand{\bfb}{{\bm{b}}}

\newcommand{\bff}{{\bm{f}}}
\newcommand{\bfg}{{\bm{g}}}
\newcommand{\bfh}{{\bm{h}}}

\newcommand{\bfq}{{\bm{q}}}

\newcommand{\bfu}{{\bm{u}}}
\newcommand{\bfw}{{\bm{w}}}
\newcommand{\bfx}{{\bm{x}}}
\newcommand{\bfy}{{\bm{y}}}
\newcommand{\bfz}{{\bm{z}}}

\newcommand{\bfGamma}{{\bm{\Gamma}}}

\newcommand{\bfeta}{{\bm{\eta}}}

\newcommand{\bflambda}{{\bm{\lambda}}}

\newcommand{\bfvarphi}{{\bm{\varphi}}}
\newcommand{\bfpsi}{{\bm{\psi}}}

\newcommand{\bfxi}{{\bm{\xi}}}

\newcommand{\bfphi}{{\bm{\phi}}}

\newcommand{\HH}{{\mathbb{H}}}

\newcommand{\MM}{{\mathbb{M}}}
\newcommand{\NN}{{\mathbb{N}}}
\newcommand{\PP}{{\mathbb{P}}}
\newcommand{\RR}{{\mathbb{R}}}

\newcommand{\WW}{{\mathbb{W}}}

\newcommand{\hh}{{\mathbb{h}}}

\DeclareMathAlphabet{\mathonebb}{U}{bbold}{m}{n}
\def\11{{\ensuremath{\mathonebb{1}}}}

\newcommand{\curB}{{\mathcal{B}}}
\newcommand{\curD}{{\mathcal{D}}}

\newcommand{\curK}{{\mathcal{K}}}
\newcommand{\curL}{{\mathcal{L}}}
\newcommand{\curM}{{\mathcal{M}}}
\newcommand{\curN}{{\mathcal{N}}}

\newcommand{\curP}{{\mathcal{P}}}

\newcommand{\curT}{{\mathcal{T}}}

\newcommand{\curV}{{\mathcal{V}}}

\newcommand{\bfcurN}{{\boldsymbol{\mathcal{N}}}}

\newcommand{\bfcurW}{{\boldsymbol{\mathcal{W}}}}

\newcommand{\bfcurY}{{\boldsymbol{\mathcal{Y}}}}
\newcommand{\bfcurZ}{{\boldsymbol{\mathcal{Z}}}}

\newcommand{\ar}{{\hbox{{\ppppcarac ar}}}}
\newcommand{\arp}{{\hbox{{\pppppcarac ar}}}}

\newcommand{\cov}{\hbox{{\Ppcarac cov}}}

\newcommand{\DM}{{\hbox{{\pppppcarac DM}}}}
\newcommand{\pDM}{{\hbox{{\ppppppcarac DM}}}}

\newcommand{\DB}{{\hbox{{\pppppcarac DB}}}}

\newcommand{\Div}{{\hbox{\pcarac div}}}

\newcommand{\err}{\hbox{{\Ppcarac err}}}
\newcommand{\errp}{\hbox{{\ppcarac err}}}

\newcommand{\FKP}{{\hbox{{\ppppcarac FKP}}}}

\newcommand{\lambdahat}{{  \hat{\kern -0.2em \lambda }}}

\newcommand{\normp}{{\hbox{{\pppppcarac norm}}}}

\newcommand{\MC}{{\hbox{{\pppppcarac MC}}}}
\newcommand{\pMC}{{\hbox{{\ppppppcarac MC}}}}

\newcommand{\MCH}{{\hbox{{\pppppcarac MCH}}}}
\newcommand{\pMCH}{{\hbox{{\ppppppcarac MCH}}}}
\newcommand{\ppMCH}{{\hbox{{\ppppppcarac MCH}}}}

\newcommand{\nDeltat}{{n\Delta t}}

\newcommand{\opt}{{\hbox{{\pppcarac opt}}}}
\newcommand{\optp}{{\hbox{{\ppppcarac opt}}}}
\newcommand{\optpp}{{\hbox{{\pppppcarac opt}}}}

\newcommand{\PCA}{{{\hbox{{\ppppcarac PCA}}}}}

\newcommand{\SB}{{\hbox{{\pppppcarac SB}}}}

\newcommand{\TB}{{\hbox{{\pppppcarac TB}}}}

\newcommand{\tr}{{\hbox{{\Ppcarac tr}}}}

\newcommand{\training}{{\hbox{{\pppcarac train}}}}

\newcommand{\wien}{{\hbox{{\pppcarac wien}}}}

\newcommand{\SV}{{\hbox{{\pppppcarac SV}}}}

\usepackage{mathrsfs}
\newcommand{\curC}{{\mathscr{C}}}       
\font\bf=ptmb8y at 10pt

\font\vcarac=ptmr8y at 10pt
\font\pvcarac=ptmr8y at 7pt
\newcommand{\bfv}{{\hbox{\bf{v}}}}      
\newcommand{\vc}{{\hbox{\vcarac{v}}}}   
\newcommand{\pvc}{{\hbox{\pvcarac{v}}}} 

\newdefinition{definition}{Definition}
\newtheorem{lemma}{Lemma}
\newtheorem{proposition}{Proposition}

\newproof{proof}{Proof}
\newdefinition{remark}{Remark}
\newdefinition{hypothesis}{Hypothesis}
\newdefinition{notation}{Notation}
\newproof{example}{Example}
\numberwithin{equation}{section}

\journal{ArXiv}

\begin{document}

\begin{frontmatter}

\title{Transient anisotropic kernel for probabilistic learning on manifolds}


\author[1]{Christian Soize \corref{cor1}}
\ead{christian.soize@univ-eiffel.fr}
\author[2]{Roger Ghanem}
\ead{ghanem@usc.edu}
\cortext[cor1]{Corresponding author: C. Soize, christian.soize@univ-eiffel.fr}
\address[1]{Universit\'e Gustave Eiffel, MSME UMR 8208, 5 bd Descartes, 77454 Marne-la-Vall\'ee, France}
\address[2]{University of Southern California, 210 KAP Hall, Los Angeles, CA 90089, United States}


\begin{abstract}
PLoM (Probabilistic Learning on Manifolds) is a method introduced in 2016 for handling small training datasets by projecting an Itô equation from a stochastic dissipative Hamiltonian dynamical system, acting as the MCMC generator, for which the KDE-estimated probability measure with the training dataset is the invariant measure. PLoM performs a projection on a reduced-order vector basis related to the training dataset, using the diffusion maps (DMAPS) basis constructed with a time-independent isotropic kernel. In this paper, we propose a new ISDE projection vector basis built from a transient anisotropic kernel, providing an alternative to the DMAPS basis to improve statistical surrogates for stochastic manifolds with heterogeneous data.  The construction ensures that for times near the initial time, the DMAPS basis coincides with the transient basis. For larger times, the differences between the two bases are characterized by the angle of their spanned vector subspaces. The optimal instant yielding the optimal transient basis is determined using an estimation of mutual information from Information Theory, which is normalized by the entropy estimation to account for the effects of the number of realizations used in the estimations.
Consequently, this new vector basis better represents statistical dependencies in the learned probability measure for any dimension. Three applications with varying levels of statistical complexity and data heterogeneity validate the proposed theory, showing that the transient anisotropic kernel improves the learned probability measure.
\end{abstract}


\end{frontmatter}

\section{Introduction}
\label{Section1}

\subsection{Objectives of the paper}
\label{Section1.1}
PLoM (Probabilistic Learning on Manifolds), introduced in 2016 \cite{Soize2016}, is a method and algorithm specifically developed for cases where the training dataset consists of a small number of data points. This method is based on projecting an Itô equation associated with a stochastic dissipative Hamiltonian dynamical system, which acts as the MCMC generator from the probability measure estimated using the KDE method applied to the points of the training dataset. The projection basis is the diffusion maps (DMAPS) basis associated with a time-independent isotropic kernel, introduced in \cite{Coifman2005, Coifman2006}.

Since 2016, all extensions and applications of PLoM (see Section~\ref{Section1.2}) have been carried out using the isotropic kernel. Through these applications, we have seen that the isotropic kernel allows for obtaining quality results, even for heterogeneous data and systems of great statistical complexity in small and large dimensions. However, improving the construction of statistical surrogates for stochastic manifolds involving conditional statistics and very heterogeneous data using PLoM based on a transient anisotropic kernel (time-dependent) was an analysis project. In this paper, we address this problem. We propose a new construction of the ISDE projection vector basis, built from a transient anisotropic kernel, which improves the representation of the statistical dependencies of the learned joint probability measure in any dimension.
\subsection{Framework of the considered problem}
\label{Section1.2}
Machine learning tools and artificial intelligence \cite{Korb2010, Murphy2012, Ghahramani2015, Russell2016}, such as probabilistic and statistical learning \cite{Vapnik2000, Hastie2009, James2013, Taylor2015, Swischuk2019}, are used in UQ for problems that would require computer resources not available with the most usual approaches. Thus, methods have emerged in the field of engineering sciences, such as learning on manifolds \cite{Oztireli2010, Soize2016, Perrin2017, Kevrekidis2020, Kontolati2022} and physics-informed probabilistic learning \cite{Pan2020, Soize2020a, Soize2021a}.

Probabilistic learning is a very active domain of research for constructing surrogate models (see for instance, \cite{Talwalkar2008, Oztireli2010, Marzouk2016, Parno2018, Perrin2018a, Kevrekidis2020, Pan2020}). Probabilistic Learning on Manifolds (PLoM) is a tool in computational statistics, introduced in 2016 \cite{Soize2016}, which can be viewed as a tool for scientific machine learning. The PLoM approach has specifically been developed for small dataset cases \cite{Soize2016, Soize2017a, Soize2019b, Soize2020b, Soize2020c}. The method avoids the scattering of learned realizations associated with the probability distribution to preserve its concentration in the neighborhood of the random manifold defined by the parameterized computational model.
his method allows for solving unsupervised and supervised problems under uncertainty when the training datasets are small. This situation is encountered in many problems in physics and engineering science with expensive function evaluations. The exploration of the admissible solution space in these situations is thus hampered by available computational resources.

Several extensions have been proposed to account for implicit constraints induced by physics, computational models, and measurements \cite{Soize2020a, Soize2021a, Soize2022b}, to reduce the stochastic dimension using a statistical partition approach \cite{Soize2022a}, and to update the prior probability distribution with a target dataset, whose points are, for instance, experimental realizations of the system observations \cite{Soize2022c}. Consequently, PLoM, constrained by a stochastic computational model and statistical moments or samples/realizations, allows for performing probabilistic learning inference and constructing predictive statistical surrogate models for large parameterized stochastic computational models.

This last capability of PLoM can also be viewed as an alternative method to Bayesian inference for high dimensions
\cite{Kennedy2001, Marzouk2007, Gentle2009, Stuart2010, Owhadi2015, Matthies2016, Dashti2017, Ghanem2017, Spantini2017, Perrin2020} and is a complementary approach to existing methods in machine learning for sampling distributions on manifolds under constraints. Although a Bayesian inference methodology has also been developed using probabilistic learning on manifolds for high dimensions \cite{Soize2020b}.

PLoM has successfully been adapted to tackle these challenges for several related problems, including nonconvex optimization under uncertainty \cite{Ghanem2018a, Soize2018a, Ghanem2018b, Farhat2019, Ghanem2019, Almeida2022, Capiez2022, Almeida2023}, fracture paths in random composites \cite{Guilleminot2020}, concurrent multiscale simulations in random media \cite{Chen2024}, stochastic homogenization in random elastic media \cite{Soize2023c}, ultrasonic transmission techniques in cortical bone microstructures \cite{Soize2020b}, updating digital twins under uncertainties \cite{Ghanem2022}, updating under-observed dynamical systems \cite{Ezvan2023, Capiez2024}, calculation of the Sobol indices \cite{Arnst2021}, dynamic monitoring \cite{Soize2021e}, surrogate modeling of structural seismic response \cite{Zhong2023}, probabilistic-learning-based stochastic surrogate models from small incomplete datasets \cite{Soize2023a, Soize2024}, and polynomial-chaos-based conditional statistics for probabilistic learning of atomic collisions \cite{Soize2023b}, as well as for aeroacoustic liner impedance metamodels from simulation and experimental data \cite{Sinha2023}.
\subsection{Methodology proposed and organization of the paper}
\label{Section1.3}
Starting with a training dataset of $n_d$ realizations of $\nu$ random variables, we consider the probability flow from each of these $n_d$ realizations towards the $\nu$-dimensional sampling probability distribution of the training dataset. We construct the associated coupled Fokker-Planck (FKP) equations with each of the realizations as initial condition. These describe the evolution of the transition probability measures over the graph described by the training dataset, that transport each independent realization, viewed as a concentrated measure at the initial time, into the joint probability measure, consisting of the common stationary probability measure  (also called the steady-state solution) of the FKP equations. This evolution describes a trajectory along which transition probabilities are consistent with both the training dataset and its postulated joint probability density function (that we approximate using a Gaussian Kernel Density Estimate (KDE)). Each of these probabilities provides a distinct geometric characterization of the training dataset, with its own plausible model of statistical dependence, resulting in a transient anisotropic kernel from which we construct a time evolving PLoM. We then set criteria for selecting among these PLoM models, which is tantamount to identifying the most appropriate
statistical dependence structure for the learned dataset, along the flow characterized by the FKP equations. The following
description of the paper’s organization provides a coherent summary of the proposed methodology.

In Section~\ref{Section2}, we define the training dataset constituted of given realizations of a non-Gaussian normalized vector-valued random variable $\bfH$ (centered and with an identity covariance matrix) and define the associated probability measure $P_\bfH$, whose density is estimated using the Gaussian Kernel Density Estimation (GKDE) method.

Section~\ref{Section3} deals with a short summary of formal results, introducing an It\^o stochastic differential equation (ISDE) and the derived Fokker-Planck (FPK) equation for which $P_\bfH$ is the invariant measure, which is the steady-state solution. We then introduce a formal formulation of the eigenvalue problem of the FPK operator and the nonstationary solution of the Fokker-Planck equation with a deterministic initial condition.

Section~\ref{Section4} is devoted to the time-dependent kernel $k_t$ derived from the time-dependent solution of the Fokker-Planck equation with a deterministic initial condition. We then define the time-dependent operator $K_t$ associated with the transient kernel $k_t$, which is a Hilbert-Schmidt operator. We propose a construction of its finite approximation, represented by a matrix $[\widehat{K}(t)]$, using a sampling of $P_\bfH$ of the bilinear form associated with $K_t$, and introduce the corresponding finite approximation of the eigenvalue problem.

In Section~\ref{Section5}, we present the direct construction of  matrix $[\widehat{K}(t)]$. We introduce the transition probability density function as the solution of the ISDE with a deterministic initial condition and study its existence, uniqueness, and properties. We then rewrite the ISDE in matrix form and propose a time-discrete approximation of this matrix-valued ISDE based on an Euler scheme. We introduce convergence criteria to verify convergence. Finally, we construct an explicit representation of the matrix $[\widehat{K}(t)]$ based on the nonstationary solution of the matrix-valued ISDE with a deterministic initial condition.

A numerical illustration of the proposed formulation is given in Section~\ref{Section6}, for which an explicit solution is known.

Section~\ref{Section7} deals with the construction and study of the vector basis for PLoM derived from the transient anisotropic kernel, connected to the DMAPS basis constructed with the isotropic kernel. We construct a time-dependent matrix $[\tilde{K}(\nDeltat)]$ of the transient anisotropic kernel derived from the matrix $[\hat{K}(\nDeltat)]$, whose fundamental property is its convergence to the matrix $[K_\DM]$ of the DMAPS isotropic kernel as time approaches zero. The time-dependent reduced-order transient basis is then the eigenvectors of $[\tilde{K}(\nDeltat)]$ associated with its dominant positive eigenvalues.
In order to qualify and quantify the gain of the constructed reduced-order transient basis, ROTB$(\nDeltat)$ at an instant sampled $\nDeltat$, with respect to the reduced-order DMAPS basis, RODB, we introduce the angle between the subspaces spanned by ROTB$(\nDeltat)$ and RODB. We propose a methodology for identifying the optimal instant sampled, which maximizes a selection criterion. This criterion is based on the estimation of mutual information from Information Theory, normalized using entropy estimation to account for the effects of the number of realizations used in the statistical estimators. Such a criterion allows for selecting the best joint learned probability measure with respect to statistical dependencies.

In Section~\ref{Section8}, we present three applications, each with a specific level of complexity and data heterogeneity in the training dataset. The first application is created such that the probability measure of $\bfH$ in $\RR^9$, defined by the points of the training dataset, is concentrated in a multiconnected domain of $\RR^9$. The constituent connected parts are manifolds of dimensions much lower than $9$, each with different dimensions. These parts may or may not be connected to each other. The training dataset of the second application consists of realizations of the random vector $\bfH$ with values in $\RR^8$, generated using a polynomial chaos expansion of degree $6$ of a real-valued random variable, whose random germ has a dimension of $2$, with each of the two random germs being a uniform random variable with different support. There are therefore $28$ terms in this expansion, and the $8$ components of $\bfH$ are defined as the random terms of given rank, defining a relatively complex random manifold in $\RR^8$. The third application results from a statistical treatment of an experimental database containing photon measurements in the ATLAS detector at CERN. The PCA step of PLoM has been performed, and $45$ components have been extracted to obtain the training dataset for the $\RR^{45}$-valued random variable $\bfH$. This application is in higher dimension than the first two but has less statistical complexity.

The paper is completed by two appendices. \ref{SectionA} presents an overview of the probabilistic learning on manifolds (PLoM) algorithm and its parameterization, using either the DMAPS basis (RODB) or the transient basis (ROTB$(\nDeltat)$) as the projection basis. \ref{SectionB} provides the formulas for estimating the Kullback-Leibler divergence, mutual information, and entropy from a set of realizations.

It should be noted that there are some very brief repetitions, which have been deliberately made to facilitate reading.
\subsection{Convention for the variables, vectors, and matrices}
\label{Section1.4}
\noindent $x,\eta$: lower-case Latin or Greek letters are deterministic real variables.\\
$\bfx,\bfeta$: boldface lower-case Latin or Greek letters are deterministic vectors.\\
$X$: upper-case Latin letters are real-valued random variables.\\
$\bfX$: boldface upper-case Latin letters are vector-valued random variables.\\
$[x]$: lower-case Latin letters between brackets are deterministic matrices.\\
$[\bfX]$: boldface upper-case letters between brackets are matrix-valued random variables.
\subsection{Algebraic notations}
\label{Section1.5}
\noindent $\NN, \NN^*$: set of natural numbers including $0$, excluding $0$.\\
$\RR, \RR^+, \RR^{+*}$: set of real numbers, subset $[0,+\infty[$, subset $]0, +\infty[$.\\
$\RR^\nu$: Euclidean vector space of dimension $\nu$.\\
\noindent $\MM_{n,m}$, $\MM_n$ : set of the $(n\times m)$, $(n\times n)$,  real matrices.\\
$\MM_n^+$, $\MM_n^{+0}$: set of  the positive-definite, positive,  $(n\times n)$ real matrices.\\
$[I_{n}]$: identity matrix in $\MM_n$.\\
$\Vert\, [x]\, \Vert_F$: Frobenius norm of matrix  $[x]$.\\
\noindent $\nabla$, $\Div$: gradient and divergence operators in $\RR^\nu$.\\
\noindent $\delta_0$: Dirac measure on $\RR^\nu$ at point $\bfzero_\nu$.\\
\subsection{Convention used for random variables}
\label{Section1.6}
In this paper, for any finite integer $m\geq 1$, the Euclidean space $\RR^m$ is equipped with the $\sigma$-algebra $\curB_{\RR^m}$. If $\bfY$ is a $\RR^m$-valued random variable defined on the probability space $(\Theta,\curT,\curP)$, $\bfY$ is a  mapping $\theta\mapsto\bfY(\theta)$ from $\Theta$ into $\RR^m$, measurable from $(\Theta,\curT)$ into $(\RR^m,\curB_{\RR^m})$, and $\bfY(\theta)$ is a realization (sample) of $\bfY$ for $\theta\in\Theta$. The probability distribution of $\bfY$ is the probability measure $P_\bfY(d\bfy)$ on the measurable set $(\RR^m,\curB_{\RR^m})$ (we will simply say on $\RR^m$). The Lebesgue measure on $\RR^m$ is noted $d\bfy$ and
when $P_\bfY(d\bfy)$ is written as $p_\bfY(\bfy)\, d\bfy$, $p_\bfY$ is the probability density function (pdf) on $\RR^m$ of $P_\bfY(d\bfy)$ with respect to $d\bfy$. Finally, $E$ denotes the mathematical expectation operator that is such that
$E\{\bfY\} = \int_{\RR^m} \bfy \, P_\bfY(d\bfy)$.
%
\section{Defining the probability measure $P_\bfH$ of random vector $\bfH$}
\label{Section2}
%
Let $\curD_\training(\bfeta) = \{\bfeta^j, j = 1,\ldots, n_d\}$ be the set of $n_d > 1$ independent realizations $\bfeta^j\in\RR^\nu$, with $\nu\geq 1$, of a second-order $\RR^\nu$-valued random variable defined on a probability space $(\Theta,\curT,\curP)$. Let $\underline\bfeta_d\in\RR^\nu$ and $[C_d]\in\MM_{n_d}$ be the associated empirical estimates of the mean value and the covariance matrix constructed with the points of $\curD_\training(\bfeta)$,
\begin{equation} \label{eq2.1}
\underline\bfeta_d = \frac{1}{n_d} \sum_{j=1}^{n_d} \bfeta^j \quad , \quad
[C_d] =  \frac{1}{n_d -1} \sum_{j=1}^{n_d} (\bfeta^j - \underline\bfeta_d )\otimes (\bfeta^j - \underline\bfeta_d )\, .
\end{equation}
It is assumed that $\curD_\training(\bfeta)$ is such that
\begin{equation} \label{eq2.2}
\underline\bfeta_d =  \bfzero_\nu \quad , \quad
[C_d] =  [I_\nu]\, .
\end{equation}
Let $\bfeta\mapsto p_\bfH(\bfeta)$ be the probability density function on $\RR^\nu$, with respect do the Lebesgue measure $d\bfeta$, defined by
\begin{equation} \label{eq2.3}
p_\bfH(\bfeta) = \frac{1}{n_d}\sum_{j=1}^{n_d} \frac{1}{(\sqrt{2 \pi} \, \hat s)^\nu} \exp\left ( -\frac{1}{2 \hat s^2} \Vert \bfeta - \frac{\hat s}{s}\bfeta^j\Vert^2\right ) \quad , \quad \forall \bfeta\in\RR^\nu \, ,
\end{equation}
where $\hat s$ and $s$ are defined by
\begin{equation} \label{eq2.4}
s = \left ( \frac{4}{n_d(2+\nu)}\right )^{1/(\nu+4)} \quad , \quad
\hat s  = \frac{s}{\sqrt{s^2 + (n_d-1)/n_d}} \, .
\end{equation}
Eqs.~\eqref{eq2.3} and \eqref{eq2.4} correspond to the Gaussian kernel-density estimation (KDE) constructed using the $n_d$ independent realizations of $\curD_\training(\bfeta)$ involving the modification \cite{Soize2015} of the usual formulation \cite{Bowman1997,Gentle2009,Givens2013}, in which $s$ is the Silverman bandwidth.
Let $\bfH$ be the second-order $\RR^\nu$-valued random variable, defined on a probability space $(\Theta,\curT,\curP)$, whose probability measure $P_\bfH(d\bfeta) = p_\bfH(\bfeta)\, d\bfeta$ on $\RR^\nu$ is defined by the probability density function $p_\bfH$ given by Eq.~\eqref{eq2.3}. It can be seen that, for any fixed $n_d > 1$, we have
\begin{equation} \label{eq2.5}
E\{\bfH\} = \int_{\RR^\nu} \bfeta \, P_\bfH(d\bfeta) = \frac{1}{2\hat s^2} \,\underline\bfeta_d = \bfzero_\nu \, ,
\end{equation}
\begin{equation} \label{eq2.6}
E\{\bfH\otimes\bfH\} = \int_{\RR^\nu} \bfeta\otimes\bfeta \, P_\bfH(d\bfeta)
= \hat s^2\, [I_\nu] + \frac{\hat s^2}{s^2} \, \frac{(n_d-1)}{n_d}\,[C_d] = [I_\nu] \, .
\end{equation}
Eqs.~\eqref{eq2.5} and \eqref{eq2.6} show that $\bfH$ is a normalized $\RR^\nu$-valued random variable. The probability density function $p_\bfH$ defined by Eq.~\eqref{eq2.3} is rewritten, for all $\bfeta$ in $\RR^\nu$, as
\begin{equation} \label{eq2.7}
p_\bfH(\bfeta) = c_\nu \, \xi(\bfeta)  \quad , \quad \xi(\bfeta) = e^{-\Phi(\bfeta)} \, ,
\end{equation}
in which  $c_\nu =  (\sqrt{2 \pi} \, \hat s)^{-\nu}$ and where $\Phi(\bfeta) = -\log(\xi(\bfeta))$ is such that
\begin{equation} \label{eq2.8}
\Phi(\bfeta) = - \log\left ( \frac{1}{n_d}\sum_{j=1}^{n_d} \exp\left ( -\frac{1}{2 \hat s^2} \Vert \bfeta - \frac{\hat s}{s}\bfeta^j\Vert^2\right ) \right ) \quad , \quad \forall \bfeta\in\RR^\nu \, .
\end{equation}
%
%
\section{Short summary of formal results}
\label{Section3}
%
This section is limited to a summary of essential results, which will be used in Section~\ref{Section4} and which are formally presented.

\subsection{It\^o stochastic differential equation related to $P_\bfH$}
\label{Section3.1}

We introduce an  It\^o stochastic differential equation (ISDE) on $\RR^\nu$, with initial condition, for which $P_\bfH(d\bfeta)$ is the invariant measure. A classical candidate to such an ISDE is written as
\begin{align}
d\bfY(t) = & \, \bfb(\bfY(t))\, dt + d\bfW(t) \quad , \quad t > 0 \, , \label{eq3.1}\\
\bfY(0) = & \, \bfx \in \RR^\nu \, , \, a.s. \, ,  \label{eq3.2}
\end{align}
where the drift vector is the function $\bfy\mapsto \bfb(\bfy)$ from $\RR^\nu$ into $\RR^\nu$ defined by
\begin{equation} \label{eq3.3}
\bfb(\bfy) = -\frac{1}{2}\, \nabla \Phi(\bfy) \quad , \quad \forall \bfy \in \RR^\nu \, .
\end{equation}
In Eq.~\eqref{eq3.1}, $\bfW(t) = (W_1(t),\ldots ,W_\nu(t))$ is the normalized Wiener stochastic process \cite{Doob1953} on $\RR^+$, with values in $\RR^\nu$, which is a stochastic process with independent increments, such that $\bfW(0) = \bfzero_\nu\,\ a.s $, and for  $0\leq \tau < t < +\infty$, the increment $\Delta\bfW_{\tau  t} = \bfW(t) - \bfW(\tau)$ is a Gaussian $\RR^\nu$-valued second-order random variable, centered and with a covariance matrix that is written as
\begin{equation} \label{eq3.4}
[C_{\Delta\bfW_{\tau t}}] = E\{\Delta\bfW_{\tau t}\otimes \Delta\bfW_{\tau t}  \} = (t-\tau)\, [I_\nu] \, .
\end{equation}
It should be noted that Eqs.~\eqref{eq3.1} and \eqref{eq3.2} is equivalent to
\begin{equation} \label{eq3.5}
\bfY(t) = \bfx + \int_0^t \bfb(\bfY(\tau))\, d\tau + \int_0^t d\bfW(\tau) \quad , \quad t \geq  0 \, .
\end{equation}
In Section~\ref{Section5}, we will see that $\{\bfY(t), t\in\RR^+\}$ is a homogeneous diffusion stochastic process, which is asymptotically stationary for $t\rightarrow +\infty$. Assuming that the transition probability measure of $\bfY(t)$ given $\bfY(0) = \bfx$ admits a density with respect to d$\bfy$, such that, for all $t > 0$, for all $\bfx$ and $\bfy$ in $\RR^\nu$, and for any Borelian $\curB$ in $\RR^\nu$, we have
\begin{align}
 \curP\, \{ \, \bfY(t) \in\curB \, \vert \, \bfY(0) = \bfx \, \} & =  \int_\curB \rho(\bfy,t \, \vert \, \bfx,0) \, d\bfy \, ,  \label{eq3.6}\\
 \lim_{t\, \rightarrow \,0_+}  \rho(\bfy,t \, \vert \, \bfx,0) \, d\bfy &  = \delta_0(\bfy - \bfx)  \, ,  \label{eq3.7}\\
\int_{\RR^\nu} \rho(\bfy,t \, \vert \, \bfx,0) \, d\bfy  & = 1 \, .  \label{eq3.7bis}
\end{align}
%

\subsection{FKP equation associated with the ISDE}
\label{Section3.2}

For all $\bfx$ in $\RR^\nu$, the transition probability density function $(\bfy,t)\mapsto \rho(\bfy,t \, \vert \, \bfx,0)$ from
$\RR^\nu\times\RR^+$ into $\RR^+$ verifies the following Fokker-Planck (FKP) equation (see for instance \cite{Guikhman1980,Friedman2006,Soize1994}),
\begin{equation} \label{eq3.8}
\frac{\partial\rho}{\partial t} + L_\FKP(\rho) = 0 \quad , \quad t > 0 \, ,
\end{equation}
with the initial condition for $t=0$ defined by Eq.~\eqref{eq3.7}. The Fokker-Planck operator $L_\FKP$ can be written, after a small algebraic manipulation and for any sufficiently differentiable function $\bfy\mapsto \vc(\bfy)$ from $\RR^\nu$ into $\RR$, as
\begin{equation} \label{eq3.9}
 \{L_\FKP(\vc)\}(\bfy) = -\frac{1}{2} \Div \left \{ p_\bfH(\bfy) \nabla \left ( \frac{\vc(\bfy)}{p_\bfH(\bfy)}\right ) \right\}  \, .
\end{equation}
The detailed balance (the probability current vanishes) is satisfied and the steady state solution of Eq.~\eqref{eq3.8} is the pdf $p_\bfH$ defined by Eq.~\eqref{eq2.3} \cite{Soize1994,Gardiner1985,Risken1989}. We then have, for $\vc = p_\bfH$,
\begin{equation} \label{eq3.10}
 L_\FKP(p_\bfH) = 0 \, .
\end{equation}
%

\subsection{Return to the invariant measure $P_\bfH$}
\label{Section3.3}
The invariant measure (see Proposition~\ref{proposition:5}) $P_\bfH(d\bfy) = p_\bfH(\bfy)\, d\bfy$ is such that, for all $\bfy$ in $\RR^\nu$ and for all $t\geq 0$,
\begin{equation} \label{eq3.10bis}
 p_\bfH(\bfy) = \int_{\RR^\nu} \rho(\bfy,t \, \vert \, \bfx,0) \, p_\bfH(\bfx) \, d\bfx \, .
\end{equation}
The ISDE defined by Eqs.~\eqref{eq3.1} and \eqref{eq3.2} admits an asymptotic ($t\rightarrow +\infty$) stationary solution whose marginal probability density function of order one is $p_\bfH$. Consequently, for all $\bfx$ and $\bfy$ in $\RR^\nu$, we have
\begin{equation} \label{eq3.11}
\lim_{t\, \rightarrow \,+\infty}  \rho(\bfy,t \, \vert \, \bfx,0)  = p_\bfH(\bfy) \, .
\end{equation}
%

\subsection{Formal formulation of the eigenvalue problem of the FKP operator}
\label{Section3.4}

The eigenvalue problem, posed in an adapted functional space, is written as
\begin{equation} \label{eq3.12}
 L_\FKP(\vc) =  \lambda\, \vc \, ,
\end{equation}
for which the current must vanish at infinity, yielding the condition,
\begin{equation} \label{eq3.13}
\lim_{\Vert\bfy\Vert \rightarrow +\infty}  p_\bfH(\bfy) \,\Vert\,  \nabla  ( p_\bfH(\bfy)^{-1}\vc(\bfy)   ) \, \Vert \, = 0\, .
\end{equation}
Continuing the development within  a formal framework, such as that used in \cite{Risken1989}, we introduce the change of function,
\begin{equation} \label{eq3.14}
\vc(\bfy) = p_\bfH(\bfy)^{1/2}\, q(\bfy) \quad , \quad \bfy\in\RR^\nu \quad , \quad q: \RR^\nu\rightarrow \RR \, .
\end{equation}
Let $\hat L_\FKP$ be the linear operator defined, for $q$ belonging to an admissible set of functions,
\begin{equation} \label{eq3.15}
\{\hat L_\FKP(q)\}(\bfy) =  p_\bfH(\bfy)^{-1/2}\,L_\FKP(p_\bfH(\bfy)^{1/2}\, q(y)) \quad , \quad \bfy\in\RR^\nu \, .
\end{equation}
Therefore, the eigenvalue problem defined by Eqs.~\eqref{eq3.12} and \eqref{eq3.13} can  be rewritten in $q$ as
\begin{equation} \label{eq3.16}
 \hat L_\FKP(q) =  \lambda\, q \, ,
\end{equation}
with the condition at infinity,
\begin{equation} \label{eq3.17}
\lim_{\Vert\bfy\Vert \rightarrow +\infty} p_\bfH(\bfy)^{1/2} \, \Vert \, \nabla  ( p_\bfH(\bfy)^{-1/2} q(\bfy)   ) \, \Vert \, = 0\, .
\end{equation}
%

%
\begin{remark} [Another algebraic representation of operator $\hat L_\FKP$] \label{remark:1}
Using Eq.~\eqref{eq2.7}, which shows that $p_\bfH(\bfy)^{-1} \nabla p_\bfH(\bfy) = - \nabla\Phi$, and using
 Eqs.~\eqref{eq3.9} and \eqref{eq3.15}, it can be seen that
\begin{equation} \label{eq3.18}
\{\hat L_\FKP(q)\}(\bfy) =  \curV(\bfy)\, q(\bfy) - \frac{1}{2} \nabla^2 q(\bfy) \quad , \quad \bfy\in\RR^\nu \, ,
\end{equation}
in which $\nabla^2$ is the Laplacian operator in $\RR^\nu$ and where $\bfy\mapsto\curV(\bfy)$ is the function from $\RR^\nu$ into $\RR$, which is defined, for all $\bfy$ in $\RR^\nu$, as
\begin{equation} \label{eq3.19}
 \curV(\bfy) =  \frac{1}{8} \Vert \,\nabla\Phi(\bfy) \,  \Vert^2  -  \frac{1}{4} \nabla^2 \Phi(\bfy) \, .
\end{equation}
\end{remark}
%
%
%
\subsection{Properties of operator $\hat L_\FKP$}
\label{Section3.5}

Let $\delta q$ be a function from $\RR^\nu$ into $\RR$, belonging to the admissible set that allows the evaluation of the bracket
\begin{equation} \nonumber
\langle \hat L_\FKP (q) \, , \delta q\rangle  = \int_{\RR^\nu} \{\hat L_\FKP(q)\}(\bfy) \,\, \delta q(\bfy) \, d\bfy \, .
\end{equation}
Removing $\bfy$ and using Eq.~\eqref{eq3.15} with Eq.~\eqref{eq3.9} yields
\begin{equation} \label{eq3.21}
\langle \hat L_\FKP (q) \, , \delta q\rangle  = -\frac{1}{2} \int_{\RR^\nu}
(p_\bfH^{-1/2} \, \delta q) \, \Div \{ p_\bfH \, \nabla ( p_\bfH^{-1/2}\, q)\}\, d\bfy \, .
\end{equation}
Using the condition at infinity, defined by Eq.~\eqref{eq3.17}, Eq.~\eqref{eq3.21} can be rewritten as,
\begin{equation} \label{eq3.22}
\langle \hat L_\FKP (q) \, , \delta q\rangle  = \frac{1}{2} \int_{\RR^\nu}
p_\bfH\, \langle \nabla(p_\bfH^{-1/2} \, q) \, , \nabla( p_\bfH^{-1/2} \, \delta q) \rangle_{\RR^\nu} \, d\bfy \, .
\end{equation}

\noindent (a) Eq.~\eqref{eq3.22} shows that $\hat L_\FKP$ is a symmetric and positive operator.

\noindent (b) Eqs.~\eqref{eq3.10} and \eqref{eq3.15} show that
\begin{equation} \label{eq3.23}
\hat L_\FKP (q_0) = 0 \quad \hbox{for} \quad q_0 = p_\bfH^{1/2} \, .
\end{equation}
In Proposition~\ref{proposition:5}, it will be proven that the ISDE defined by Eq.~\eqref{eq3.1}, with the initial condition defined by  \eqref{eq3.2},  has a unique solution and a unique invariant measure $p_\bfH(\bfy)\, d\bfy$. Consequently the dimension of the null space of operator $L_\FKP$ is $1$. Since $p_\bfH(\bfy)\, d\bfy$ is a bounded positive measure (probability measure), the right-hand side of Eq.~\eqref{eq3.22} shows that the null space  of $\hat L_\FKP$, which is also of dimension  $1$, is constituted of the function $q_0 = p_\bfH^{1/2}$. For $q=\delta q \not = q_0$, and $\Vert q_0\Vert \not = 0$, we have $\langle \hat L_\FKP (q) \, ,  q\rangle\, > 0$. Therefore, $\hat L_\FKP$ is a positive operator (in the quotient space by the null space).

%
\begin{hypothesis}[On the spectrum of operator $\hat L_\FKP$] \label{hypothesis:1}
It is assumed that $p_\bfH$ defined by Eq.~\eqref{eq2.3}, which is constructed with the $n_d$ points $\{\bfeta^j, j=1,\ldots, n_d\}$ of the training dataset, is such that the spectrum of $\hat L_\FKP$ is countable. Due to (a) and (b), we then deduce that the eigenvalues of $\hat L_\FKP$ (defined by Eqs.~\eqref{eq3.16} and \eqref{eq3.17}) are positive except one that is zero. We will also assume that the multiplicity of each eigenvalue is finite.
\end{hypothesis}

\subsection{Eigenvalue problem for operator $\hat L_\FKP$}
\label{Section3.6}

Under Hypothesis~\ref{hypothesis:1}, the eigenvalue problem $\hat L_\FKP(q_\alpha) = \lambda_\alpha q_\alpha$ for operator
$\hat L_\FKP$, with the condition defined by Eq.~\eqref{eq3.17}, is such that
\begin{equation} \label{eq3.24}
 0 =\lambda_0 < \lambda_1 \leq \lambda_2 \leq \ldots \, ,
\end{equation}
the multiplicity of each eigenvalue being finite.
We will admit that the family $\{q_\alpha,\alpha\in\NN\}$ of the eigenfunctions is a Hilbert basis of $L^2(\RR^\nu)$. We then have
\begin{equation} \label{eq3.26}
  \langle q_\alpha , q_\beta \rangle_{L^2} = \int_{\RR^\nu} q_\alpha(\bfy) \, q_\beta(\bfy)\, d\bfy = \delta_{\alpha\beta} \, .
\end{equation}
The eigenfunction $q_0$ associated with $\lambda_0=0$, is such that (see Eq.~\eqref{eq3.23}),
\begin{equation} \label{eq3.27}
  q_0 = p_\bfH^{1/2} \quad ,\quad \Vert\, q_0\Vert_{L^2} = 1 \, ,
\end{equation}
and we have
\begin{equation} \label{eq3.28}
  \sum_{\alpha\in\NN} q_\alpha(\bfy) \, q_\alpha(\bfx)\, d\bfy = \delta_0(\bfy-\bfx) \, .
\end{equation}
From Eqs.~\eqref{eq3.26} and \eqref{eq3.27}, it can be deduced that
\begin{equation} \label{eq3.29}
  \forall \alpha \geq 1 \quad , \quad  \int_{\RR^\nu} p_\bfH(\bfy)^{1/2} \, q_\alpha(\bfy) \, d\bfy = 0 \, .
\end{equation}

%
\subsection{Nonstationary solution of the Fokker-Planck equation with initial condition}
\label{Section3.7}

The transition probability density function $\rho(\bfy,t\, \vert \, \bfx,0)$ introduced in Section~\ref{Section3.1} and satisfying Eq.~\eqref{eq3.8} with the initial condition defined by Eq.~\eqref{eq3.7}, can  be written, using the Hilbert basis $\{q_\alpha,\alpha\in\NN\}$ defined in Section~\ref{Section3.6}, as
\begin{equation}\label{eq3.30}
 \rho(\bfy,t\, \vert \, \bfx,0) = p_\bfH(\bfy)^{1/2} \, p_\bfH(\bfx)^{-1/2} \, \sum_{\alpha\in\NN} e^{-\lambda_\alpha t} q_\alpha(\bfy) \, q_\alpha(\bfx) \quad , \quad t > 0\, .
\end{equation}
This representation of $\rho(\bfy,t\, \vert \, \bfx,0)$, defined by Eq.~\eqref{eq3.30}, actually satisfied all the required properties:

Eqs.~\eqref{eq3.24}, \eqref{eq3.27}, and  \eqref{eq3.29} yield $\int_{\RR^\nu} \rho(\bfy,t\, \vert \, \bfx,0) \, d\bfy = 1$.

Eqs.~\eqref{eq3.24} and  \eqref{eq3.27} yield $\lim_{t\rightarrow +\infty}\rho(\bfy,t\, \vert \, \bfx,0) \, d\bfy = p_\bfH(\bfy)$.

Eq.~\eqref{eq3.28} yields $\lim_{t\rightarrow 0_+}\rho(\bfy,t\, \vert \, \bfx,0) \, d\bfy = \delta_0(\bfy-\bfx)$.

Eqs.~\eqref{eq3.27} and  \eqref{eq3.29} yield $\int_{\RR^\nu} \rho(\bfy,t\, \vert \, \bfx,0) \,p_\bfH(\bfx)\,  d\bfx = p_\bfH(\bfy)$ that is Eq.~\eqref{eq3.10bis}.

\section{Time-dependent kernel, its associated operator, and finite approximation}
\label{Section4}
%
In this section,  we define the kernel $k_t(\bfy,\bfx)$ and give its basic properties directly deduced from the properties of $\rho(\bfy,t\, \vert \, \bfx,0)$ and $p_\bfH(\bfy)$, without using the spectral representation defined by Eq.~\eqref{eq3.30}. From the spectral representation of $k_t(\bfy,\bfx)$, we deduce its spectral representation using the spectral representation of  $\rho(\bfy,t\, \vert \, \bfx,0)$, defined by Eq.~\eqref{eq3.30}. Finally, we define the linear operator $K_t$ associated with kernel $k_t$ and we give the spectral representation of operator $K_t$.
%
\subsection{Definition of the kernel $k_t$ and its basic probabilistic properties}
\label{Section4.1}
The kernel $k_t(\bfy,\bfx)$ associated with the transition probability density function $\rho(\bfy,t\, \vert \, \bfx,0)$ is defined as follows.

\begin{definition}[Kernel $k_t$ on $\RR^\nu\times\RR^\nu$] \label{definition:1}
For every fixed $t > 0$, the kernel function $(\bfy,\bfx)\mapsto k_t(\bfy,\bfx)$, from $\RR^\nu\times \RR^\nu$ into $\RR^+$, is defined by
\begin{equation}\label{eq4.1}
 k_t(\bfy,\bfx) = \frac{\rho(\bfy,t\, \vert \, \bfx,0)}{p_\bfH(\bfy)} \, .
\end{equation}
\end{definition}

The following Lemma gives basic properties of kernel $k_t$.
%
\begin{lemma}[Properties of kernel $k_t$] \label{lemma:1}
For every fixed $t > 0$, and for all $\bfy$ and $\bfx$ in $\RR^\nu$, we have the following properties:
\begin{align}
(a) \quad &\int_{\RR^\nu} k_t(\bfy,\bfx)\, p_\bfH(\bfy)\, d\bfy = 1 \quad , \quad
                     \int_{\RR^\nu} k_t(\bfy,\bfx)\, p_\bfH(\bfx)\, d\bfx = 1                         \, , \label{eq4.2} \\
(b) \quad & \int_{\RR^\nu}\int_{\RR^\nu}   k_t(\bfy,\bfx)\, p_\bfH(\bfy)\, p_\bfH(\bfx)\,d\bfy \, d\bfx =1 \, ,  \label{eq4.3} \\
(c) \quad & \lim_{t\rightarrow 0_+}   k_t(\bfy,\bfx)\, p_\bfH(\bfy)\, d\bfy = \delta_0(\bfy-\bfx) \quad , \quad
                     \lim_{t\rightarrow +\infty} k_t(\bfy,\bfx) =1                                   \, ,  \label{eq4.4} \\
(d) \quad & (\bfy,\bfx)\mapsto  \rho(\bfy,t\, \vert \, \bfx,0) \in C^0(\RR^\nu\times\RR^\nu) \Rightarrow
                      (\bfy,\bfx)\mapsto  k_t(\bfy,\bfx) \in C^0(\RR^\nu\times\RR^\nu)                 \, ,  \label{eq4.5} \\
(e) \quad & k_t(\bfy,\bfx) = k_t(\bfx,\bfy)                                                                   \, .  \label{eq4.6}
\end{align}
\end{lemma}
%
\begin{proof} (Lemma~\ref{lemma:1}). Definition~\ref{definition:1} is used.

\noindent (a) The first equation in Eq.~\eqref{eq4.2} is due to Eq.~\eqref{eq3.7bis} and the second one is due to Eq.~\eqref{eq3.10bis}.

\noindent (b)  Eq.~\eqref{eq4.3}  is directly deduced from Eq.~\eqref{eq4.2}.

\noindent (c) The first equation in Eq.~\eqref{eq4.4} is due to Eq.~\eqref{eq3.7} and the second one is due to Eq.~\eqref{eq3.11}.

\noindent (d) For all $\bfy$ in $\RR^\nu$, $p_\bfH(y) > 0$ and $p_\bfH\in C^0(\RR^\nu)$. The hypothesis
$\rho(\cdot ,t\, \vert \, \cdot,0)\in C^0(\RR^\nu\times \RR^\nu)$ yields Eq.~\eqref{eq4.5}.

\noindent (e) let $p_{\bfY(t),\bfH}(\bfy,t;\bfx,0)$ be the joint pdf of $\bfY(t)$ with $\bfH$ in which $\bfY(t)$ is the solution of Eq.~\eqref{eq3.1} for fixed $t > 0$, with the random initial condition $\bfY(0) = \bfH$. We have the classical property related to the definition of the invariant measure,
$p_{\bfY(t)}(\bfy,t) = \int_{\RR^\nu} p_{\bfY(t),\bfH}(\bfy,t;\bfx,0)\, d\bfx =  \int_{\RR^\nu} \rho(\bfy,t\, \vert \, \bfx,0)\, p_\bfH(\bfx)\, d\bfx = p_\bfH(\bfy)$. For all $\bfy$ and $\bfx$ in $\RR^\nu$, and for $t > 0$, we have,
$\rho(\bfy,t\, \vert \, \bfx,0)\, p_\bfH(\bfx) = p_{\bfY(t),\bfH}(\bfy,t;\bfx,0) =p_{\bfH,\bfY(t)}(\bfx,0;\bfy,t) =
\rho(\bfx,0\, \vert \, \bfy, t)\, p_{\bfY(t)}(\bfy,t) = \rho(\bfx,0\, \vert \, \bfy, t)\, p_{\bfH}(\bfy)$.
Consequently,  $k_t(\bfy,\bfx)= \rho(\bfy,t\, \vert \, \bfx,0)\,p_\bfH(\bfx) / (p_\bfH(\bfy)\,p_\bfH(\bfx)) =
\rho(\bfx,0\, \vert \, \bfy, t)\, p_\bfH(\bfy) / (p_\bfH(\bfy)\,p_\bfH(\bfx)) = k_t(\bfx,\bfy)$.
\end{proof}
%

\subsection{Hypothesis and properties of kernel $k_t$ from its representation}
\label{Section4.2}
In Section~\ref{Section3}, we introduced an hypothesis of existence of a discrete (countable) spectrum $\{\lambda_\alpha,\alpha\in\NN\}$ of the FKP operator $\hat L_\FKP$. In this section, we study the spectral representation of kernel $k_t$, for every fixed $t > 0$, deduced from the time-dependent spectral representation of $\rho(\bfy,t\,\vert\, \bfx,0)$, defined by Eq.~\eqref{eq3.30}.

\begin{definition}[Hilbert space $\HH= L^2(\RR^\nu;p_\bfH)$] \label{definition:2}
Let $\HH= L^2(\RR^\nu;p_\bfH)$ be the Hilbert space of the square-integrable real-valued functions on $\RR^\nu$, with respect to the probability measure $p_\bfH(\bfy)\, d\bfy$ on $\RR^\nu$, equipped with the inner product and the associated norm,
\begin{equation}\label{eq4.7}
\langle u\, , \vc \rangle_\HH = \int_{\RR^\nu} u(\bfy)\, \vc(\bfy) \, p_\bfH(\bfy)\, d\bfy \quad , \quad \Vert u\Vert_\HH = \langle u\, , u \rangle_\HH^{1/2} \, .
\end{equation}
\end{definition}
%

%
\begin{lemma}[Hilbert basis in $\HH$] \label{lemma:2}
Let $\{q_\alpha,\alpha\in\NN\}$ be the Hilbert basis of $L^2(\RR^\nu)$ introduced in Section~\ref{Section3.6}.
For all $\alpha\in\NN$, we defined the real-valued function $\psi_\alpha$ on $\RR^\nu$ such that
\begin{equation}\label{eq4.8}
\psi_\alpha(\bfy) = q_\alpha(\bfy)\, p_\bfH(\bfy)^{-1/2}\quad , \quad \forall \bfy\in\RR^\nu \, .
\end{equation}
Then, $\{\psi_\alpha,\alpha\in\NN\}$ is a Hilbert basis of $\HH$ and we have,
\begin{align}
(a) \quad &\psi_\alpha\in\HH\quad , \quad \Vert\psi_\alpha\Vert_\HH = \Vert q_\alpha\Vert_{L^2} = 1
           \quad , \quad \forall\alpha\in\NN                                                                     \, , \label{eq4.9} \\
(b) \quad &\langle\psi_\alpha\, , \psi_\beta \rangle_\HH = \int_{\RR^\nu} \psi_\alpha(\bfy)\, \psi_\beta(\bfy) \,
              p_\bfH(\bfy)\, d\bfy = \delta_{\alpha\beta} \quad , \quad \forall (\alpha,\beta)\in\NN\times\NN \, , \label{eq4.10} \\
(c) \quad & \psi_0(\bfy) = 1 \, , \, \forall\bfy\in\RR^\nu \quad , \quad \Vert\psi_0\Vert_\HH = 1              \, , \label{eq4.11} \\
(d) \quad & \int_{\RR^\nu} \psi_\alpha(\bfy)\, p_\bfH(\bfy) \, d\bfy = 0 \quad , \quad \forall \alpha\in\NN^*  \, , \label{eq4.12} \\
(e) \quad & \sum_{\alpha\in\NN} \psi_\alpha(\bfy)\, \psi_\beta(\bfx) \, p_\bfH(\bfy) \, d\bfy = \delta_0(\bfy-\bfx)
              \quad , \quad \forall (\bfy,\bfx) \in \RR^\nu\times\RR^\nu                                          \, . \label{eq4.13}
\end{align}

\end{lemma}
%
\begin{proof} (Lemma~\ref{lemma:2}).

\noindent (a) Since $p_\bfH^{-1/2}\in C^0(\RR^\nu)$, Eqs.~\eqref{eq4.8} yields Eq.~\eqref{eq4.9}; combined with
Eq.~\eqref{eq3.29}, this yields $\Vert\psi_\alpha\Vert_\HH = \int_{\RR^\nu} q_\alpha(\bfy)^2\, d\bfy = \Vert q_\alpha\Vert_{L^2} = 1$.

\noindent (b) Using Eq.~\eqref{eq3.26} yields $\langle\phi_\alpha\, , \psi_\beta \rangle_\HH = \int_{\RR^\nu} q_\alpha(\bfy)\, q_\beta(\bfy) \, d\bfy = \langle q_\alpha\, , q_\beta \rangle_{L^2}= \delta_{\alpha\beta}$. Thus $\{\psi_\alpha,\alpha\in\NN\}$ is an orthonormal family in $\HH$. For all $u$ in $L^2(\RR^\nu)$, the linear mapping $u\mapsto \vc = u p_\bfH^{-1/2}$ is a continuous injection from $L^2(\RR^\nu)$ into $\HH$  with $\Vert \vc\Vert_\HH = \Vert u \Vert_{L^2}$. Therefore, $\{\psi_\alpha,\alpha\in\NN\}$ is a Hilbert basis of $\HH$.

\noindent (c) The two equations in Eq.~\eqref{eq4.11} are directly deduced from  Eqs.~\eqref{eq3.27} and \eqref{eq4.8}.

\noindent (d) Since  $\langle \psi_0\, , \psi_\alpha\rangle_{\HH} = 0$ for all $\alpha\in\NN^*$, we obtain Eq.~\eqref{eq4.12}.

\noindent (e) Since  $\{\psi_\alpha,\alpha\in\NN\}$ is a Hilbert basis of $\HH$, Eq.~\eqref{eq4.13} holds.
\end{proof}
%

%
\begin{proposition}[Spectral representation of kernel $k_t$] \label{proposition:1}
Let $\{\psi_\alpha,\alpha\in\NN\}$ be the Hilbert basis of $\HH$ defined in Lemma~\ref{lemma:2}.

\noindent (a) For every fixed $t > 0$, the symmetric kernel $k_t$ can be written, for all $\bfy$ and $\bfx$ in $\RR^\nu$, as
\begin{equation}\label{eq4.14}
k_t(\bfy,\bfx) =\sum_{\alpha\in\NN} b_\alpha(t)\, \psi_\alpha(\bfy)\, \psi_\alpha(\bfx) \, ,
\end{equation}
in which the family of positive real numbers $\{b_\alpha(t)= exp(-\lambda_\alpha\, t)\, , \alpha\in\NN\}$, is such that
\begin{equation}\label{eq4.15}
1 = b_0(t) > b_1(t) \geq b_2(t) \geq \ldots  \, .
\end{equation}
\noindent (b) If for every fixed $t > 0$, kernel $k_t$ satisfies
\begin{equation}\label{eq4.16}
\int_{\RR^\nu} \int_{\RR^\nu} k_t(\bfy,\bfx)^2 \, p_\bfH(\bfy)\, p_\bfH(\bfx)\, d\bfy \, d\bfx = c_t^2 < +\infty \, ,
\end{equation}
where $c_t > 1$  is a positive constant depending on $t$, then,
\begin{equation}\label{eq4.17}
\sum_{\alpha\in\NN} b_\alpha(t)^2 = c_t^2 < +\infty\, .
\end{equation}
\end{proposition}
%
\begin{proof} (Proposition~\ref{proposition:1}).

\noindent (a) Substituting  Eq.~\eqref{eq3.30} with $q_\alpha(\bfy) = \psi_\alpha(\bfy)\, p_\bfH(\bfy)^{1/2}$ (see Eq.~\eqref{eq4.8}) into
Eq.~\eqref{eq4.1} yields Eq.~\eqref{eq4.14}. From Eq.~\eqref{eq3.24} and since $b_\alpha(t)= exp(-\lambda_\alpha\, t)$, we obtain
Eq.~\eqref{eq4.15}.

\noindent (b) Assuming Eq.~\eqref{eq4.16}, substituting Eq.~\eqref{eq4.14} into Eq.~\eqref{eq4.16}, and using Eq.~\eqref{eq4.10} yields
Eq.~\eqref{eq4.17}. Since $b_0(t) = 1$, it can be deduced that $c_t^2 > 1$ and thus, $c_t > 1$.
\end{proof}

%
\subsection{Hilbert-Schmidt operator $K_t$ associated with kernel $k_t$}
\label{Section4.3}
We now introduce the linear operator in $\HH$, defined by kernel $k_t$, and we study its properties and spectrum.
%
\begin{definition}[Operator $K_t$ associated with kernel $k_t$] \label{definition:3}
For every fixed $t > 0$, we defined the linear operator $K_t$ from $\HH$ into $\HH$ such that, for all $u$ and $\vc$ in $\HH$,
\begin{equation}\label{eq4.18}
  \langle K_t\, u\, , \vc\rangle_\HH = \int_{\RR^\nu}\int_{\RR^\nu}k_t(\bfy,\bfx)\, u(\bfx)\, \vc(\bfy) \, p_\bfH(\bfy)\, p_\bfH(\bfx)\, d\bfy\, d\bfx\, ,
\end{equation}
where the symmetric kernel $k_t$ verifies the condition defined by Eq.~\eqref{eq4.16}.
\end{definition}
%

\begin{proposition}[$K_t$ as a Hilbert-Schmidt operator in $\HH$] \label{proposition:2}
For every fixed $t > 0$, let $K_t$ be the continuous linear operator defined by Eq.~\eqref{eq4.18}, in which kernel $k_t$ is symmetric on $\RR^\nu \times \RR^\nu$, and verifies Eq.~\eqref{eq4.16}.

\noindent (a) For all $u$ and $\vc$ in $\HH$, operator $K_t$ is such that
\begin{equation}\label{eq4.19}
  \langle K_t\, u\, , \vc\rangle_\HH = \sum_{\alpha\in\NN} b_\alpha(t) \, \langle u\, ,\psi_\alpha\rangle_\HH \, \langle \vc\, ,\psi_\alpha\rangle_\HH\, ,
\end{equation}
and is a positive symmetric operator in $\HH$. For all $u$ in $\HH$,
\begin{equation}\label{eq4.20}
  K_t\, u = \sum_{\alpha\in\NN} b_\alpha(t) \, \langle u\, ,\psi_\alpha\rangle_\HH \, \psi_\alpha\, .
\end{equation}
\noindent (b) For all $\alpha$ in $\NN$, $\psi_\alpha\in\HH$ is the eigenfunction independent of $t$, associated with the positive eigenvalue $b_\alpha(t)$, satisfying Eq.~\eqref{eq4.15}, of operator $K_t$,
\begin{equation}\label{eq4.21}
  K_t\, \psi_\alpha = b_\alpha(t) \, \psi_\alpha\quad , \quad b_\alpha(t) = \exp(-\lambda_\alpha t) \quad , \quad \alpha \in \NN \, ,
\end{equation}
which shows that, for all $\alpha$ and $\beta$ in $\NN$,
\begin{equation}\label{eq4.22}
  \langle K_t\, \psi_\alpha\, , \psi_\beta\rangle_\HH = b_\alpha(t) \, \delta_{\alpha\beta}\, .
\end{equation}
\noindent (c) For all $u$ in $\HH$, we have
\begin{equation}\label{eq4.23}
  \Vert K_t\, u \Vert_\HH^2 = \sum_{\alpha\in\NN} b_\alpha(t)^2 \, \langle u\, ,\psi_\alpha\rangle_\HH^2 \, ,
\end{equation}
and for Hilbert basis $\{\psi_\alpha\, , \alpha\in\NN\}$ of $\HH$,
\begin{equation}\label{eq4.24}
 \sum_{\alpha\in\NN} \Vert K_t\, \psi_\alpha \Vert_\HH^2 =  c_t^2 < + \infty \, ,
\end{equation}
where $c_t^2$, defined by Eq.~\eqref{eq4.16}, is such that $c_t^2 = \sum_{\alpha\in\NN} b_\alpha(t)^2$, and therefore, is a Hilbert-Schmidt operator in $\HH$.
\end{proposition}

\begin{proof} (Proposition~\ref{proposition:2}).
Under the condition defined by Eq.~\eqref{eq4.16}, it is well known that operator $K_t$ is continuous from $\HH$ into $\HH$.

\noindent (a) Substituting Eq.~\eqref{eq4.14} into Eq.~\eqref{eq4.18} and using Eq.~\eqref{eq4.7} yield  Eq.~\eqref{eq4.19}. This equation shows that $K_t$ is a symmetric and positive operator because $\langle K_t\, u\, , u \rangle_\HH > 0$ for all $u$ in $\HH$ with $\Vert u\Vert_\HH \not = 0$. Note that  Eq.~\eqref{eq4.20} is directly deduced from Eq.~\eqref{eq4.19}.

\noindent (b)  From Eqs.~\eqref{eq4.8} and \eqref{eq4.9}, we have $\psi_\alpha\in\HH$. Taking $u = \psi_\beta$ in Eq.~\eqref{eq4.20} and using Eq.~\eqref{eq4.10} yield $K_t\, \psi_\beta = \sum_{\alpha\in\NN} b_\alpha(t) \, \langle \psi_\beta\, ,\psi_\alpha\rangle_\HH\, \psi_\alpha = b_\beta(t)\, \psi_\beta$, Eq.~\eqref{eq4.22} is obtained from Eq.~\eqref{eq4.19} by taking $u=\psi_\alpha$ and $\vc = \psi_\beta$.
 The relationship between $b_\alpha(t)$ and $\lambda_\alpha$ comes from Proposition~\ref{proposition:1}.

\noindent (c) Using Eqs.~\eqref{eq4.20} and \eqref{eq4.10} yields Eq.~\eqref{eq4.23}. Taking $u=\psi_\alpha$ in Eq.~\eqref{eq4.23} yields
$\Vert K_t\, \psi_\alpha \Vert_\HH^2 = b_\alpha(t)^2$. From Eq.~\eqref{eq4.17} and
$\sum_{\alpha\in\NN} \Vert K_t\, \psi_\alpha \Vert_\HH^2 = \sum_{\alpha\in\NN} b_\alpha(t)^2$, we obtain Eq.~\eqref{eq4.24}. It can be deduced (see for instance \cite{Guelfand1967}) that $K_t$ is a Hilbert-Schmidt operator in $\HH$.
\end{proof}
%
\subsection{Finite approximation of operator $K_t$ and of its eigenvalue problem}
\label{Section4.4}
The Hilbert-Schmidt operator $K_t$ defined by Eq.~\eqref{eq4.18}, operates in infinite dimension. The Hilbert basis
$\{ \psi_\alpha\, , \alpha\in\NN\}$ (which relates to the Hilbert basis $\{ q_\alpha\, , \alpha\in\NN\}$, see Lemma~\ref{lemma:2} and  Eq.~\eqref{eq3.16}) is not explicitly known, thereby preventing the use of the representation defined by Eq.~\eqref{eq4.20}.
We must thus construct a finite approximation of $K_t$.
Since  $\RR^\nu$ is an unbounded set and $\nu$ can be very large, classical discretization such as finite-difference or finite-element methods (see \cite{Spencer1993,Masud2005,Kumar2006,Pichler2013} for Fokker-Planck equation and \cite{Deng2009} for fractional Fokker-Planck equation) or such methods based on  shape-morphing modes for solving the Fokker-Planck equation as proposed in \cite{Anderson2024}, are not directly adapted for solving the eigenvalue problem of operator $\hat L_\FKP$. Another classical method consists in introducing a finite family of functions in $\HH$, generating a finite dimension subspace of $\HH$, and in performing the projection of $K_t$ on this finite subspace. Such an approach is not really adapted to operator $\hat L_\FKP$ for which a large number of eigenvalues and associated eigenfunctions have to be computed.
It should be noted that a related problem, but distinct from the one addressed, is that of the numerical method for Schr\"odinger operator and the associated equation (see for instance \cite{Popelier2011} for solving the Schr\"odinger Equation, \cite{Feit1982} for the solution of the Schr{\"o}dinger equation by spectral methods,  \cite{Iitaka1994} for numerically solving the time-dependent Schr{\"o}dinger equation,  and \cite{Simos1997} for the numerical solution of the Schr{\"o}dinger equation using finite-difference method).
Nevertheless, such approaches are not well adapted to the objective of the actual developments, which has been detailed in Section~\ref{Section1}. We then propose to use a statistical sampling of $\RR^\nu$ equipped with the probability measure $p_\bfH(\bfeta)\, d\bfeta$, which will be well adapted to our objective of performing a construction connected to the DMAPS approach.

\begin{proposition}[Probabilistic interpretation of the bilinear form $\langle K_t u\, , \vc \rangle_\HH$] \label{proposition:3}
Let us assume that, for every fixed $t > 0$, we have $(\bfy,\bfx)\mapsto  \rho(\bfy,t\, \vert \, \bfx,0) \in C^0(\RR^\nu\times\RR^\nu)$. From Eq.~\eqref{eq4.5}, it can be deduced that {\color{blue}$(\bfy,\bfx)\mapsto  k_t(\bfy, \bfx) \in C^0(\RR^\nu\times\RR^\nu)$}.
For every fixed $t > 0$, for all $u$ and $\vc$ in $\HH \cap C^0(\RR^\nu)$, the restriction to $(\HH \cap C^0(\RR^\nu))\times (\HH \cap C^0(\RR^\nu))$ of the bilinear form $(u,\vc)\mapsto \langle K_t u\, , \vc \rangle_\HH$,
defined on $\HH\times \HH$ by Eq.~\eqref{eq4.18} with the continuous symmetric function $k_t$ verifying Eq.~\eqref{eq4.16}, can be written as
\begin{equation}\label{eq4.25}
\langle K_t u\, , \vc \rangle_\HH  = E\{k_t(\bfH,\undertilde{\bfH}) \, u(\undertilde{\bfH})\, \vc (\bfH)\}\, ,
\end{equation}
in which $\undertilde{\bfH}$ is an independent copy of $\bfH$. The joint probability measure $P_{\bfH,\undertilde{\bfH}}(d\bfy,d\bfx)$ of $\bfH$ with $\undertilde{\bfH}$ is $p_\bfH(\bfy)\times p_\bfH(\bfx)\, d\bfy\, d\bfx$.
The real-valued random variables $u(\undertilde\bfH)$, $\vc(\bfH)$, and the positive-valued random variable $k_t(\bfH,\undertilde{\bfH})$, defined on $(\Theta,\curT,\curP)$, are second-order random variables,
\begin{equation}\label{eq4.25bis}
 E\{u(\undertilde\bfH)^2 \} < +\infty \quad , \quad E\{\vc(\bfH)^2 \} < +\infty \quad , \quad E\{k_t(\bfH,\undertilde{\bfH})^2\} < +\infty\, .
\end{equation}
Let $Z_{t,u,\pvc}$ be the real-valued random variable, defined on $(\Theta,\curT,\curP)$, such that
\begin{equation}\label{eq4.26}
Z_{t,u,\pvc} = k_t(\bfH,\undertilde{\bfH}) \, u(\undertilde{\bfH})\, \vc (\bfH)\, .
\end{equation}
Then, $Z_{t,u,\pvc}$ is such that,
\begin{equation}\label{eq4.26bis}
E\{Z_{t,u,\pvc}\} < +\infty \, .
\end{equation}
\end{proposition}
%
\begin{proof} (Proposition~\ref{proposition:3}).

\noindent (a) Since $u$ and $\vc$ are continuous functions in $\HH$, $u(\bfH)$ and $\vc(\bfH)$ are real-valued random variables defined on $(\Theta,\curT,\curP)$ and are second-order because,
\begin{equation}
E\{ u(\undertilde\bfH)^2\} =  \int_{\RR^\nu} u(\bfx)^2 p_\bfH(\bfx)\, d\bfx = \Vert u \Vert_\HH^2  < +\infty  \quad , \quad
E\{ \vc(\bfH)^2\}          =  \int_{\RR^\nu} \vc(\bfy)^2 p_\bfH(\bfy)\, d\bfy = \Vert \vc \Vert_\HH^2  < +\infty \, . \nonumber
\end{equation}
Since $k_t\in C^0(\RR^\nu\times\RR^\nu)$ (see Eq.~\eqref{eq4.5}) and due to Eq.~\eqref{eq4.16},
$k_t(\bfH,\undertilde{\bfH})$ is a second-order positive-valued random variable defined on $(\Theta,\curT,\curP)$,
\begin{equation}\nonumber
E\{k_t(\bfH,\undertilde{\bfH})^2\} =  \int_{\RR^\nu}\int_{\RR^\nu} k_t(\bfy,\bfx)^2\,p_\bfH(\bfy)\,p_\bfH(\bfx)  \, d\bfy\, d\bfx =c_t^2 < +\infty\, .
\end{equation}

\noindent (b) For all $u$ in $\HH$ and for $t > 0$, from Eq.~\eqref{eq4.18}, it can be seen that
$(K_t u)(\bfy) = \int_{\RR^\nu} k_t(\bfy,\bfx)\, p_\bfH(\bfx)^{1/2} \, u(\bfx)\, p_\bfH(\bfx)^{1/2} \, d\bfx$.
Using the Cauchy-Schwarz inequality  and Eq.~\eqref{eq4.16} yield
\begin{equation}\nonumber
\Vert K_t u \Vert_\HH^2 \, \leq \, \Vert u \Vert_\HH^2\, \int_{\RR^\nu}\int_{\RR^\nu} k_t(\bfy,\bfx)^2\,p_\bfH(\bfy)\,p_\bfH(\bfx)  \, d\bfy\, d\bfx
= c_t^2\,\Vert u \Vert_\HH^2 \, ,
\end{equation}
(which, in passing, shows the continuity of the operator $K_t$ in $\HH$  as stated at the beginning of the proof of Proposition~\ref{proposition:2}).
Consequently, we have $\langle K_t u\, , \vc \rangle_\HH^2 \leq \Vert K_t u \Vert_\HH^2 \, \Vert \vc \Vert_\HH^2$, which shows that
\begin{equation}\label{eq4.27}
\langle K_t u\, , \vc \rangle_\HH^2  \leq c_t^2 \, \Vert u \Vert_\HH^2\, \Vert \vc \Vert_\HH^2  < +\infty\, ,
\end{equation}
and therefore,  Eq.~\eqref{eq4.26bis} holds.
\end{proof}
%

%
\begin{definition}[Estimator constructed with a statistical sampling and associated estimation] \label{definition:4}
Let $\{\bfH^i,i=1,\ldots, n_d\}$ and $\{\undertilde{\bfH}^j,j=1,\ldots, n_d\}$ be $n_d$ independent copies  of $\bfH$ and $\undertilde\bfH$, respectively. For every fixed $t > 0$, and for all $u$ and $\vc$ in $\HH\cap C^0(\RR^\nu)$, let $Z_{t,u,\pvc}^{ij}$ be the real valued random variable on $(\Theta,\curT,\curP)$, such that, for all $i$ and $j$ in $\{1,\ldots , n_d\}$,
\begin{equation}\label{eq4.31}
Z_{t,u,\pvc}^{ij} = k_t(\bfH^i , \undertilde{\bfH}^j) \, u(\undertilde{\bfH}^j)\, \vc(\bfH^i) \, .
\end{equation}
Let $\widehat Z_{t,u,\pvc}^{\,(n_d)}$ be the real-valued random variable on $(\Theta,\curT,\curP)$, defined by
\begin{equation}\label{eq4.33}
\widehat Z_{t,u,\pvc}^{\,(n_d)} = \frac{1}{n_d^2} \sum_{i=1}^{n_d}\sum_{j=1}^{n_d} Z_{t,u,\pvc}^{ij}\, .
\end{equation}
Then, $\widehat Z_{t,u,\pvc}^{\,(n_d)}$ is an estimator of
\begin{equation}\label{eq4.32}
{\underline z}_{\, t,u,\pvc} = E\{Z_{t,u,\pvc}\} = \langle K_t u\, , \vc \rangle_\HH\, ,
\end{equation}
where $Z_{t,u,\pvc}$ is given by Eq.~\eqref{eq4.26}.
Let $\{\bfeta^j , j=1,\ldots , n_d\}$ be  the $n_d$ independent realizations of $\bfH$ introduced in Section~\ref{Section2}. Since $\bfH^i$ and $\undertilde{\bfH}^j$ are independent copies of $\bfH$ (because $\undertilde\bfH$ is an independent copy of $\bfH$), an estimation
of ${\underline z}_{\, t,u,\pvc}$ is a realization ${\underline z}_{\, t,u,\pvc}^{(n_d)}$ of the estimator  $\widehat Z_{t,u,\pvc}^{\,(n_d)}$, which is written as
\begin{equation}\label{eq4.34}
{\underline z}_{\, t,u,\pvc}^{(n_d)} = \frac{1}{n_d^2} \sum_{i=1}^{n_d}\sum_{j=1}^{n_d} k_t(\bfeta^i,\bfeta^j)\, u(\bfeta^j)\, \vc(\bfeta^i)\, .
\end{equation}
\end{definition}

%
\begin{lemma}[Convergence of the sequence of estimators $\{\widehat Z_{t,u,\pvc}^{\,(n_d)}\}_{n_d}$] \label{lemma:3}
Under the hypotheses and notations of Definition~\ref{definition:4}, the sequence of real-valued random variables
$\{\widehat Z_{t,u,\pvc}^{\,(n_d)}\}_{n_d}$ on $(\Theta,\curT,\curP)$ is convergent in probability to
${\underline z}_{\, t,u,\pvc}$,
\begin{equation}\label{eq4.35}
  \forall \epsilon > 0 \quad , \quad \lim_{n_d\rightarrow +\infty} \curP \{ \, \vert \,  \widehat Z_{t,u,\pvc}^{\,(n_d)}
                               - {\underline z}_{\, t,u,\pvc}\, \vert \, \, \geq \, \epsilon\} = 0 \, .
\end{equation}
We have also the almost sure convergence, thanks to the strong law of large numbers,
\begin{equation}\label{eq4.36}
  \curP\{ \lim_{n_d\rightarrow +\infty} \widehat Z_{t,u,\pvc}^{\,(n_d)}
                               = {\underline z}_{\, t,u,\pvc}\} = 1 \, .
\end{equation}
\end{lemma}
%
\begin{proof} (Lemma~\ref{lemma:3}).
The proof uses the usual results from mathematical statistics (see for instance \cite{Serfling1980}).
\end{proof}
%

%
\begin{remark} \label{remark:3}

\noindent (a) As is well known, the speed of convergence is proportional to $1/\sqrt{n_d^2} = 1/n_d$ and is independent of dimension $\nu$. The quantification of the approximation error could traditionally be estimated using the central limit theorem, which involves the variance of the estimator (see, for instance, \cite{Serfling1980,Givens2013,Soize2017b}). In the numerical illustration provided in Section \ref{Section6}, we will show the numerical calculation of  the first eigenvalues of the Fokker-Planck operator for which a reference is known.

\noindent (b) Lemma~\ref{lemma:3} with Definition~\ref{definition:4} and Proposition~\ref{proposition:3} allow a finite approximation of $\langle K_t u\, , \vc \rangle_\HH$ to be constructed and consequently, to deduce the corresponding finite approximation of the eigenvalue problem defined by Eq.~\eqref{eq4.21}.
\end{remark}
%
%
\begin{proposition}[Finite approximation of the eigenvalue problem] \label{proposition:4}
For every fixed $t > 0$, for all $u$ and $\vc$ in $\HH\cap C^0(\RR^\nu)$, and for $n_d$ sufficiently large, we have (in the sense of the convergence described in Lemma~\ref{lemma:3}),
\begin{equation}\label{eq4.35}
\langle K_t u\, , \vc \rangle_\HH   \simeq \langle   [\widehat K(t)]\,\hat \bfu\, , \hat\bfv  \rangle_{\RR^{n_d}} \, ,
\end{equation}
in which $\langle \cdot , \cdot \rangle_{\RR^{n_d}}$ is the usual Euclidean inner product in $\RR^{n_d}$ and where $[\widehat K(t)]\in\MM_{n_d}^{+0}$ is such that, for all $i$ and $j$ in $\{ 1,\ldots, n_d\}$,
\begin{equation}\label{eq4.36}
[\widehat K(t)]_{ij} = \frac{1}{n_d} \, k_t(\bfeta^i,\bfeta^j) \, ,
\end{equation}
and where the vectors $\hat\bfu$ and $\hat\bfv$ in $\RR^\nu$ are such that
\begin{equation}\label{eq4.37}
\hat\bfu = \left(\frac{u(\bfeta^1)}{\sqrt{n_d}} \, , \, \ldots \, , \, \frac{u(\bfeta^{n_d})}{\sqrt{n_d}}\right )\quad , \quad
\hat\bfv = \left(\frac{\vc(\bfeta^1)}{\sqrt{n_d}} \, , \, \ldots \, , \, \frac{\vc(\bfeta^{n_d})}{\sqrt{n_d}}\right ) \, .
\end{equation}
The corresponding finite approximation of the eigenvalue problem defined by Eq.~\eqref{eq4.21} is written as
\begin{equation}\label{eq4.38}
[\widehat K(t)] \, \hat\bfpsi_\alpha(t) = \hat b_\alpha(t)\, \hat\bfpsi_\alpha(t) \quad , \quad \hat\bfpsi_\alpha(t) \in\RR^{n_d} \, ,
\end{equation}
in which
\begin{equation}\label{eq4.39}
\hat b_0(t)\, \geq \, \hat b_1(t) \, \geq \,\ldots \, \geq \, \hat b_{n_d-1}(t) \, \geq 0\, ,
\end{equation}
and where the normalization of the eigenvectors is chosen so that for $\alpha$ and $\beta$ in $\{0,1,\ldots, n_d-1 \}$,
\begin{equation}\label{eq4.40}
\langle \hat\bfpsi_\alpha(t) , \hat\bfpsi_\beta(t)  \rangle_{\RR^{n_d}} = \delta_{\alpha\beta} \, .
\end{equation}
\end{proposition}
%
\begin{proof} (Proposition~\ref{proposition:4}).
From Eqs.~\eqref{eq4.32}, \eqref{eq4.34}, and \eqref{eq4.36}, it can be deduced that, for $n_d$ sufficiently large,
$\langle K_t u\, , \vc \rangle_\HH \simeq n_d^{-2} \sum_{i=1}^{n_d}\sum_{j=1}^{n_d} k_t(\bfeta^i,\bfeta^j)\, u(\bfeta^j)\, \vc(\bfeta^i)$.
Since $k_t(\bfeta^i,\bfeta^j) = k_t(\bfeta^j,\bfeta^i)$ by symmetry of $k_t$ (see Eq.~\eqref{eq4.6}), the right-hand side member can be written as $\langle   [\widehat K(t)]\,\hat \bfu\, , \hat\bfv  \rangle_{\RR^{n_d}}$  where $[\widehat K(t)]$, $\hat \bfu$, and $\hat\bfv$ are defined by Eqs.~\eqref{eq4.36} and \eqref{eq4.37}. From Proposition~\ref{proposition:2}, $K_t$ is a positive operator in $\HH$, and as Hilbert-Schmidt operator, $b_\alpha(t) \rightarrow 0_+$ as $\alpha\rightarrow +\infty$. For the finite approximation, matrix $[\widehat K(t)]$ is then symmetric  and positive. Using Eq.~\eqref{eq4.35}, the finite approximation  of Eq.~\eqref{eq4.21} is then written as Eq.~\eqref{eq4.38}. Since
$[\widehat K(t)]$ is real, symmetric, and positive, we have Eqs.~\eqref{eq4.39} and we choose the normalization of $\{\hat\bfpsi_\alpha(t)\}_\alpha$ so that  \eqref{eq4.40} holds.
\end{proof}
%

%
\begin{remark} [About the finite approximation of the eigenvalue problem] \label{remark:4}

\noindent (a) It should be noted that we have chosen the construction of $\{\hat\bfpsi_\alpha(t)\}_\alpha$ as an orthonormal basis in $\RR^{n_d}$. Consequently, $\hat\bfpsi_\alpha(t)$ is not related  to $\psi_\alpha$ by a simple sampling, similar to the one described  in  Eq.~\eqref{eq4.37}.

\noindent (b) In addition, $\hat\bfpsi_\alpha(t)$ depends, \textit{a priori}, on $t$, while $\psi_\alpha$ is independent of $t$. Only for
$n_d\rightarrow +\infty$, $\hat\bfpsi_\alpha(t)$ goes to a vector independent of $t$. Similarly, although $b_0(t)= \exp(-\lambda_0 t) = 1$ because $\lambda_0 = 0$ (see Eq.~\eqref{eq4.15}), we do not have, \textit{ a priori}, $\hat b_0(t) = 1$, but this equality holds for $n_d\rightarrow +\infty$. Nevertheless, for $n_d$ finite, we will exploit this existing dependence on $t$ for the construction of the reduced transient basis at time $t$, which will be connected to the reduced DMAPS basis for $t\rightarrow 0$.

\noindent (c) Note that as $n_d\, \rightarrow \, +\infty$, $\hat b_\alpha(t) $ tends to $b_\alpha(t)$. Since $b_\alpha(t) = \exp(-\lambda_\alpha t)$, that is to say $\lambda_\alpha = -\frac{1}{t} \, \log\, b_\alpha(t) $, we choose to define $\ \lambdahat_\alpha(t)$ by a similar formula,  such that for $t > 0$ and $\alpha$ for which $\hat b_\alpha(t) > 0$,
\begin{equation}\label{eq4.41}
 \lambdahat_\alpha(t)  = -\frac{1}{t} \, \log\, \hat b_\alpha(t) \, .
\end{equation}
\noindent (d) For every fixed $t >0$, to solve the eigenvalue problem defined by Eq.~\eqref{eq4.38}, we have to construct matrix $[\widehat K(t)]$ with an adapted methodology. This will be the object of Section~\ref{Section5}.
\end{remark}
%
\section{Construction of the matrix of the finite approximation}
\label{Section5}
%
In this section, we present the methodology to construct matrix $[\widehat K(t)]$ as defined in Proposition~\ref{proposition:4}. This construction requires the numerical evaluation of kernel $k_t$, because the entry $[\widehat K(t)]_{ij}$ of $[\widehat K(t)]$ is
given by $k_t(\bfeta^i,\bfeta^j)/n_d$ (see Eq.~\eqref{eq4.36}). According to Definition~\ref{definition:1} (see Eq.~\eqref{eq4.1}), for $i$ and $j$ in $\{1,\ldots , n_d\}$,
\begin{equation}\label{eq5.0}
[\widehat K(t)]_{ij} = \frac{1}{n_d} \, \frac{\rho(\bfeta^i,t\, \vert \, \bfeta^j,0)}{p_\bfH(\bfeta^i)} \, .
\end{equation}
In Eq.~\eqref{eq5.0}, $p_\bfH$ is explicitly defined by Eq.~\eqref{eq2.3}, and $\rho(\bfy,t\, \vert \, \bfx,0)$ is the transient probability density function of the stochastic process $\{\bfY(t), t\geq 0\}$, starting from $\bfY(0) =\bfx$ in $\RR^\nu$. This function is the solution of the ISDE defined by Eq.~\eqref{eq3.1} for $t > 0$ and with the initial condition $\bfY(0) =\bfx$ (see Eq.~\eqref{eq3.2}).
We will begin by using a a classical mathematical result concerning the solution of the ISDE, which must be validated for the specific case where the invariant measure is  defined in Section~\ref{Section2}. Additionally, we will obtain a proof of the properties introduced in Section~\ref{Section3.2}.
Then, we will present the numerical method for constructing $[\widehat K(t)]$ by numerically solving the ISDE and using nonparametric statistics to estimate $\rho(\bfeta^i,t\, \vert \, \bfeta^j,0)$. Subsequently, we will derive an explicit algebraic formula for $[\widehat K(t)]_{ij}$.

\subsection{Existence and uniqueness of the solution of ISDE and properties of the transition probability}
\label{Section5.1}
Let $\bfY=\{\bfY(t), t\geq 0\}$ be the $\RR^\nu$-valued stochastic process satisfying (see Eqs.~\eqref{eq3.1} and \eqref{eq3.2}),
\begin{align}
d\bfY(t) = & \, \bfb(\bfY(t))\, dt + d\bfW(t) \quad , \quad t > 0 \, , \label{eq5.1}\\
\bfY(0) = & \, \bfx \in \RR^\nu \, , \, a.s. \, ,  \label{eq5.2}
\end{align}
with $\bfx\in\RR^\nu$, where the drift $\bfy\mapsto \bfb(\bfy) : \RR^\nu \mapsto \RR^\nu$ is defined by Eq.~\eqref{eq3.3} with Eq.~\eqref{eq2.7}, and where $\{\bfW(t),t\geq 0\}$ is the normalized Wiener process defined on $(\Theta,\curT,\curP)$.
%
%
\begin{proposition}[Existence and uniqueness] \label{proposition:5}
Eqs.~\eqref{eq5.1} and \eqref{eq5.2} define a unique homogeneous diffusion $\RR^\nu$-valued stochastic process $\bfY$ defined on $(\Theta,\curT,\curP)$, whose transition probability measure is homogeneous (that is to say, it depends only on $t - 0 = t$),
\begin{equation}\label{eq5.3}
P_{\bfY(t)\vert\bfY(0)}(\curB,t\, \vert \, \bfx,0) = \curP\{\bfY(t)\in\curB\, \vert \, \bfY(0)=\bfx\}\quad , \quad t > 0\, ,
\end{equation}
where $\curB$ is any Borel set in $\RR^\nu$. For all $t > 0$ and for all $\bfx$ in $\RR^\nu$, $P_{\bfY(t)\vert\bfY(0)}(d\bfy , t\,\vert\,\bfx,0)$ admits a density function $\bfy\mapsto \rho(\bfy,t\, \vert \, \bfx,0) : \RR^\nu \mapsto ]0\, , +\infty[$ with respect to the Lebesgue measure $d\bfy$ on $\RR^\nu$, such that
\begin{align}
& P_{\bfY(t)\, \vert \, \bfY(0)}(d\bfy,t\, \vert \, \bfx,0) = \rho(\bfy,t\, \vert \, \bfx,0)\, d\bfy \quad , \quad t > 0 \, ,\label{eq5.4}\\
& \lim_{t\rightarrow 0_+} \rho(\bfy,t\, \vert \, \bfx,0)\, d\bfy = \delta_0(\bfy-\bfx) \, . \label{eq5.5}
\end{align}
Stochastic process $\bfY$ has almost-surely continuous trajectories and for all $t > 0$, $\bfY(t)$ is a second-order random variable,
\begin{equation}\label{eq5.6}
\forall t \geq 0 \quad , \quad E\{\Vert \bfY(t)\Vert^2\} < +\infty \, .
\end{equation}
For $t\rightarrow +\infty$, $\bfY$ is asymptotic to a stationary stochastic process whose first-order marginal probability measure is the invariant measure $p_\bfH(\bfy)\, d\bfy$,
\begin{equation}\label{eq5.7}
\lim_{t\rightarrow +\infty} \rho(\bfy,t\, \vert \, \bfx,0)\, d\bfy = p_\bfH(\bfy)\, d\bfy\, .
\end{equation}
The function $(\bfy,\bfx)\mapsto \rho(\bfy,t\, \vert \, \bfx,0)$ is continuous from $\RR^\nu\times\RR^\nu$ into $\RR^{+*}$,
\begin{equation}\label{eq5.8}
\rho(\cdot,t\, \vert \, \cdot ,0) \in C^0(\RR^\nu\times\RR^\nu, \RR^{+*}) \, .
\end{equation}
\end{proposition}
%
\begin{proof} (Proposition~\ref{proposition:5}).
The proof is presented in five steps.

\noindent (a) It is easy to prove that drift function $b$ is continuous on $\RR^\nu$.\\

\noindent (b) Since the diffusion matrix is the identity matrix and $b$ belongs to $C^0(\RR^\nu,\RR^\nu)$, we can establish the existence of a unique diffusion stochastic process (see \cite{Guikhman1980} Ch. VIII, Sec. 2;  \cite{Ikeda1981} Ch. IV, Secs. 2, 3, and 5; or \cite{Friedman2006} Ch. V), provided that for all $\bfy$ and $\bfy'$ in $\RR^\nu$, we have
\begin{equation}\label{eq5.9}
\Vert \bfb(\bfy)-\bfb(\bfy')\Vert \, \leq \, c\, \Vert \bfy - \bfy'\Vert \quad , \quad \Vert\bfb(\bfy)\Vert \, \leq \,
C\, (1 +\Vert\bfy\Vert)\, .
\end{equation}
Using Eqs.~\eqref{eq2.3}, \eqref{eq2.7}, and \eqref{eq3.3}, $\bfb(\bfy)$ can be rewritten, for all $\bfy$ in $\RR^\nu$, as
\begin{equation}\label{eq5.10}
\bfb(\bfy) =\frac{1}{2} \, \xi(\bfy)^{-1} \nabla\xi(\bfy) \quad ,\quad \xi(\bfy) = \frac{1}{n_d}
\sum_{j=1}^{n_d} \exp\left \{ -\frac{1}{2\hat s^2} \Vert \frac{\hat s}{s} \bfeta^j - \bfy \Vert^2 \right\} \, > \, 0 \, .
\end{equation}
Calculating $\nabla\xi(\bfy)$, it can be seen that
\begin{equation}\nonumber
2\,\Vert\bfb(\bfy)\Vert \, \leq \, \frac{1}{\xi(\bfy)\, n_d \, \hat s^2}
\sum_{j=1}^{n_d} \frac{\hat s}{s}\, \Vert\bfeta^j\Vert\, \exp\left \{ -\frac{1}{2\hat s^2} \Vert \frac{\hat s}{s} \bfeta^j - \bfy \Vert^2 \right\} + \frac{1}{\hat s^2} \, \Vert\bfy\Vert \, .
\end{equation}
Since $n_d$ is finite and since $E\{\Vert\bfH\Vert^2\} = \frac{1}{n_d}\, \sum_{j=1}^{n_d} \Vert\bfeta^j\Vert^2 = \nu$ is also finite, we have
$ \sup_{j=1,\ldots,n_d} \Vert\bfeta^j\Vert = c_\eta < +\infty$ and therefore,
\begin{equation}\nonumber
2\,\Vert\bfb(\bfy)\Vert \, \leq \, \frac{1}{2} \left ( \frac{1}{\hat s^2} \frac{\hat s}{s}\, c_\eta + \frac{1}{\hat s^2} \, \Vert\bfy\Vert\right )\, \leq \, c\, (1 + \Vert\bfy\Vert)\quad , \quad 0 < c < +\infty\, .
\end{equation}
The second inequality in Eq.~\eqref{eq5.9} is then proven. It can now be verified that
\begin{equation}\nonumber
\Vert \bfb(\bfy)-\bfb(\bfy')\Vert \,\, \leq \, \frac{1}{2\hat s^2} (\, \Vert\bff(\bfy) - \bff(\bfy')\Vert \, + \, \Vert\bfy-\bfy'\Vert \, ) \, ,
\end{equation}
in which $\bff(\bfy)= \bfa(\bfy)/\beta(\bfy)$ where $\bfa(\bfy)$ and $\beta(\bfy)$ are written as
\begin{equation}\nonumber
\bfa(\bfy) = \frac{\hat s}{s}\sum_{j=1}^{n_d} \bfeta^j\, \exp\left \{ -\frac{1}{2\hat s^2} \Vert \frac{\hat s}{s} \bfeta^j - \bfy \Vert^2 \right\} \quad , \quad \beta(\bfy) = \sum_{j=1}^{n_d} \exp\left \{ -\frac{1}{2\hat s^2} \Vert \frac{\hat s}{s} \bfeta^j - \bfy \Vert^2 \right\}\, .
\end{equation}
We must have $\Vert\bff(\bfy) - \bff(\bfy')\Vert \, \leq c\, \Vert\bfy-\bfy'\Vert$, that is true if
$\Vert\,  [\nabla\bff(\bfy)]\,  \Vert_F \, \leq \, \tilde c < +\infty$, in which the $(\nu\times\nu)$-real matrix $[\nabla\bff(\bfy)]$ is written as
$[\nabla\bff(\bfy)] =\beta(\bfy)^{-1}\,[\nabla\bfa(\bfy)] - \beta(\bfy)^{-2}\, \bfa(\bfy)\otimes\nabla\beta(\bfy)$, and where $\Vert\, \cdot\, \Vert_F$ is the Frobenius norm. Introducing, temporary, the notation $e_j = \exp\left \{ -\frac{1}{2\hat s^2} \Vert \frac{\hat s}{s}\bfeta^j - \bfy \Vert^2 \right\}$, we have
\begin{equation}\nonumber
[\nabla\bff(\bfy)] = \frac{1}{\hat s^2}\, \frac{\hat s^2}{s^2} \frac
{\sum_{j=1}^{n_d} \sum_{j^\prime=1}^{n_d} e_j e_{j^\prime} (\bfeta^j\otimes\bfeta^j - \bfeta^j\otimes\bfeta^{j^\prime})}
{\sum_{j=1}^{n_d} \sum_{j^\prime=1}^{n_d} e_j e_{j^\prime}} \, .
\end{equation}
It can then be deduced that,
\begin{equation}\nonumber
\Vert \, [\nabla\bff(\bfy)] \,\Vert_F \, \leq \,  \frac{1}{\hat s^2}\, \frac{\hat s^2}{s^2}
\sup_{j,j^\prime} \, \Vert \bfeta^j\otimes\bfeta^j - \bfeta^j\otimes\bfeta^{j^\prime}\Vert_F \, .
\end{equation}
As previously, since $\sup_j \,\Vert\bfeta^j\Vert = c_\eta < +\infty$, we have $\Vert \, [\nabla\bff(\bfy)] \,\Vert_F \, \leq \, \tilde c_F$ with $\tilde c_F$  a finite positive constant independent of $\bfy$, that yields the proof for the Lipchitz continuity of $\bfb$.\\

\noindent (c) We will admit that, for all $t >0$ and for all $\bfx$ in $\RR^\nu$, the transition probability measure has a density $\bfy\mapsto\rho(\bfy,t\, \vert\, \bfx,0)$ on $\RR^\nu$ with respect to $d\bfy$, and that $(\bfy,\bfx)\mapsto \rho(\bfy,t\, \vert\, \bfx,0)$ is continuous on $\RR^\nu\times\RR^\nu$ (for all $t > 0$) (see for instance \cite{Khasminskii1980} Ch III, Secs. 6 to 9, or \cite{Stroock1979} Ch. 10).\\

\noindent (d)  As a diffusion stochastic process, the trajectories are almost surely continuous functions. Since $\bfb$ is continuous
and $\Vert\bfb(\bfy)\Vert \, \leq \, c\, (1+\Vert\bfy\Vert)$ implies $\Vert\bfb(\bfy)\Vert^2\, \leq \, 2c^2\, (1+\Vert\bfy\Vert^2)$, and since $E\{\Vert\bfY(0)\Vert^2\} = \Vert\bfx\Vert^2 < +\infty$ for all fixed $\bfx$ in $\RR^\nu$, we have Eq.~\eqref{eq5.6} (see \cite{Ikeda1981} Ch. IV, Sec. 2).\\

\noindent (e) The existence of the asymptotic stationary solution with Eq.~\eqref{eq5.7} can be found in \cite{Khasminskii1980} Ch. III, or \cite{Soize1994} Ch. VI, Secs. 5 and 6.
\end{proof}

\subsection{Rewriting the It\^o equation in a matrix form}
\label{Section5.2}
Let $\{\bfY^j(t),t\geq 0\}$ be the solution of Eqs.~\eqref{eq5.1} and \eqref{eq5.2} with the initial condition $\bfY^j(0) =\bfeta^j\in\RR^\nu$, in which $\bfeta^j$ is defined in Section~\ref{Section2}. We then have
\begin{equation}\label{eq5.11}
\rho(\bfy,t\, \vert \, \bfeta^j,0)\, d\bfy = P_{\bfY^j(t)\, \vert \, \bfY^j(0)}(d\bfy,t\, \vert \, \bfeta^j,0)\, .
\end{equation}
Let $\{ [ \bfY(t) ]\,, t \geq 0\}$ be the $\MM_{\nu,n_d}$-valued stochastic process and $[\eta_d]$ be the matrix in $\MM_{\nu,n_d}$ such that
\begin{equation}\label{eq5.12}
[ \bfY(t) ] = [\bfY^1(t) \ldots \bfY^{n_d}(t) ] \quad , \quad [\eta_d] = [\bfeta^1 \ldots \bfeta^{n_d}] \, .
\end{equation}
Therefore, $\{ [ \bfY(t) ]\,, t \geq 0\}$ is solution of the matrix-valued ISDE,
\begin{align}
d[\bfY(t)] = & \,\frac{1}{2} \,[L([\bfY(t)])]\, dt + d[\bfW(t)] \quad , \quad t > 0 \, , \label{eq5.13}\\
[\bfY(0)] = & \, [\eta_d] \, , \, a.s. \, ,  \label{eq5.14}
\end{align}
where $[y]\mapsto [L([y])]$ is the function from $\MM_{\nu,n_d}$ into $\MM_{\nu,n_d}$, defined, for $k\in\{1,\ldots , \nu\}$ and
$j\in \{1,\ldots , n_d\}$, by
\begin{equation}\label{eq5.15}
[L([y])]_{kj} = \frac{1}{\xi(\bfy^j)}\, \frac{\partial\xi(\bfy^j) }{\partial y^j_k} \, .
\end{equation}
In Eq.~\eqref{eq5.15}, $\bfy^j$ is the $j$-th column of $[y] = [\bfy^1 \ldots \bfy^{n_d}]$, where $\bfy^j = (y^j_1,\ldots , y^j_\nu)$. Additionally, $\xi(\bfy)$ is defined by Eq.~\eqref{eq5.10}, and $\{ [ \bfW(t) ]\,, t \geq 0\}$ is the normalized $\MM_{\nu,n_d}$-valued Wiener stochastic process.

\subsection{Time-discrete approximation of the ISDE and convergence analysis}
\label{Section5.3}
To estimate $[\widehat K(\nDeltat)]_{ij}$, defined by Eq.~\eqref{eq5.0}, using nonparametric statistics, we need to generate realizations of
Eqs.~\eqref{eq5.13} and \eqref{eq5.14}. Consequently, a first stage involves  introducing a time-discrete approximation of
Eq.~\eqref{eq5.13} and analyzing the convergence.\\

\noindent {\textit{(i) Time sampling $\Delta t$ and $\delta t$}.
We first define a time sampling ${\nDeltat , n\geq 1 }$ to be used for estimating $[\widehat K(\nDeltat)]$. The time step $\Delta t$
will be defined in Section~\ref{Section7}. However, $\Delta t$ may not be sufficiently small to achieve  a satisfactory
rate of convergence. We then introduce $\delta t \leq \Delta t$ such that $\Delta t = n_s \times \delta t$, with $n_s \geq 1$. To discretize Eq.~\eqref{eq5.13}, we use the time sampling $\{\mu\,\delta t, \mu \geq 1\}$.\\

\noindent {\textit{(ii) ISDE discretization}.
Assuming that $\delta t$ is sufficiently small relative to $1$, employing the Euler scheme (as seen, for example, in \cite{Kloeden1992}) to discretize the solution $\{ [ \bfY(t) ]\,, t \geq 0\}$ of Eqs.~\eqref{eq5.13} and \eqref{eq5.14} yields
\begin{align}
[\bfY_\mu] = & \,  [\bfY_{\mu -1}] + \frac{\delta t}{2}\, [L([\bfY_{\mu -1}])] + \sqrt{\delta t}\, [\bfGamma_\mu]
                \quad , \quad \mu \geq 1 \, , \label{eq5.16}\\
[\bfY_0] = & \, [\eta_d]\, , \, a.s.        \, ,  \label{eq5.17}
\end{align}
where $[\bfY_\mu]$ is the approximation of $[\bfY(t)]$ at $t = \mu\, \delta t$, and where
$\{[\bfGamma_\mu]_{kj}\, ; k=1,\ldots,\nu\, ; j = 1,\ldots , n_d\, ; \mu \geq 1\}$ is an infinite family of independent normalized Gaussian real-valued random variables.\\

\noindent {\textit{(iii) Convergence of the time-discrete approximation}.
For each $j$ in ${1,\ldots,n_d}$, Proposition~\ref{proposition:5} can be applied to the stochastic process $\{\bfY^j(t), t\geq 0\}$, where $\bfY^j(t)$ represents the $j$-th column of $[\bfY(t)]$.
Let $\{ [\bfY^\delta(t)],t\geq 0\}$ be the time-discrete approximation of $\{ [\bfY(t)],t\geq 0\}$. Take any fixed positive real number $T$,  let $\mu_T = \text{int} (T/\delta t)$ denotes the nearest integer to $T/\delta t$. It can be observed that as $\delta t$ approaches $0$, $\mu_T$ tends to $+\infty$. The following classical result holds (see, for instance, \cite{Kloeden1992}, Page 323).
%
\begin{lemma}[Strong convergence] \label{lemma:4}
Under Proposition~\ref{proposition:5}, the time-discrete approximation $\{ [\bfY^\delta(t)],t\geq 0\}$ of $\{ [\bfY(t)],t\geq 0\}$ converges strongly  at time $T$ if
\begin{equation}\label{eq5.18}
\lim_{\delta t \rightarrow 0} E\{\, \Vert\,  [\bfY_{\mu_T}] - [\bfY(T)]\,\Vert^2_F \} = 0 \, ,
\end{equation}
where $[\bfY_{\mu_T}]$ is determined by Eqs.~\eqref{eq5.16} and \eqref{eq5.17} up to $\mu=\mu_T$.
\end{lemma}

\noindent \textit{(iv) Criterion to determine if $\delta t$ is small enough}.
We can employ a criterion based on weak convergence, which is related to the covariance matrix $[C_{\bfY^j(t)}] \in \MM_\nu^+$ of the $\RR^\nu$-valued random variable $\bfY^j(t)$, for $t$ fixed in the interval $]0, T]$.
Let $\sigma_\bfY^2(t)$ be defined by
\begin{equation}\label{eq5.19}
\sigma_\bfY^2(t) = E\{\, \Vert \,[\bfY(t)] - E\{[\bfY(t)]\} \,  \Vert_F^2 \} = \sum_{j=1}^{n_d}\sum_{k=1}^\nu
               E\{ (Y_k^j(t) - E\{Y_k^j(t) \} )^2\} \, .
\end{equation}
It can easily be seen that
\begin{equation}\label{eq5.20}
\sigma_\bfY^2(t) = \sum_{j=1}^{n_d}\tr [C_{\bfY^j(t)}] \, .
\end{equation}
Let $\sigma_{\bfY_{\mu_T}}^2$  be the corresponding quantity for random matrix $[\bfY_{\mu_T}]$,
\begin{equation}\label{eq5.21}
\sigma_{\bfY_{\mu_T}}^2 = E\{\, \Vert \,[\bfY_{\mu_T}] - E\{[\bfY_{\mu_T}]\} \, \Vert_F^2 \} = \sum_{j=1}^{n_d}  \tr [C_{\bfY^j_{\mu_T}}]\, .
\end{equation}
The criterion can then be based on the following properties,
\begin{equation}\label{eq5.22}
\lim_{\delta t \rightarrow 0} \,\, \vert \, \sigma_{\bfY_{\mu_T}} - \sigma_{\bfY}(T)\,  \vert\,  = \, 0 \, .
\end{equation}
We will detailed this criterion in paragraph (vi).\\

\noindent {\textit{(v) Generation of independent realizations of the time-discrete approximation}.
Let $n_s\geq 1$ and $N > 1$ be two fixed integers. Let $\Delta t$ be fixed and let $\delta t = \Delta t / n_s$. Note that $\mu_T$ introduced in Section~\ref{Section5.3}-(iv) is such that $\mu_T = n_s\times N$. We define $T$, $\curN$, and $\curN_\mu$ by,
\begin{equation}\label{eq5.23}
T = N \times \Delta t = n_s\times N \times \delta t \quad , \quad \curN =\{1,\ldots , N\} \quad , \quad \curN_\mu = \{ 1,\ldots , n_s \times N\}\, .
\end{equation}
For all $\mu$ in $\curN_\mu$, let $\{ [\gamma_\mu^\ell]\in\MM_{\nu,n_d}\, , \ell=1,\ldots, n_\MC \}$ be $n_\MC$ independent
realizations of the random matrix $[\bfGamma_\mu]$. As discussed in  Section~\ref{Section5.3}-(ii), it can be seen that the family
$\{ [\gamma_\mu^\ell]\, , \ell=1,\ldots, n_\MC \, ; \mu\in\curN_\mu\}$ consists of $n_\MC \times n_s\times N$ independent realizations. For each $\ell$ in $\{1,\ldots , n_\MC\}$, the realization $\{ [\tilde y_\mu^\ell]\, , \mu\in\curN_\mu\}$ of the time-discrete approximation
$\{ [\bfY_\mu]\, , \mu\in\curN_\mu\}$ is computed using the following recurrence (refer to Eqs.~\eqref{eq5.16} and \eqref{eq5.17}),
\begin{align}
[\tilde y_\mu^\ell] = & \, [\tilde y_{\mu-1}^\ell] + \frac{\delta t}{2}\, [L([\tilde y_{\mu-1}^\ell])] + \sqrt{\delta t}\, [\gamma_\mu^\ell]
                \quad , \quad \mu \in\curN_\mu \, , \label{eq5.24}\\
[\tilde y_0^\ell] = & \, [\eta_d]                    \, ,  \label{eq5.25}
\end{align}
In the following, for performing the statistical estimations, we will use the subsequence
$\{ [y_n^\ell]\, , n\in\curN\}$ of $\{ [\tilde y_\mu^\ell]\, , \mu\in\curN_\mu\}$, such that
\begin{equation}\label{eq5.26}
\forall n \in\curN \quad , \quad [y_n^\ell] = [\tilde y_\mu^\ell] \quad , \quad  \mu = n_s\times n\, .
\end{equation}

\noindent {\textit{(vi) Practical criteria for controlling the convergence parameters}.
Let $\Delta t$ be fixed, as well as $n_s$, meaning $\delta t$ is fixed, to ensure satisfaction of the  convergence criteria introduced in Section~\ref{Section5.3}-(iv). An additional practical criterion can be applied to verify the adequacy of all the convergence parameters, as provided in Lemma~\ref{lemma:5}.
%
\begin{lemma}[Practical criteria for controlling the convergence] \label{lemma:5}
For $k\in\{1,\ldots , \nu\}$, $j\in\{1,\ldots, n_d\}$, and $n\in\curN$, let $[\underline y_{\,n}]$ and $[\sigma_n]$ be the matrices in $\MM_{\nu,n_d}$ defined by,
\begin{equation} \label{eq5.27} 
[\underline y_{\,n}]_{kj} =  \frac{1}{n_\MC} \sum_{\ell=1}^{n_\pMC} [y_n^\ell]_{kj} \quad , \quad
([\sigma_n]_{kj})^2 =  \frac{1}{n_\MC} \sum_{\ell=1}^{n_\pMC} ([y_n^\ell]_{kj} - [\underline y_{\,n}]_{kj}) ^2\, .
\end{equation}
Let $n\mapsto \underline y(n)$ and $n\mapsto \underline\sigma(n)$ be the positive-valued functions defined, for $n\in\curN$, by
\begin{equation} \label{eq5.29}
\underline y(n)  = \frac{1}{\sqrt{\nu \times n_d}} \,\, \Vert \,[\underline y_{\,n}] \, \Vert_F
 \quad , \quad \underline\sigma(n) = \frac{1}{\sqrt{\nu\times n_d}} \,\, \Vert \,[\sigma_n] \, \Vert_F \, .
\end{equation}
Then, for $\delta t \rightarrow 0$ and $n_\MC\rightarrow +\infty$, we have
\begin{equation} \label{eq5.30}
\lim_{N\rightarrow +\infty} \, \underline y(n)  = 0 \quad , \quad
\lim_{N\rightarrow +\infty} \,\underline\sigma(n) = 1 \, .
\end{equation}
\end{lemma}
%
\begin{proof} (Lemma~\ref{lemma:5}).
It can be seen that $[\underline y_{\,n}]_{kj}$ and $([\sigma_n]_{kj})^2$ are the empirical estimates of the mean value and the variance of the time-discrete approximation $[\bfY_\mu]_{kj}$ for $\mu = n_s\times n$ of $Y_k^j(\nDeltat)$. From Proposition~\ref{proposition:5}, we know that, for $t\rightarrow +\infty$, $\{\bfY(t)\, , t\geq 0\}$, and consequently, $\{\bfY^j(t)\, , t\geq 0\}$, is asymptotically stationary and that Eq.~\eqref{eq5.7} holds. Since $E\{\bfH\} = \bfzero_\nu$ and $E\{\bfH\otimes \bfH \} = [I_\nu]$ (see Eqs.~\eqref{eq2.5} and \eqref{eq2.6}), if
$\delta t \rightarrow 0$ and $n_\MC\rightarrow +\infty$, we have for $N\rightarrow +\infty$,
$[\underline y_{\,n}]_{kj}\rightarrow E\{H_k\} = 0$ and  $([\sigma_n]_{kj})^2 \rightarrow [C_d]_{kk} = 1$.
Since $\Vert\,[\underline y_{\,n}]  \,  \Vert_F^2 = \sum_{k=1}^\nu \sum_{j=1}^{n_d} ( [\underline y_{\,n}]_{kj})^2$ and
$\Vert\,[\sigma_n]  \,  \Vert_F^2 = \sum_{k=1}^\nu \sum_{j=1}^{n_d} ( [\sigma_n]_{kj})^2$, using the same normalization $1/\sqrt{\nu\times n_d}$ for $\underline y(n)$ and $\underline\sigma(n)$, we obtain Eq.~\eqref{eq5.30}.
\end{proof}
%

\subsection{Estimation of the matrix of the finite approximation}
\label{Section5.4}
We now can estimate $[\widehat K(\nDeltat)]_{ij}$ defined by Eq.~\eqref{eq5.0} using the realizations $\{ [y_n^\ell],\ell=1,\ldots , n_\MC\}$ for $n\in\curN$, defined by Eq.~\eqref{eq5.26}.
%
\begin{proposition}[Estimation of matrix $\lbrack \widehat K(\nDeltat)\rbrack$] \label{proposition:6}
Under Propositions~\ref{proposition:4} and \ref{proposition:5}, results and notations of Section~\ref{Section5.3}-(v), for
$n \in \curN$, an estimate of matrix $[\widehat K(\nDeltat)]$,  whose entries are defined by Eq.~\eqref{eq4.36}, is written as
\begin{equation}\label{eq5.31}
[\widehat K(\nDeltat)] = [\widehat B]^{-1}\, [\widehat\curK(\nDeltat)]\, ,
\end{equation}
in which $[\widehat B]$ is the diagonal matrix in $\MM_{n_d}$ for which its entries are
\begin{equation}\label{eq5.31bis}
[\widehat B]_{ij} = \delta_{ij}\, \sum_{j'=1}^{n_d} \exp  \left\{ -\frac{1}{2 {\hat{s}}^{\,2}} \Vert \bfeta^i - \frac{\hat s}{s} \bfeta^{j'}\Vert^2 \right\} \, .
\end{equation}
The entries of matrix $[\widehat\curK(\nDeltat)]$ in  $\MM_{n_d}$ are written as,
\begin{equation}\label{eq5.31ter}
[\widehat\curK(\nDeltat)]_{ij} =
 \frac{1}{n_\pMC}\sum_{\ell=1}^{n_\pMC}\left ( \frac{ {\hat{s}}^{\,\nu} }{ s^\nu_\SB \prod_{k=1}^\nu [\sigma_n]_{kj}} \right )
  \exp  \left\{ -\frac{1}{2} \sum_{k=1}^\nu \left ( \frac{\eta_k^i - [y_n^\ell]_{kj}}{s_\SB [\sigma_n]_{kj}} \right )^2  \right\} \, ,
\end{equation}
in which $s$ and $\hat s$ are defined by Eq.~\eqref{eq2.4}, where $s_\SB$ is written as
\begin{equation}\label{eq5.32}
s_\SB = \left\{\frac{4}{n_\MC(2+\nu)} \right\}^{1/(\nu +4)} \, ,
\end{equation}
where $[y_n^\ell]\in\MM_{\nu,n_d}$ is defined by Eq.~\eqref{eq5.26}, and where $[\sigma_n]\in \MM_{\nu,n_d}$ is defined by Eq.~\eqref{eq5.27}.
\end{proposition}

\begin{proof} (Proposition~\ref{proposition:6}).
Eq.~\eqref{eq5.31} is deduced from Eq.~\eqref{eq5.0} with $t=\nDeltat$ and with $p_\bfH(\bfeta^i)$ given by Eq.~\eqref{eq2.3} in which $\bfeta=\bfeta^i$, and with the estimate of $\rho(\bfeta^i,\nDeltat\,\vert\, \bfeta^j,0)$ given by the Gaussian kernel-density estimation method (see \cite{Bowman1997,Duong2008}) with the Silverman bandwidth $s_\SB$ defined by Eq.~\eqref{eq5.32}, which is written as
\begin{equation}\nonumber
\rho(\bfeta^i,\nDeltat\,\vert\, \bfeta^j,0) =
 \frac{1}{n_\pMC}\sum_{\ell=1}^{n_\pMC} \frac{1}{(\sqrt{2\pi} s_\SB)^\nu  \prod_{k=1}^\nu [\sigma_n]_{kj}}
  \exp  \left\{ -\frac{1}{2} \sum_{k=1}^\nu \left ( \frac{\eta_k^i - [y_n^\ell]_{kj}}{s_\SB [\sigma_n]_{kj}} \right )^2  \right\} \, .
\end{equation}
\end{proof}
%
%
\begin{remark}  \label{remark:5}

(a) In Section~\ref{Section6}, we will provide an illustration by numerically solving the eigenvalue problem defined by Eq.~\eqref{eq4.38} for $t= n\, \Delta t$, using Eq.~\eqref{eq5.31}. Specifically, we consider the Gaussian case with dimension $\nu=1$. In this scenario, $\bfH$ represents a normalized Gaussian real-valued random variable. For this case, a reference solution is available.

\noindent (b) Nevertheless, it would be challenging to employ such a formulation with reasonable convergence in very high dimensions, where $\nu$ equals several tens or even hundreds. This would necessitate a large value of $n_d$.

\noindent (c) In fact, Proposition~\ref{proposition:6} will be employed in Section~\ref{Section7} to derive an expression that connects to the kernel $[\curK_\DM]_{ij} = \exp \left \{ -\frac{1}{4\varepsilon_\pDM} \Vert \bfeta^i-\bfeta^j \Vert^2 \right \}$, which is used to calculate the DMAPS basis and proves efficient for large values of $\nu$.
\end{remark}
%

\section{Numerical illustration of the proposed formulation}
\label{Section6}

As explained in Remark~\ref{remark:5}-(a), this section presents a numerical illustration of the formulation introduced in Section~\ref{Section5} to solve the approximated eigenvalue problem defined by Eq.\eqref{eq4.38}, which is derived from the eigenvalue problem in Eq.\eqref{eq4.21}. To validate the formulation, we select a reference case where the eigenvalue problem defined by Eq.~\eqref{eq4.21} can be exactly solved. This is feasible when the dimension $\nu$ of $\bfH$ is $1$ and $\bfH$ is a normalized Gaussian real-valued random variable.

\subsection{Reference case definition and explicit solution}
\label{Section6.1}
The quantities related to the reference case will be indexed by letter $r$. The probability density function of $H_r$ on $\RR$
is $p_{H_r}(y) = (2\pi)^{-1/2} \exp(-y^2/2)$. The potential function $\Phi_r$ (see Eq.~\eqref{eq2.7}) is $\Phi_r(y) = y^2/2$ and the drift $b_r$ defined by Eq.~\eqref{eq3.3} is $b_r(y)  = -y/2$. The ISDE defined by Eqs.~\eqref{eq5.1} and \eqref{eq5.2} are rewritten as
\begin{align}
 dY_r(t)  & =  -\frac{1}{2} Y_r(t)\, dt + dW_r(t) \quad , \quad t > 0\, , \label{eq6.1} \\
  Y_r(0)  & =  x \in \RR \, .  \label{eq6.2}
\end{align}
Consequently, $\{Y_r(t) , t\geq 0\}$ is a second-order Gaussian stochastic process, which is explicitely defined by
\begin{equation}\label{eq6.3}
Y_r(t) = x\, e^{-t/2} + \int_0^{\,t} e^{-(t-\tau)/2} \, dW_r(\tau) \, .
\end{equation}
A simple calculation shows that the mean value $m_r(t) = E\{Y_r(t)\}$ and the standard deviation
$\sigma_r(t) = (E\{(Y_r(t) -m_r(t))^2\})^{1/2}$ of the random variable $Y_r(t)$ for fixed $t > 0$, are written as
\begin{equation}\label{eq6.4}
m_r(t) = x\, e^{-t/2} \quad , \quad \sigma_r(t) = \sqrt{1-e^{-t}} \, .
\end{equation}
For all $t > 0$, as  $\{ Y_r(t)\, \vert \, Y_r(0) = x\}$ is a Gaussian random variable, the transition probability density function is
\begin{equation}\label{eq6.5}
\rho_r(y,t \,\vert \, x,0) = \frac{1}{\sqrt{2\pi}\, \sigma_r(t)} \exp\left\{ -\frac{1}{2\sigma_r(t)^2}(y-m_r(t))^2\right\} \, .
\end{equation}
Note that Eqs.~\eqref{eq6.4} and \eqref{eq6.5} show that, we effectively have
$\lim_{t\rightarrow 0_+} \rho_r(y,t \,\vert \, x,0)\, dy = \delta_0(y-x)$ (see Eq.~\eqref{eq5.5}) and
$\lim_{t\rightarrow +\infty} \rho_r(y,t \,\vert \, x,0) = p_{H_r}(y)$ (see Eq.~\eqref{eq5.7}).
It can be deduced that the kernel $k_{r,t}(y,x)$ on $\RR\times \RR$, defined by Eq.~\eqref{eq4.1} is written, for $t > 0$, as
\begin{equation}\label{eq6.6}
k_{r,t}(y,x) = \frac{1}{\sqrt{1 -e^{-t}}} \exp\left\{- \frac{1}{2} \frac{e^{-t}}{(1-e^{-t})} (y^2 + x^2 -2e^{t/2} y\, x)^2\right\} \, .
\end{equation}
Let $\{h_\alpha(y)\, , \alpha \in \NN \}$ be the Hermite polynomials and $\{h_\alpha(y)/\sqrt{\alpha!}\, , \alpha \in \NN \}$ the Hilbert basis in $L^2(\RR;p_{H_r})$,
\begin{equation}\nonumber
\int_\RR \frac{h_\alpha(y)}{\sqrt{\alpha!}}\, \frac{h_\beta(y)}{\sqrt{\beta!}} \, p_{H_r}(y) \, dy  = \delta_{\alpha\beta} \quad , \quad h_0(y)=1 \quad , \quad  h_1(y) = y \quad , \quad  h_{\alpha+1}(y) = y\, h_\alpha(y) - \frac{d  }{dy} h_\alpha(y)\, .
\end{equation}
Let $\{\hh_\alpha(y)\, , \alpha \in \NN \}$ be the polynomials defined by
\begin{equation}\nonumber
h_\alpha(y) = \frac{1}{\sqrt{2^\alpha}} \, \hh_\alpha(\frac{y}{\sqrt{2}}) \quad , \quad
\hh_\alpha(y) = \sqrt{2^\alpha}\, h_\alpha(\sqrt{2}\,y)  \, .
\end{equation}
We have the formula \cite{Hansen1975},
\begin{equation}\nonumber
\frac{1}{\sqrt{1 - 4a^2}} \exp\left\{ \frac{4 a}{1-4 a^2} (\hat x \, \hat y - a\hat x^2 - a \hat y^2)\right\} =
\sum_{\alpha\in\NN} \frac{a^\alpha}{\alpha !}\, \hh_\alpha(\hat x) \, \hh_\alpha(\hat y)\, .
\end{equation}
Taking $x= \sqrt{2}\,\hat x$, $y= \sqrt{2}\,\hat y$, and $a =\frac{1}{2}e^{-t/2} < 1/2$ for $t > 0$, Eq.~\eqref{eq6.6} can be rewritten as
\begin{equation}\label{eq6.7}
k_{r,t}(y,x) = \sum_{\alpha\in\NN} e^{-\lambda_{r,\alpha} t} \psi_{r,\alpha}(y)\, \psi_{r,\alpha}(x) \, ,
\end{equation}
\begin{equation}\label{eq6.8}
\lambda_{r,\alpha} = \frac{\alpha}{2}  \quad , \quad \psi_{r,\alpha}(y)= \frac{1}{\sqrt{2^\alpha \,\alpha !}} \, \hh_\alpha(\frac{y}{\sqrt{2}})
= \frac{1}{\sqrt{\alpha !}}  \, h_\alpha(y)\, ,
\end{equation}
and thus, for $\alpha$ and $\beta$ in $\NN$,
\begin{equation}\label{eq6.9}
\int_\RR  \psi_{r,\alpha}(y)\, \psi_{r,\,\beta}(y)\, p_{H_r}(y)\, dy = \delta_{\alpha,\,\beta} \, .
\end{equation}
Comparing Eqs.~\eqref{eq6.7} and \eqref{eq6.9}  with Eqs.~\eqref{eq4.14} and \eqref{eq4.10} shows that, for this reference case, the eigenvalues of the Fokker-Planck operator $L_\FKP$ are
\begin{equation}\label{eq6.10}
\lambda_{r,\alpha} = \frac{\alpha}{2} \quad  , \quad \alpha\in \NN\, .
\end{equation}
%
%
\subsection{Estimating the eigenvalues with the proposed numerical formulation}
\label{Section6.2}

\noindent (i) For the convergence analysis, we consider $10$ values of $n_d$ constituting the set $\curN_d = \{100, 300, 400, 800, 1000, 1200,$  $ 1500, 1800, 2000, 2200\}$. For each $n_d\in\curN_d$, the matrix $[\eta_d] =[\eta^1 \ldots \eta^{n_d}]\in \MM_{1,n_d}$ is generated with an adapted generator (instruction randn$(1,n_d)$ for Matlab Gaussian generator).

\noindent (ii) For each value of $n_d$ in $\curN_d$, the generation of $n_\MC = n_d$ independent realizations $\{ [\tilde y_\mu^\ell] , \mu\in\curN_\mu\}$ with $\curN_\mu = \{1,\ldots , n_s\times N\}$ is performed using Eqs.~\eqref{eq5.24} and \eqref{eq5.25} with $n_s=1$, $\delta t = \Delta t = 0.061796$, $N=150$. For $n_d=1200$ and $n_\MC = 1200$, Figs.~\ref{fig:figure1a} and \ref{fig:figure1b} display the graphs of functions $n\mapsto \underline{y}(n)$ and $n\mapsto \underline{\sigma}(n)$ computed  using Eq.~\eqref{eq5.29} with Eq.~\eqref{eq5.27}. It can be seen that the criterion defined by Eq.~\eqref{eq5.30} is satisfied.
\begin{figure}[!h]
    \centering
    \begin{subfigure}[b]{0.25\textwidth}
    \centering
        \includegraphics[width=\textwidth]{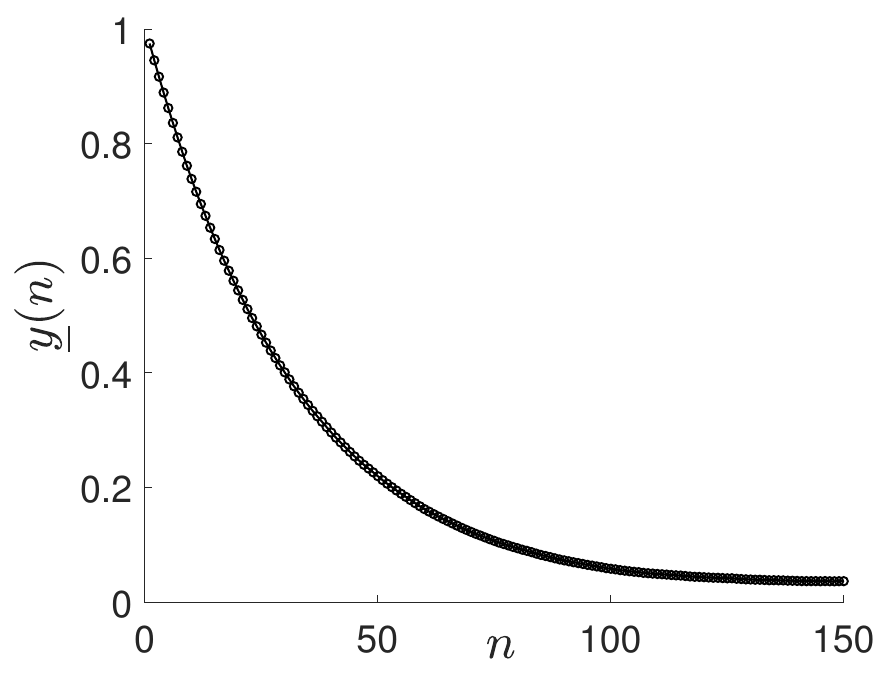}
        \caption{Graph of function $n\mapsto \underline{y}(n)$.}
        \label{fig:figure1a}
    \end{subfigure}
    \hfil
    \begin{subfigure}[b]{0.25\textwidth}
        \centering
        \includegraphics[width=\textwidth]{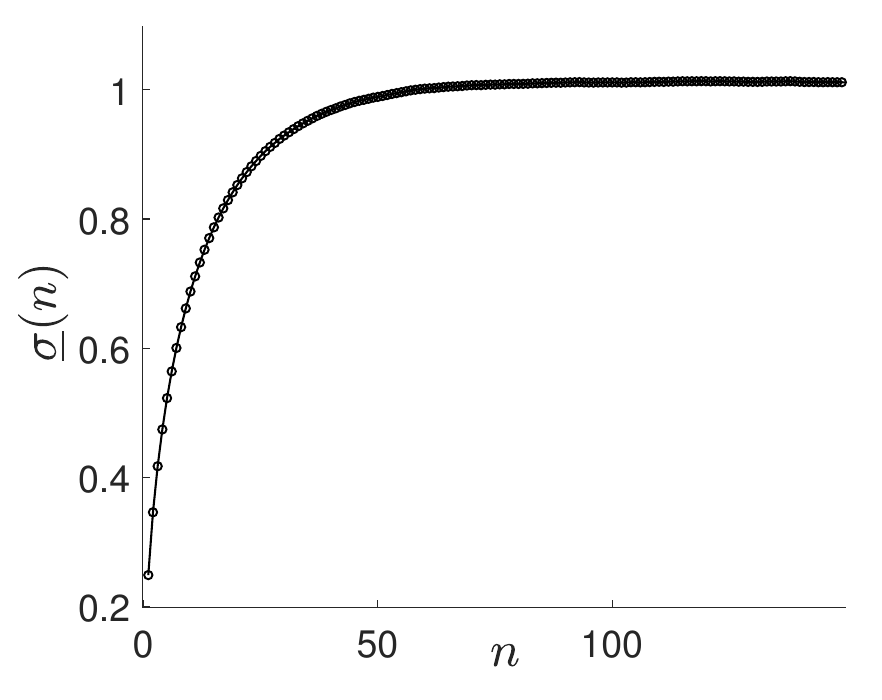}
        \caption{Graph of function $n\mapsto \underline{\sigma}(n)$.}
        \label{fig:figure1b}
    \end{subfigure}
    \caption{Criteria defined in Lemma~\ref{lemma:5} for controlling the convergence of the Gaussian-case reference with $\nu=1$, $n_s=1$, $n_d=1200$, and $n_{\rm{MC}} = 1200$.}
    \label{fig:figure1}
\end{figure}

\noindent (iii) For every $n_d\in\curN_d$, matrix $[\widehat K(2\, \Delta t)] \in\MM_{n_d}^{+0}$ is calculated using Eq.~\eqref{eq5.31} with
$\hat s = 0.2486$, $\hat s / s = 0.9690$, and $s_\SB =0.2565$. The first $6$ largest eigenvalues, $\hat b_\alpha(2\, \Delta t)$ of
$[\widehat K(2\, \Delta t)]$ and  $\lambdahat_\alpha(2\, \Delta t)$, simply denoted by $\lambdahat_\alpha$, are obtained from
Eq.~\eqref{eq4.41}.
Fig.~\ref{fig:figure2b} shows the graph of function $n_d\mapsto \errp_{\lambda}(n_d)$ that quantifies the relative error between $\alpha\mapsto\lambda_{r,\alpha}$ and $\alpha\mapsto \lambdahat_\alpha$, written as
$\errp_{\lambda}(n_d) = \sum_{\alpha=0}^5(\lambda_{r,\alpha}- \lambdahat_\alpha)^2 / \sum_{\alpha=0}^5 \lambda_{r,\alpha}^2$.
It can be seen that the error decreases as $n_d$ increases. For $n_d = n_\MC = 1200$, Fig.~\ref{fig:figure2a} compares the reference
eigenvalues $\lambda_{r,\alpha} =\alpha/2$ (see Eq.~\eqref{eq6.10}) with the computed eigenvalues $\lambdahat_\alpha$  for $\alpha=0,1, \ldots,5$. The comparison is good enough.
\begin{figure}[!h]
    \centering
    \begin{subfigure}[b]{0.25\textwidth}
    \centering
        \includegraphics[width=\textwidth]{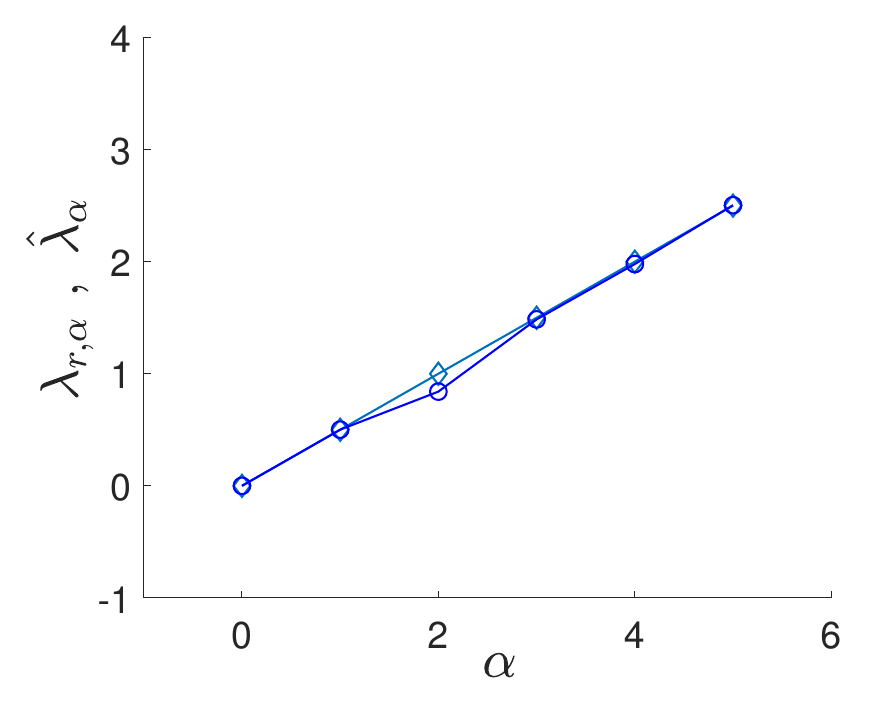}
        \caption{graphs $\alpha\mapsto\lambda_{r,\alpha}$ (diamond) and $\alpha\mapsto\lambdahat_\alpha$ (circle).}
        \label{fig:figure2a}
    \end{subfigure}
    \hfil
    \begin{subfigure}[b]{0.25\textwidth}
        \centering
        \includegraphics[width=\textwidth]{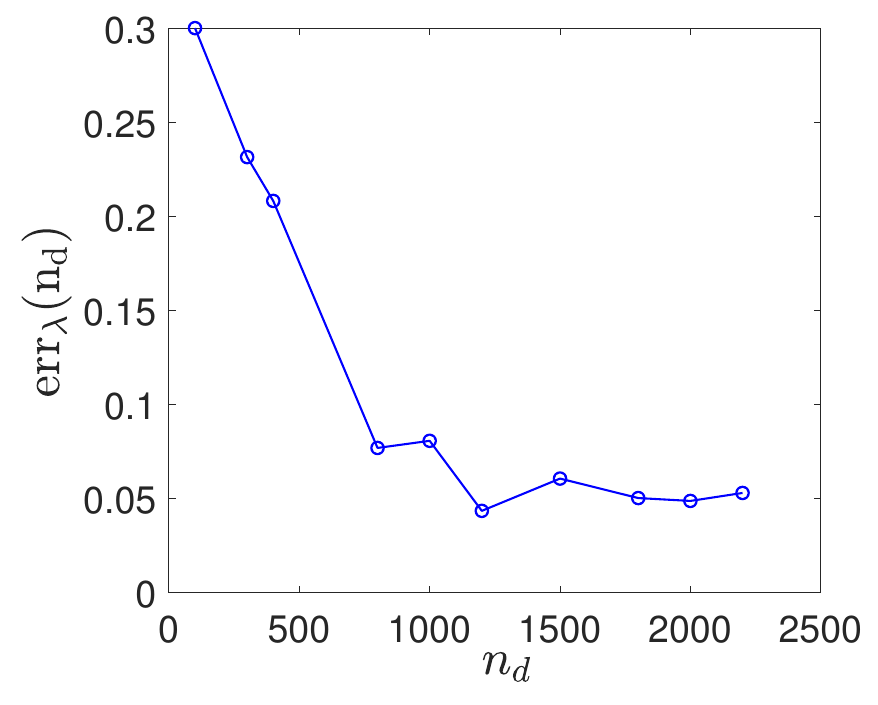}
        \caption{Graph $n_d\mapsto \errp_{\lambda}(n_d)$.}
        \label{fig:figure2b}
    \end{subfigure}
    \caption{(a) Comparison of the reference eigenvalues $\lambda_{r,\alpha}$ (diamond) with the computed eigenvalues $\lambdahat_\alpha$ (circle)  for $\alpha=0,1, \ldots,5$. (b) Graph of $n_d\mapsto \errp_{\lambda}(n_d)$ quantifying the relative error between $\alpha\mapsto\lambda_{r,\alpha}$ and $\alpha\mapsto\lambdahat_\alpha$.}
    \label{fig:figure2}
\end{figure}

%
\section{Vector basis for PLoM derived from the transient anisotropic kernel, connected to the DMAPS basis}
\label{Section7}
%
We revisit the objective presented in Section~\ref{Section1}.
The diffusion-maps (DMAPS) basis, used by PLoM \cite{Soize2016,Soize2020c,Soize2022a}, is associated with the isotropic kernel detailed in \cite{Coifman2005,Coifman2006}. This section introduces the construction of a transient vector basis, based on the transient anisotropic
kernel described in Section~\ref{Section5} (see Proposition~\ref{proposition:6}). This approach incorporates the requirement that as $\Delta t \rightarrow 0$, the transient kernel at the first time $\Delta t$ coincides with the DMAPS isotropic kernel. Consequently, the two vector bases will be linked asymptotically as $t\rightarrow 0$. To develop such a "connected" transient anisotropic kernel to the DMAPS isotropic kernel, we undertake a reparameterization of Eq.~\eqref{eq5.31}, which defines the matrix $[\widehat K(\nDeltat)]\in\MM_{n_d}^{+0}$ for $n\in\curN$. This reformulation enables the construction of the vector basis for high dimensions (large values of $\nu$) and relatively small $n_d$, as implemented in PLoM, designed for probabilistic learning with small training datasets.
%
\subsection{Transient anisotropic kernel connected to the DMAPS isotropic kernel}
\label{Section7.1}
The following definition provides the construction of a transient anisotropic kernel that is linked to the DMAPS isotropic kernel. This construction draws inspiration from the expression of $[\widehat K(\Delta t)]$ defined by Eq.~\eqref{eq5.31}.

%
\begin{definition}[Transient anisotropic kernel connected to DMAPS] \label{definition:5}
For each $n$ in $\curN$, for $\Delta t$ defined in Section~\ref{Section5.3}-(i), for $n_\MC$ and $[y_n^\ell]$ defined in
Section~\ref{Section5.3}-(v), for $[\sigma_n]$ defined by Eq.~\eqref{eq5.27}, and for  a given real $\varepsilon_\DM > 0$, we define the
matrix $\tilde K(\nDeltat)]\in\MM_{n_d}^{+0}$, such that
\begin{equation} \label{eq7.1}
[\tilde K(\nDeltat)] = [B]^{-1} [\curK(\nDeltat)]\, ,
\end{equation}
such that, for all $i$ and $j$ in $\{1,\ldots,n_d\}$, the matrix $[\curK(\nDeltat)]\in\MM_{n_d}$ has entries,
\begin{equation}\label{eq7.2}
[\curK(\nDeltat)]_{ij} =
 \frac{1}{n_\MC}\sum_{\ell=1}^{n_\pMC}\left ( \Pi_{k=1}^\nu \{[\sigma_n]_{kj}/\Delta t\} \right )^{-1}
        \exp  \left\{ -\frac{1}{4\varepsilon_\DM} \sum_{k=1}^\nu \left ( \frac{\eta_k^i - [y_n^\ell]_{kj}}{[\sigma_n]_{kj}/\sqrt{\Delta t}} \right )^2  \right\}
\, ,
\end{equation}
and $[B]\in\MM_{n_d}$ is a diagonal matrix whose entries are,
\begin{equation}\label{eq7.3}
[B]_{ij} = \delta_{ij} \sum_{j'=1}^{n_d} \exp \left\{ -\frac{1}{4\varepsilon_\DM} \Vert\bfeta^i -\bfeta^{j'}\Vert^2 \right\} \, .
\end{equation}
\end{definition}
%
%
\begin{remark} [About the choice of the smoothing parameter $\varepsilon_\DM$] \label{remark:6}
An optimal value, $\varepsilon_\optp$, for the smoothing parameter $\varepsilon_\DM$, is proposed for the PLoM algorithm in \cite{Soize2022a} and is detailed in \ref{SectionA.2}-(i). This optimal value enables the analysis of high-dimensional problems (large value of $\nu$).
It should be noted that Proposition~\ref{proposition:7} will explain the definition of $[\tilde K(\nDeltat)]$ as defined by Eqs.~\eqref{eq7.1} to \eqref{eq7.3}, showing the connection with DMAPS.
\end{remark}

%
\begin{proposition}[Limit of $\lbrack\tilde K(\Delta t)\rbrack$ for $\Delta t \rightarrow 0$] \label{proposition:7}
We use the notation introduced in Definition~\ref{definition:5}. Let $[\curK_\DM]\in\MM_{n_d}^{+0}$ be the matrix of the DMAPS isotropic kernel, defined for all  $i, j \in \{1,\ldots , n_d\}$, by
\begin{equation} \label{eq7.4}
[\curK_\DM]_{ij} = \exp\left \{-\frac{1}{4\,\varepsilon_\DM} \Vert\bfeta^i-\bfeta^j\Vert^2\right\}\, .
\end{equation}
Let $[K_\DM]\in\MM_{n_d}$ be the matrix defined by
\begin{equation} \label{eq7.5}
[K_\DM] = [B]^{-1} [\curK_\DM]\, ,
\end{equation}
where $[B]$ is the diagonal matrix defined by Eq.~\eqref{eq7.3}. Let $\Delta t$ be defined by
\begin{equation} \label{eq7.6}
\Delta t = \hat s^2 / \kappa \quad , \quad \kappa \geq 1 \, ,
\end{equation}
where $\hat s$ is defined by Eq.~\eqref{eq2.4}. Then, as $\kappa \rightarrow +\infty$, and consequently $\Delta t \rightarrow 0$, we have
\begin{equation} \label{eq7.7}
[\tilde K(\Delta t)] \rightarrow [K_\DM] \, ,
\end{equation}
which means that, for $n=1$, the matrix $[\tilde K(n\Delta t)]$ of the transient anisotropic kernel, defined by Eq.~\eqref{eq7.1}, converges to the matrix $[K_\DM]$ of the DMAPS isotropic kernel as $\Delta t \rightarrow 0$.
\end{proposition}
%

\begin{proof} (Proposition~\ref{proposition:7}).
For $n=1$, for $\kappa\rightarrow +\infty$, that is to say, for $\Delta t \rightarrow 0$, Eqs.~\eqref{eq5.1} and \eqref{eq5.2} show that, for all $\ell\in\{1,\ldots,n_\MC\}$, $k\in\{1,\ldots,\nu\}$, and $j\in\{1,\ldots , n_d\}$, we have $[y_1^\ell]_{kj} \rightarrow \eta_k^j$ and $[\sigma_n]_{kj} \rightarrow \sqrt{\Delta t}$, and consequently, $[\curK(\Delta t)] \rightarrow [\curK_\DM]$. From Eqs.~\eqref{eq7.1} and \eqref{eq7.5}, it can be deduced Eq.~\eqref{eq7.7}.
\end{proof}
%
\subsection{Construction of the reduced-order transient basis and its counterpart for the DMAPS basis}
\label{Section7.2}
In this section, we: (i) review the construction of the reduced-order diffusion-maps (DMAPS) basis (RODB) represented by a matrix
$[g_\DM]\in \MM_{n_d,m_\optpp}$ with $m_\optp < n_d$, and (ii) construct, for each $n\in\curN$, the reduced-order transient basis (ROTB$(\nDeltat)$) represented by a matrix $[g(n\Delta t)]\in \MM_{n_d,m_\optpp}$.\\

\noindent \textit{(i) Reminder of the construction of the RODB}.
This construction, due to \cite{Coifman2005}, is the one used by the PLoM algorithm \cite{Soize2016,Soize2020c}.
The matrix $[K_\DM]$, defined by Eq.~\eqref{eq7.5}, has positive entries and represents the transition matrix of a Markov chain.
The eigenvalues $b_{\DM,1},\ldots,b_{\DM,n_d}$ and the associated eigenvectors $\bfg_\DM^1,\ldots,\bfg_\DM^{n_d}$ are such that
\begin{equation} \label{eq7.8}
[K_\DM]\, \bfg_\DM^{\,\beta} = b_{\DM,\,\beta}\, \bfg_\DM^{\,\beta} \quad , \quad
1=b_{\DM,1} > b_{\DM,2} \geq \ldots \geq b_{\DM,n_d}\, .
\end{equation}
Using Eq.~\eqref{eq7.5}, this eigenvalue problem  is rewritten as the symmetric eigenvalue problem,
\begin{equation} \label{eq7.9}
\PP^S_\DM \, \bfphi^{\,\beta} = b_{\DM,\,\beta}\, \bfphi^{\,\beta} \quad , \quad
\langle \bfphi^{\,\beta},\bfphi^{\,\beta'}\rangle =\delta_{\beta\beta'} \quad , \quad
\bfg_\DM^{\,\beta} = [B]^{-1/2}\bfphi^{\,\beta} \, ,
\end{equation}
in which $\PP^S_\DM = [B]^{-1/2}[\curK_\DM]\,[B]^{-1/2}$ is a symmetric matrix.
The diffusion-maps basis $\{\bfg_\DM^1,\ldots, \bfg_\DM^{n_d}\}$ forms a vector basis of $\RR^{n_d}$.
As explained in \cite{Soize2016,Soize2020c}, PLoM uses the RODB of order $m$, which  is defined by
$\{\bfg_\DM^1 \ldots  \bfg_\DM^{\,m}\}$.
This basis depends on  $\varepsilon_\DM$ and $m$.
The optimal value $m_\optp$ of $m$ is defined (see \cite{Soize2022a}) by
\begin{equation} \label{eq7.10}
m_\optp = \nu +1 \, .
\end{equation}
The optimal value $\varepsilon_\optp$ of $\varepsilon_\DM$ is estimated to obtain
\begin{equation} \label{eq7.11}
1 = b_{\DM,1} > b_{\DM,2} \simeq \ldots \simeq b_{\DM,m_\optpp}
 \gg b_{\DM,m_\optpp + 1}\geq \ldots \geq b_{\DM,n_d} > 0\, ,
\end{equation}
in which the jump amplitude
\begin{equation} \label{eq7.12}
J_\DM = b_{\DM,m_\optpp + 1}  /b_{\DM,m_\optpp} \, ,
\end{equation}
must be equal to $J_\DM=0.1$ (following \cite{Soize2020c}), but which can also be chosen in the interval $[0.1\, , 0.5]$ when $\nu$ is large.
Therefore, the RODB is defined for $m = m_\optp$ and is represented by the matrix,
\begin{equation} \label{eq7.13}
[g_\DM] = [\bfg_\DM^1 \ldots  \bfg_\DM^{m_\optpp}] \in \MM_{n_d,m_\optpp} \, .
\end{equation}
%

\noindent \textit{(ii) Construction of the ROTB$(\nDeltat)$}.
For each $n$ in $\curN$, the ROTB$(\nDeltat)$ represented by $[g(n\Delta t)]$ is constructed as the eigenvectors of matrix
$[\tilde K(n\Delta t)]$ defined by Eq.~\eqref{eq7.1}. Taking into account Eq.~\eqref{eq7.7}, we want that for $n=1$, $[g(\Delta t)]\in
\MM_{n_d,m_\optpp}$ goes to $[g_\DM]\in\MM_{n_d,m_\optpp}$ as $\Delta t \rightarrow 0$. For $n > 1$, matrix $[\tilde K(n\Delta t)]$ is \textit{a priori} symmetric, but matrix $[\curK(n\Delta t)]$ is not symmetric. For applying a similar approach to the one defined by Eq.~\eqref{eq7.9}, we then symmetrize the matrix $[\PP(\nDeltat)] = [B]^{-1/2} [\curK(n\Delta t)] \, [B]^{-1/2}$, introducing
$[\PP^S(\nDeltat)] = ([\PP(\nDeltat)] + [\PP(\nDeltat)]^T)/2$.
The transient basis $\{\bfg^1(\nDeltat), \ldots , \bfg^{n_d}(\nDeltat)\}$ associated with the eigenvalues
$\tilde b_1(\nDeltat) \geq \tilde b_2(\nDeltat) \geq \ldots \geq \tilde b_{n_d}(\nDeltat)$
are then computed by solving the symmetric eigenvalue problem
\begin{equation} \label{eq7.14}
[\PP^S(\nDeltat)]\, \bfvarphi^{\,\beta}(\nDeltat) =  \tilde b_\beta(\nDeltat) \, \bfvarphi^{\,\beta}(\nDeltat) \quad , \quad
\langle \bfvarphi^{\,\beta}(\nDeltat) \, , \bfvarphi^{\,\beta'}(\nDeltat)\rangle = \delta_{\beta\beta'} \, .
\end{equation}
\begin{equation} \label{eq7.15}
\bfg^{\,\beta}(\nDeltat) = [B]^{-1/2} \bfvarphi^{\,\beta}(\nDeltat) \, .
\end{equation}
The ROTB$(\nDeltat)$ is then defined by the matrix

\begin{equation} \label{eq7.16}
[g(\nDeltat)] = [\bfg^1(\nDeltat) \ldots \bfg^{m_\optpp}(\nDeltat)]\in\MM_{n_d,m_\optpp}\, ,
\end{equation}
in which $m_\optp$ is defined as explained in Section~\ref{Section7.2}-(i).

\subsection{Criteria for comparing the reduced-order transient basis with the reduced-order DMAPS basis}
\label{Section7.3}
The PLoM algorithm is based on the use of the RODB (see \ref{SectionA}). With such an RODB based on the isotropic kernel, PLoM has proven generally efficient, even in extremely difficult cases, as demonstrated in numerous publications since 2016. This efficacy will also be evident through the applications presented in Section~\ref{Section8}. However, for cases involving very heterogeneous data in the training dataset, the learned statistical dependence between the components of the random vector $\bfH$ (through the learned probability measure for $\bfH$) can \textit{a priori} be improved using ROTB$(\nDeltat)$ for given $n$. To assess a possible improvement with this reduced-order transient basis compared to the reduced-order DMAPS basis, quantitative criteria are necessary. In this section, we introduce criteria for comparing the two vector bases, and then, in Section~\ref{Section7.4}, we will present a methodology for identifying the instance, $\nDeltat$, that maximizes the selection criteria of the ROTB$(\nDeltat)$.

(i) The first criterion will be the angle between the two vector subspaces generated by the two vector bases, ROTB$(\nDeltat)$ and RODB. If this angle is close to zero then the two bases coincide. This must be the case when we choose  the instant $n=1$ for the ROTB$(\nDeltat)$  (See Proposition~\ref{proposition:7}).

(ii) The second criterion is the concentration of the learned probability measure in relation to the concentration of the probability measure of the training dataset. It has been proven that the PLoM algorithm, which uses the RODB, was designed to preserve the concentration of the learned probability measure from a small training dataset. For this, in \cite{Soize2020c,Soize2022a}, we introduced the indicator $d^2$, linked to the mean-square convergence, which we will recall and use. We also introduce a second concentration criterion, KL, based on the Kullback-Leibler divergence \cite{Bhattacharyya1943,Kullback1951} between the learned probability measure and the probability measure of the training dataset.

(iii) The third criterion is mutual information \cite{Kolmogorov1956,Cover2006}, which is defined as the relative entropy introduced by Kullback and Leibler \cite{Kullback1951}. This will be used to quantify the level of statistical dependencies among the components of the centered random vector $\bfH$, whose covariance matrix is the identity matrix. Such mutual information will be estimated for the learned probability measure generated by PLoM, comparing the RODB and the ROTB$(\nDeltat)$ as a function of $n$.

(iv) The last criterion is the Entropy from Information Theory \cite{Kapur1992,Cover2006,Gray2011}, introduced by Shannon \cite{Shannon1948}. The estimation of this entropy is a function of the number of realizations  used for estimating the probability measure. This property will be used in Section~\ref{Section7.4} to normalize the estimation of the mutual information.

\subsubsection{Angle between the subspaces spanned by ROTB$(\nDeltat)$ and RODB}
\label{Section7.3.1}
Let $n$ be fixed in $\curN$. It is assumed that $\text{rank}([g_\DM]) = \text{rank}([g(\nDeltat)])= m_\optp$ in which the two matrices are defined by Eqs.~\eqref{eq7.13} and \eqref{eq7.16}. In addition, it is assumed that the null space of the matrix $[g(\nDeltat)]^T\, [g_\DM]$ is $\{0\}$.
Let $V(\nDeltat) = \text{span}\{\bfg^1(\nDeltat), \ldots , \bfg^{m_\optpp}(\nDeltat)\}$ be the $m_\optp$-dimension subspace of $\RR^{n_d}$ (spanned by the ROTB$(\nDeltat)$).
Let $V_\DM = \text{span}\{\bfg^1_\DM, \ldots , \bfg_\DM^{m_\optpp}\}$ be the $m_\optp$-dimension subspace of $\RR^{n_d}$ (spanned by the RODB).
For $\beta\in\{1,\ldots m_\optp\}$, let $\hat g_\DM^{\,\beta} = g_\DM^{\,\beta} / \Vert g_\DM^{\,\beta}\Vert$ be the normalized vector $g_\DM^{\,\beta}$ and let $\hat g^{\,\beta}(\nDeltat) = g^{\,\beta}(\nDeltat) / \Vert g^{\,\beta}(\nDeltat)\Vert$ be the normalized vector $g^{\,\beta}(\nDeltat)$.
Let $[\hat g_\DM] = [\hat g^1_\DM\ldots \hat g_\DM^{m_\optpp} ]$ and
$[\hat g(\nDeltat)] = [\hat g^1(\nDeltat)\ldots \hat g(\nDeltat)^{m_\optpp} ]$ be matrices in $\MM_{n_d, m_\optpp}$.
The angle $\gamma (\nDeltat)$ between the subspaces $V(\nDeltat)$ and $V_\DM$ is defined by
\begin{equation}\label{eq7.17}
 \gamma (\nDeltat) = \arccos ( \sigma_{\rm{min}}([\hat g(\nDeltat)]^T\, [\hat g_\DM] ) ) \, ,
\end{equation}
in which $\sigma_{\rm{min}}$ denotes the smallest singular value. If the angle is close to $0$, the two subspaces are nearly linearly dependent.
\subsubsection{Indicators related to the learned probability measure}
\label{Section7.3.2}

In this section, we detail the indicators used for comparison: concentration of the learned probability measure, mutual information, and entropy. These will be employed to compare the probability measures associated with the training dataset and the learned dataset generated by the PLoM algorithm, as summarized in \ref{SectionA}. Specifically, we use the reduced-order DMAPS basis (RODB) and, alternatively, the reduced-order transient basis (ROTB$(\nDeltat)$). To facilitate the comparisons presented in the numerical illustrations, we introduce the necessary notations for clarifying the diverse quantities and their numerical calculations.

- \textit{Training dataset}. In Section~\ref{Section2}, the independent realizations of the $\RR^\nu$-valued random variable $\bfH$ are $\bfeta^1,\ldots ,\bfeta^{n_d}$, and the matrix $[\eta_d] = [\bfeta^1\ldots \bfeta^{n_d}]\in \MM_{\nu,n_d}$ represents one realization of the random matrix $[\bfH]$, used by PLoM, and defined in \ref{SectionA.1}. The probability density function of $\bfH$ is $p_\bfH$ defined by Eq.~\eqref{eq2.3}.

- \textit{Learned dataset}. In \ref{SectionA}, the learned realizations generated by PLoM are those of the $\MM_{\nu,n_d}$-valued random variable $[\bfH_\ar]$ and are written as ${[\eta_\ar^\ell], \ell=1,\ldots, n_\MCH}$ (see \ref{SectionA.3}). The corresponding realizations ${\bfeta_\ar^{\ell'}, \ell=1,\ldots, n_\ar}$, with $n_\ar = n_d \times n_\MCH$, of the $\RR^\nu$-valued random variable $\bfH_\ar$ are obtained by reshaping. The learned probability density function of $\bfH_\ar$ is denoted by $p_{\bfH_\arp}$. When PLoM is used with RODB, $[\bfH_\ar]$, $[\eta_\ar^\ell]$, $\bfeta_\ar^{\ell'}$, and $p_{\bfH_\arp}$ will be rewritten as $[\bfH_\DB]$, $[\eta_\DB^\ell]$, $\bfeta_\DB^{\ell'}$, and $p_\DB$, respectively. When PLoM is used with ROTB$(\nDeltat)$, $[\bfH_\ar]$, $[\eta_\ar^\ell]$, $\bfeta_\ar^{\ell'}$, and $p_{\bfH_\arp}$ will be rewritten as $[\bfH_\TB(\nDeltat)]$, $[\eta_\TB^\ell(\nDeltat)]$, $\bfeta_\TB^{\ell'}(\nDeltat)$, and
$p_\TB( . \, ; \nDeltat)$, respectively.\\

\noindent\textit{(i) Concentration of the learned probability measure for RODB and ROTB$(\nDeltat)$}.
Based on the mean-square norm, for PLoM with RODB and ROTB$(\nDeltat)$, the learned probability measure concentration is written (see Eq.~\eqref{eqA6}) as
\begin{equation} \label{eq7.18}
d_\DB^{\, 2}(m_\optp) = E\{\, \Vert \, [\bfH_\DB] - [\eta_d]\,  \Vert^2\, \}\,  /\, \Vert\,  [\eta_d]\, \Vert^2\, ,
\end{equation}
\begin{equation} \label{eq7.19}
d_\TB^{\, 2}(m_\optp;\nDeltat) = E\{\, \Vert \, [\bfH_\TB(\nDeltat)] - [\eta_d]\, \Vert^2\, \} \, / \, \Vert \, [\eta_d]\, \Vert^2\, .
\end{equation}
Using the realizations, these quantities are estimated by
\begin{equation} \label{eq7.20}
\hat d_\DB^{\, 2}(m_\optp) = (1/n_\MCH)\sum_{\ell=1}^{n_\ppMCH}\{\, \Vert\,  [\bfeta_\DB^\ell]  - [\eta_d]\, \Vert^2\, \}\,  /\, \Vert\,  [\eta_d]\, \Vert^2 \, .
\end{equation}
\begin{equation} \label{eq7.21}
\hat d_\DB^{\, 2}(m_\optp; \nDeltat) = (1/n_\MCH)\sum_{\ell=1}^{n_\ppMCH}\{\, \Vert\,  [\bfeta_\TB^\ell(\nDeltat)]  - [\eta_d]\, \Vert^2\, \} \, /\, \Vert \, [\eta_d]\, \Vert^2 \, .
\end{equation}
The concentration of the learned probability measure can also be estimated using the Kullback-Leiber divergence and its estimation from a set of realizations as presented in \ref{SectionB.1}. For the PLoM formulated with RODB and ROTB$(\nDeltat)$, we then have,

\begin{equation} \label{eq7.22}
D(p_\DB \, \Vert \, p_\bfH )  = \int_{\RR^\nu} p_\DB(\bfeta)\,\log \left ( \frac{p_\DB(\bfeta)}{p_\bfH(\bfeta)} \right) \, d\bfeta \, ,
\end{equation}
\begin{equation} \label{eq7.23}
D(p_\TB( .\, ;\nDeltat)\, \Vert \, p_\bfH )  = \int_{\RR^\nu} p_\TB(\bfeta \, ; \nDeltat)\,\log \left ( \frac{p_\TB(\bfeta ; \nDeltat)}{p_\bfH(\bfeta)} \right) \, d\bfeta \, .
\end{equation}
%

\noindent \textit{(ii) Mutual-information-based components statistical dependencies of the learned probability measure for  RODB and ROTB$(\nDeltat)$}.
The mutual information and its estimation from a set of realizations are presented in \ref{SectionB.2}. For the training dataset and for PLoM formulated with RODB and ROTB$(\nDeltat)$, we have , for $\RR^\nu$-valued random variables $\bfH$, $\bfH_\DB$, and $\bfH_\TB(\nDeltat)$,
\begin{equation} \label{eq7.24}
I(\bfH) = D(p_\bfH \, \Vert \, \otimes_{k=1}^\nu p_{H_k}) =
              \int_{\RR^\nu} p_\bfH(\bfeta)\,\log \left ( \frac{p_\bfH(\bfeta)}{\Pi_{k=1}^\nu p_{H_k}(\eta_k)} \right) \, d\bfeta  \, ,
\end{equation}
\begin{equation} \label{eq7.25}
I(\bfH_\DB) = D(p_\DB \, \Vert \, \otimes_{k=1}^\nu p_{\DB,k}) =
             \int_{\RR^\nu} p_\DB(\bfeta)\,\log \left ( \frac{p_\DB(\bfeta)}{\Pi_{k=1}^\nu p_{\DB,k}(\eta_k)} \right) \, d\bfeta  \, ,
\end{equation}
\begin{equation} \label{eq7.26}
I(\bfH_\TB;\nDeltat) = D(p_\TB(\cdot \, ; \nDeltat) \, \Vert \, \otimes_{k=1}^\nu p_{\TB,k}(\cdot \,  ; \nDeltat)) =
   \int_{\RR^\nu} p_\TB(\bfeta\, ;\nDeltat)\,\log \left ( \frac{p_\TB(\bfeta\, ;\nDeltat)}{\Pi_{k=1}^\nu p_{\TB,k}(\eta_k\, ;\nDeltat)} \right) \, d\bfeta  \, ,
\end{equation}
in which $p_{\DB,k}$ is the pdf of component $H_{\DB,k}$ of $\bfH_\DB$ and where $p_{\TB,k}(\cdot \,  ; \nDeltat)$ is the pdf of component
$H_{\TB,k}(\nDeltat)$ of $\bfH_{\TB}(\nDeltat)$.\\

\noindent\textit{(iii) Entropy for RODB and ROTB$(\nDeltat)$}.
The entropy and its estimation from a set of realizations are presented in \ref{SectionB.3}. For the training dataset and for PLoM formulated with RODB and ROTB$(\nDeltat)$, we have for $\RR^\nu$-valued random variables $\bfH$, $\bfH_\DB$, and $\bfH_\TB(\nDeltat)$,
\begin{equation}\label{eq7.27}
S_\bfH = -\int_{\RR^\nu} p_\bfH(\bfeta)\,\log p_\bfH(\bfeta) \, d\bfeta  \quad , \quad
S_\DB = -\int_{\RR^\nu}  p_\DB(\bfeta)\,\log p_\DB(\bfeta) \, d\bfeta     \, ,
\end{equation}
\begin{equation}\label{eq7.28}
S_\TB(\nDeltat) = -\int_{\RR^\nu}  p_\TB(\bfeta\, ; \nDeltat)\,\log p_\TB(\bfeta\, ; \nDeltat) \, d\bfeta \, .
\end{equation}
%
\subsection{Identification methodology for the instant maximizing the selection criterion of the ROTB$(\nDeltat)$}
\label{Section7.4}
In this section, all quantities denoted with a hat (such as $\hat I(\bfH)$) represent the estimated values of the corresponding quantities without a hat (such as $I(\bfH)$), using the realizations as explained in \ref{SectionB}.
Let $n_d$ be the number of points in the training dataset and $n_\ar = n_d\times n_\MCH \gg n_d$ be the number of learned realizations with PLoM (see \ref{SectionA.3}), either with the RODB or with the ROTB$(\nDeltat)$.\\

\noindent \textit{(i) Defining the subset $\curC_N$ of $\curN$ containing the admissible values of $n$}.
Let $N$ be the largest value of $n$ for which the ROTB$(\nDeltat)$, represented by matrix $[g(\nDeltat)]$, is computed. This value $N$ being fixed, the set $\curN = \{1,\ldots, N\}$ defined by Eq.~\eqref{eq5.23} is fixed.
The subset $\curC_N$ of the admissible values of $n$ is then defined by
\begin{equation}\label{eq7.29}
\curC_N = \{ n\in \curN \, , \, \frac{1}{\nu} \, \hat d_\TB^{\,2}(m_\optp;\nDeltat) \leq \tau_c \ll 1\} \, ,
\end{equation}
where $\hat d_\TB^{\, 2}(m_\optp;\nDeltat)$ is defined by Eq.~\eqref{eq7.21} and where $0 < \tau_c \ll 1$ is fixed sufficiently small with respect to $1$ in order to preserve the concentration of the learned probability measure (see Section~\ref{Section7.3}-(i) and \ref{SectionA.4}). It should be noted that the maximum value of $d_\DB^{\, 2}(n_d)$ is $2$ (corresponding to a measure concentration completely lost, which can be obtained  with classical MCMC algorithms). The PLoM algorithm allows for preserving the concentration of the learned probability measure \cite{Soize2020c} yielding values of the order $0.01$ to $0.1$ through all the performed applications. For defining the concentration criterion in Eq.~\eqref{eq7.29}, we have normalized with respect to dimension $\nu$.\\

\noindent \textit{(ii) Characterizing a better ROTB$(\nDeltat)$  compared to RODB}.
From Proposition~\ref{proposition:7}, in particular from Eq.~\eqref{eq7.7}, it can be deduced that, for $n=1$, we have $\hat I(\bfH_\TB\, ; 1 \times \Delta t) \simeq \hat I(\bfH_\DB)$ for $\Delta t$ sufficiently small (which is the considered case).
Let us assume that $\hat I(\bfH) < \hat I(\bfH_\DM)$. For $n$ fixed in $\curC_N$, we will say that the ROTB$(\nDeltat)$, represented by matrix $[g(\nDeltat)]$, is better than the RODB, represented by matrix $[g_\DM]$, if
\begin{equation}\label{eq7.30}
\hat I(\bfH) \leq \hat I(\bfH_\TB\, ; \nDeltat) <  \hat I(\bfH_\DM) \, .
\end{equation}

\noindent \textit{(iii) Determining the optimal value $n_\optp$ of $n$ in $\curC_N$}.
The underlying idea is to select a reduced-order transient basis that gives a learned probability measure whose mutual information is as close as possible to the mutual information of the probability measure $p_\bfH(\bfeta), d\bfeta$.
Based on paragraph (ii) above, the optimal transient basis $[g(n_\optp)]$ among the set of possible
transient bases $\{ [g(\nDeltat)] , n\in\curN \}$ is obtained for $n=n_\optp$ such that
\begin{equation}\label{eq7.31}
n_\optp = \arg \min_{n\,\in\curC_N}  \hat I(\bfH_\TB\, ; \nDeltat) \, .
\end{equation}

\noindent \textit{(iv) Defining the normalized estimate of the mutual information}.
As explained in \ref{SectionB.3}, if $n_{\rm{samp}}$ is the number of realizations (either $n_d$ or $n_\ar$), the entropy estimation asymptotically decreases as $-log(n_{\rm{samp}})$ when $n_{\rm{samp}}$ increases.
For comparing the estimate of the mutual information for $\bfH$, which uses $n_d$ realizations, with the one of $\bfH_\DB$ or $\bfH_\TB(\nDeltat)$, which uses $n_\ar \gg n_d$ learned realizations, we must normalize the estimated mutual information with respect to the number of realizations. Taking into account the relationship between the mutual information and the entropy estimates (see \ref{SectionB.3}), we chose to normalize the estimated mutual information by dividing by the function $\chi + \log(n_{\rm{samp}})$ (with $n_{\rm{samp}}=n_d$ or $n_{\rm{samp}}=n_\ar$), where $\chi$ is a real number that must be identified and which must be such that $\chi + log(n_d) > 0$. Since $n_\ar > n_d$ this condition implies that  $\chi + log(n_\ar) > 0$ is automatically verified.
For the $\RR^\nu$-valued random variables $\bfH$, $\bfH_\DB$, and $\bfH_\TB(\nDeltat)$, the normalized estimated mutual information are then defined by
\begin{equation} \label{eq7.32}
\hat I_\normp(\bfH) = \frac{\hat I(\bfH)}{\chi + log(n_d)} \quad , \quad
\hat I_\normp(\bfH_\DB) = \frac{\hat I(\bfH_\DB)}{\chi + log(n_\ar)} \quad , \quad
\hat I_\normp(\bfH_\TB\, ; \nDeltat) = \frac{\hat I(\bfH_\TB\, ; \nDeltat)}{\chi + log(n_\ar)} \, .
\end{equation}
in which  $\hat I(\bfH)$, $\hat I(\bfH_\DB)$, and $\hat I(\bfH_\TB\, ; \nDeltat)$ are the estimates (computed with Eq.~\eqref{eqB6}) of the mutual information $I(\bfH)$, $I(\bfH_\DB)$, and $I(\bfH_\TB\, ; \nDeltat)$ defined by Eqs.~\eqref{eq7.24} to \eqref{eq7.26}.
Based on paragraph (iii) above, constant $\chi$ is defined as follows. Let $\chi_\optp$ be the solution  in $\chi$ of the equation,
\begin{equation} \label{eq7.33}
\frac{\hat I(\bfH)}{\chi + log(n_d)} = \frac{\hat I(\bfH_\TB\, ; n_\optp \,\Delta t)}{\chi + log(n_\ar)}  \, .
\end{equation}
If $\chi_\optp + \log(n_d) > 0$, then $\chi_\optp$ is the desired value of $\chi$.
It should be noted that the optimal value $n_\optp$ is calculated with the estimation of the non-normalized mutual information (see Eq.~\eqref{eq7.31}. Normalization is only introduced for the purpose of comparing the value of this criterion for $\bfH$ and $\bfH_\TB$.\\

\noindent \textit{(v) Identification of the instant maximizing the selection criterion of the ROTB$(\nDeltat)$}.
Let $n_\optp$ be defined by Eq.~\eqref{eq7.31}, and let $\chi_\optp$ be identified by solving Eq.~\eqref{eq7.33}. Then, we have
\begin{equation} \label{eq7.34}
\hat I_\normp(\bfH) = \hat I_\normp(\bfH_\TB\, ; n_\optp \Delta t) \leq \hat I_\normp(\bfH_\TB\, ; \nDeltat)
\quad , \quad \forall n\in\curC_N\, .
\end{equation}
We can then conclude that the ROTB$(n_\optp \Delta t)$, represented by matrix $[g(n_\optp \Delta t)]$, is better than the RODB, represented by matrix $[g_\DM]$. The proof is straightforward.

\section{Numerical applications}
\label{Section8}
%
\subsection{Preamble}
\label{Section8.1}

We will present three applications, each with its own specificities. However, as the generation of the training dataset related to the random vector $\bfX$ is relatively complex to describe, we will start with the training dataset related to the normalized random vector $\bfH$ with values in $\RR^\nu$. The number of realizations in the training dataset is $n_d$, and the training dataset is represented by the matrix $[\eta_d] \in \MM_{\nu,n_d}$. This means that the PCA step of PLoM, which transforms $\bfX$ into $\bfH$ (see \ref{SectionA.1}), is not detailed here. Readers interested in the training dataset for $\bfH$ can request, from the "corresponding author" of the article, the transfer of the matrix $[\eta_d]$ for each application presented.
We will still briefly give the main specificities of these 3 applications.

\textit{Application 1}. This application (\textit{Appli 1}) was created so that the probability measure of $\bfH$ in $\RR^9$ ($\nu=9$), which is defined by the points of the training dataset, is concentrated in a multiconnected domain of $\RR^9$, with the constituent connected parts being manifolds of dimensions much lower than 9, each having different dimensions. These parts may or may not be connected to each other.

\textit{Application 2}. For the second application (\textit{Appli 2}), the $n_d$ realizations of the random vector $\bfH$ with values in $\RR^8$ are generated using a polynomial chaos expansion of degree $6$ of a real random variable, whose random germ is of dimension $2$, with each of the two random germs being a uniform random variable of different support. There are therefore $28$ terms in this expansion, and the $8$ components of $\bfH$ are defined as the random terms of rank $2$, $3$, $6$, $8$, $12$, $13$, $17$, and $19$. We thus define a relatively complex random manifold in $\RR^8$.

\textit{Application 3}. The third application (\textit{Appli 3}) results from a statistical treatment of an experimental database containing photon measurements in the ATLAS detector at CERN. This dataset was obtained by loading the file 'pid22_E262144_eta_20_25_voxalisation.csv' from the free access CERN Open Data Portal. The PCA step of PLoM has been performed, and an extraction of $45$ components has been done to obtain the training dataset for the $\RR^{45}$-valued random variable $\bfH$. This application is in higher dimension than the first two, but the statistical complexity is less.

\subsection{Additional convergence analysis conducted for the three applications}
\label{Section8.2}
For each of the three applications, we will provide detailed analyses of the calculation of the optimal value of the instant $n_\optp \Delta t$, which allows for the selection of the best reduced-order transient basis (ROTB). We will also present the convergence results of the PLoM algorithms under the normalization constraints (see \ref{SectionA.6}). However, we cannot show all the convergence results for the other parameters. Below are the different analyses that were carried out with respect to the parameters that control the construction of the reduced-order transient basis:\\

\noindent (i) Convergence with respect to the value of $\kappa$. This point is important because $\Delta t$ must be sufficiently small to apply Proposition~\ref{proposition:7}. For the three applications, we found that $\kappa = 30$ is an appropriate value. \\

\noindent (ii) Convergence with respect to the size of $\delta t$, that is to say, the value of $n_s$. For the three applications, we found that once $\Delta t$ is fixed by the value of $\kappa$, the criterion associated with Eq.~\eqref{eq5.30} was satisfied for $n_s=1$, as well as the evolution of the angle $\gamma(n_\opt \Delta t)$ at the optimal instant was very insensitive to the values of $n_s$ greater than
$1$. This is due to the fact that $\Delta t$ is already small enough to achieve good accuracy with the Euler scheme used to integrate the ISDE. We found that $n_s=1$, thus $\delta t = \Delta t$ is a good value.\\

\noindent (iii) Convergence with respect to $n_\MC$. This analysis was performed by examining the convergence of the angle $\gamma(\nDeltat)$ for $n \in \curN$. We observed that a large value of $n_\MC$ was necessary to achieve convergence. This point is particularly important and must be carefully checked for the applications.

\subsection{Parameters defining the training dataset and controlling  the construction of the RODB}
\label{Section8.3}
For each application, Table~\ref{table:table1} provides the values of the parameters that define the training dataset and the probability measure $P_\bfH(d\bfeta) = p_\bfH(\bfeta)\, d\bfeta$ on $\RR^\nu$, as well as the parameters defined in Section~\ref{Section7.2}-(i), which control the construction of the reduced-order DMAPS basis (RODB).

%
\begin{table}[h]
  \caption{For each application, values of the parameters defining the training dataset and parameters controlling  the construction of the RODB}\label{table:table1}
\begin{center}
 \begin{tabular}{|c||c|c|c|c|c||c|c|c|} \hline
                         & \multicolumn{5}{c||}{Training}   & \multicolumn{3}{c|}{RODB}  \\
    \hline
                         &  $\nu$   &  $n_d$   &  $s$    &   $\hat s$   &  $\hat s / s$  & $J_\DM$   &  $m_\optp$   &  $\varepsilon_\DM$ \\
    \hline
    Appli $1$            &    9      &   400   & 0.5835  &   0.5044     &   0.8645       &  0.2      &   10         &    56               \\
    Appli $2$            &    8      &   400   & 0.5623  &   0.4946     &   0.8725       &  0.1      &   10         &    53               \\
    Appli $3$            &   46      &   560   & 0.8357  &   0.6416     &   0.7677       &  0.5      &   46         &    74               \\
    \hline
  \end{tabular}
\end{center}
\end{table}
\subsection{Parameters controlling the construction of the ROTB$(\nDeltat)$}
\label{Section8.4}
For each application, Table~\ref{table:table2} provides the values of the parameters that control the construction of the reduced-order transient basis ROTB$(\nDeltat)$. In particular, the optimal time $n_\optp \Delta t$ is estimated using the indicators defined in Section~\ref{Section7.3.2}. The evolution of these indicators as a function of $n$ will be presented in detail in Sections~\ref{Section8.6} to \ref{Section8.8} for the three applications.
\begin{table}[h]
  \caption{For each application, values of the parameters controlling the construction of the ROTB$(\nDeltat)$}\label{table:table2}
\begin{center}
 \begin{tabular}{|c||c|c|c|c|c|c|c|c|} \hline
                         & \multicolumn{8}{c|}{ROTB$(\nDeltat)$}     \\
    \hline
                         &  $N$      &  $\tau_c$ & $\kappa$ & $\Delta t$ & $n_s$ &  $\delta t$  & $n_\MC$   &  $n_\optp$  \\
    \hline
    Appli $1$            &    9      &   0.002   & 30       & 0.00848    &   1   &   0.00848    & 400000    &   9         \\
    Appli $2$            &   10      &   0.002   & 30       & 0.00802    &   1   &   0.00802    & 400000    &   5         \\
    Appli $3$            &    9      &   0.002   & 30       & 0.01370    &   1   &   0.01370    & 448000    &   9         \\
    \hline
  \end{tabular}
\end{center}
\end{table}
\subsection{Parameters of PLoM with RODB and ROTB}
\label{Section8.5}
For each of the three applications, Table~\ref{table:table3} provides the values of the parameters defined in \ref{SectionA.5} and \ref{SectionA.6} used by the PLoM algorithm with the reduced-order DMAPS basis (RODM) or the reduced-order transient basis ROTB$(\nDeltat)$. Taking into account the convergence analysis carried out on $n_\MC$ and the expertise on the statistical convergence of the quantities considered, we choose $n_\MCH = n_\MC$, a sufficiently large value (see \ref{SectionA.5}).
For the PLoM algorithm, the constraints related to normalization, defined by Eq.~\eqref{eqA12}, are always applied in the computation (see \ref{SectionA.6})
\begin{table}[h]
  \caption{Values of the parameters of PLoM with RODB and ROTB$(\nDeltat)$ for each application}\label{table:table3}
\begin{center}
 \begin{tabular}{|c||c|c|c|c|c||c|c|c|c|c|} \hline
                         & \multicolumn{10}{c|}{PLoM with RODB and with ROTB$(n_\optp \Delta t)$}     \\
    \hline
              & $f_0$ & $\Delta t_\SV$ & $M_0$ & $n_\MCH$ & $n_\ar$ & $\beta_1$ & $\beta_2$ & $i_2$ & $i_{\rm{last}}$ & $\err(i_{\rm{last}})$\\
    \hline
    \hline
    Appli $1$ with RODB & 4 & 0.1585 & 30 & 1000  & 400000 & 0.001 & 0.05 & 20 & 2563  & 0.000998 \\
    Appli $1$ with ROTB & 4 & 0.1585 & 30 & 1000  & 400000 & 0.001 & 0.05 & 20 & 3438  & 0.000994 \\
    \hline
    \hline
    Appli $2$ with RODB & 4 & 0.1541 & 30 & 1000  & 400000 & 0.001 & 0.05 & 20 & 1882  & 0.000999 \\
    Appli $2$ with ROTB & 4 & 0.1541 & 30 & 1000  & 400000 & 0.001 & 0.05 & 20 & 2589  & 0.000998 \\
    \hline
    \hline
    Appli $3$ with RODB & 4 & 0.2016 & 30 &  800  & 448000 & 0.001 & 0.05 & 20 & 6000  & 0.00261 \\
    Appli $3$ with ROTB & 4 & 0.2016 & 30 &  800  & 448000 & 0.001 & 0.05 & 20 & 6000  & 0.00383 \\
    \hline
  \end{tabular}
\end{center}
\end{table}
%

%
%
\subsection{Results for Application 1}
\label{Section8.6}

\noindent (i) Figure~\ref{fig:figure3} displays the graphs of the probability density function (pdf)  of components $1$, $2$, $4$, and $5$ for $\bfH$ estimated with the $n_d$ realizations of the training dataset, and the pdf estimated with $n_\ar$ learned realizations, for $\bfH_\ar$ using MCMC without PLoM (a,d,g,j), for $\bfH_\DB$ using  PLoM with RODB (b, e, h, k), and for $\bfH_\TB$ using PLoM with ROTB$(n_\optp \Delta t)$ (c, f, i, l).}

\noindent (ii)  Figure~\ref{fig:figure4} shows the joint probability density function of components $4$ and $5$ of $\bfH$ estimated with the $n_d$ realizations of the training dataset (a) and estimated with $n_\ar$ learned realizations for $\bfH_\ar$ using MCMC without PLoM (b), for $\bfH_\DB$ using PLoM with RODB (c), and for $\bfH_\TB$ using PLoM with ROTB$(n_\optp \Delta t)$ (d).

\noindent (iii)  In Fig.~\ref{fig:figure5}, the clouds of $n_\ar$ points corresponding to $n_\ar$ learned realizations can be seen for components $1$, $2$, $3$ (a,b,c) and components $3$, $4$, $5$ (d,e,f). These are shown for $\bfH_\ar$ using MCMC without PLoM (a,d), for $\bfH_\DB$ using PLoM with RODB (b,e), and for $\bfH_\TB$ using PLoM with ROTB$(n_\optp \Delta t)$ (c,f).

\noindent (iv) Figure~\ref{fig:figure6} plots the functions that characterize the reduced-order transient basis ROTB$(\nDeltat)$ as a function of time $\nDeltat$:
\begin{itemize}
\item The eigenvalues of matrix $[K_\DM]$ and those of the of symmetrized matrix $[\tilde K(\nDeltat)]$ are shown in
Fig.~\ref{fig:figure6a}.
\item The probability-measure concentration using the $d^{\,2}/\nu$-criterion is shown in Fig.~\ref{fig:figure6b}.
For the learning without PLoM, the $d^{\,2}$-concentration is $0.6465$, which shows that the concentration is lost, and for the PLoM with the RODM, the concentration is $0.0116$, which shows that the concentration is preserved.
\item The other criterion of the probability-measure concentration is given by Kullback measure, shown in Fig.~\ref{fig:figure6c}.
For the learning without PLoM, Kullbach is $0.3486$, and for the PLoM with the RODM, Kullback is $3.9435$. Comparing Figs.~\ref{fig:figure6b} and \ref{fig:figure6c} shows that the two criteria are consistent and give the same analysis of the concentration.
\item The angle between the subspaces spanned by RODB and ROTB$(\nDeltat)$ is displayed in Fig.~\ref{fig:figure6d}. It can be seen that, for the optimal time $9\, \Delta t$, the angle is $53.03^\circ$, which is a significant angle showing that the two bases are different while the $d^{\,2}$-concentration remains small at $0.0174$.
\item The entropy estimation of pdf $p_\TB(\cdot\, ; \nDeltat)$ is given in Fig.~\ref{fig:figure6e}.
\item The normalized mutual information (MI) of the pdfs $p_\bfH$ and $p_\TB(\cdot\, ; \nDeltat)$ is shown in Fig.~\ref{fig:figure6f}.
This figure shows that the optimal value of $n$ is $n_\optp = 9$.
For the non-normalized estimation of the mutual information, we have $\hat I(\bfH) = 4.4668$, $\hat I(\bfH_\TB\, ; n_\optp \Delta t) = 6.9996$, and $\hat I(\bfH_\DM) = 7.0945$. For the normalized one, we have
$\hat I_\normp(\bfH) = \hat I_\normp(\bfH_\TB\, ; n_\optp \Delta t) = 0.3666$ and $\hat I_\normp(\bfH_\DM) = 0.3716$.
\end{itemize}

\noindent (v) Finally, examination of these figures shows that traditional learning without PLoM gives poor results compared to PLoM, which allows the concentration to be preserved and properly learns the geometry of the probability measure support. We also see that PLoM with the optimal ROTB provides an improvement in learning compared to PLoM with RODM and, therefore, should improve the estimates of conditional statistics thanks to better learning of the joint probability measure.

\begin{figure}[!t]
    \centering
    \begin{subfigure}[b]{0.25\textwidth}
    \centering
        \includegraphics[width=\textwidth]{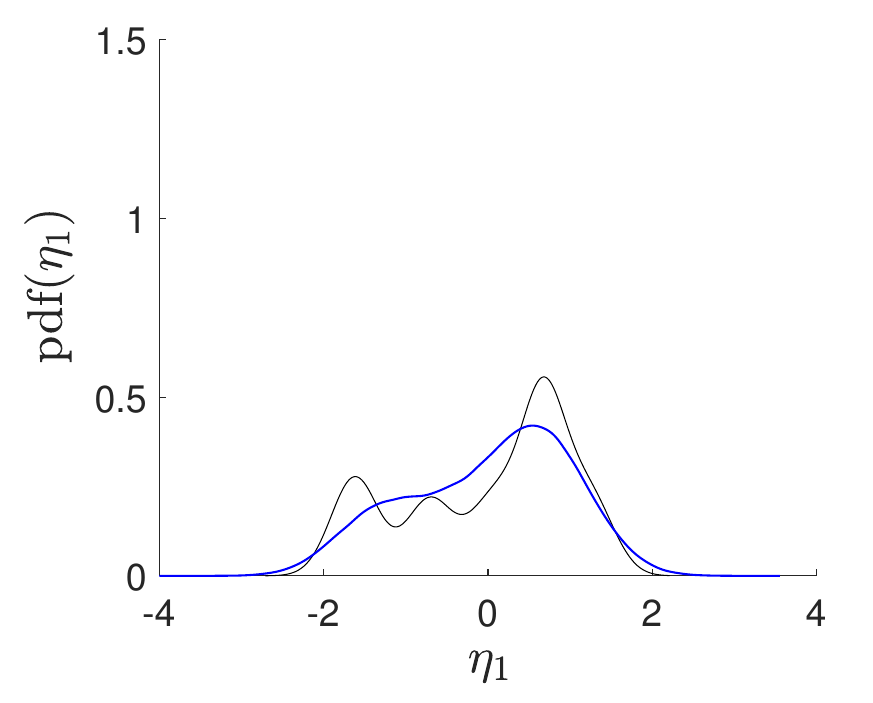}
        \caption{pdf of $H_1$ and $H_{\ar,1}$.}
        \label{fig:figure3a}
    \end{subfigure}
    \hfil
    \begin{subfigure}[b]{0.25\textwidth}
        \centering
        \includegraphics[width=\textwidth]{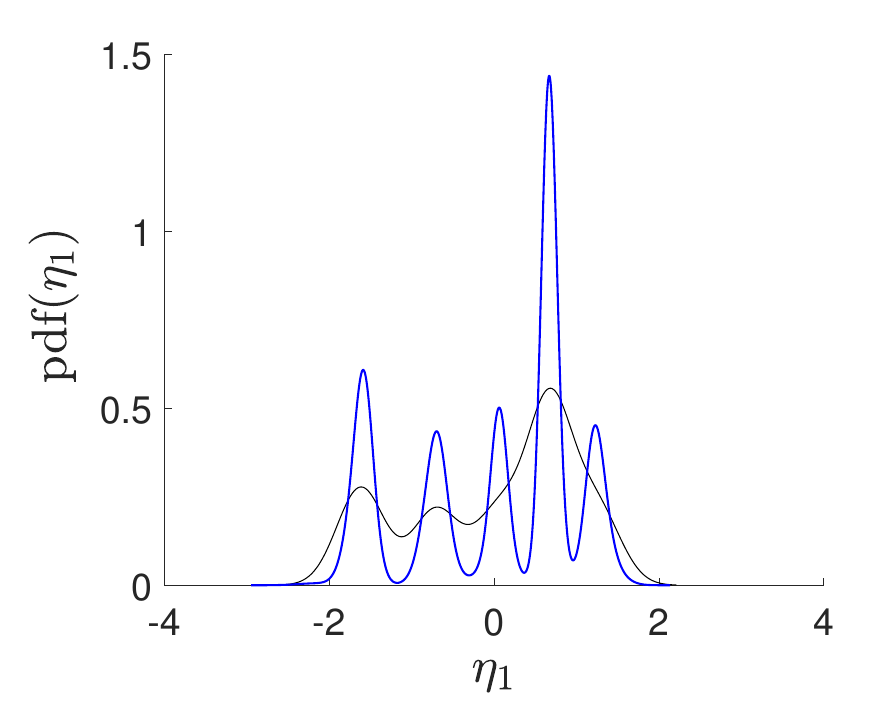}
        \caption{pdf of $H_1$ and $H_{\DB,1}$.}
        \label{fig:figure3b}
    \end{subfigure}
    \hfil
    \begin{subfigure}[b]{0.25\textwidth}
        \centering
        \includegraphics[width=\textwidth]{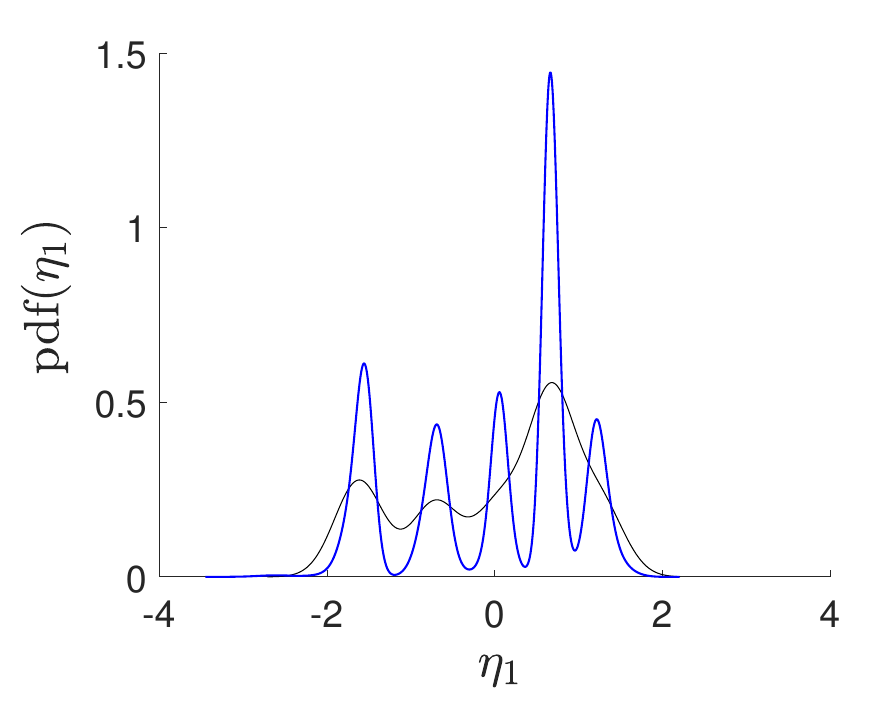}
        \caption{pdf of $H_1$ and $H_{\TB,1}$ at time $n_\optp\,\Delta t$.}
        \label{fig:figure3c}
    \end{subfigure}
    %
    \centering
    \begin{subfigure}[b]{0.25\textwidth}
    \centering
        \includegraphics[width=\textwidth]{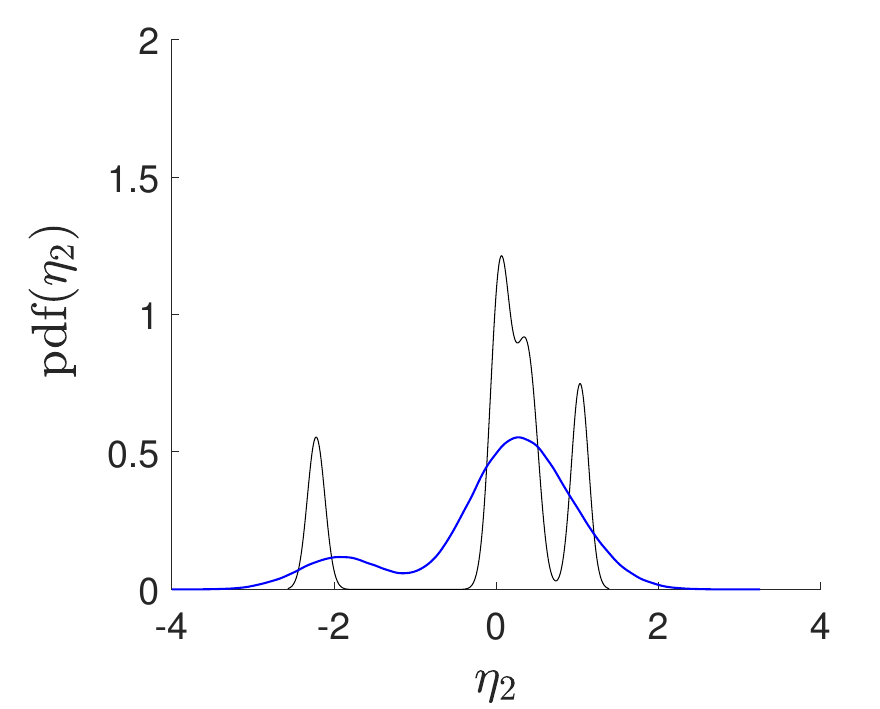}
         \caption{pdf of $H_2$ and $H_{\ar,2}$.}
        \label{fig:figure3d}
    \end{subfigure}
    \hfil
    \begin{subfigure}[b]{0.25\textwidth}
        \centering
        \includegraphics[width=\textwidth]{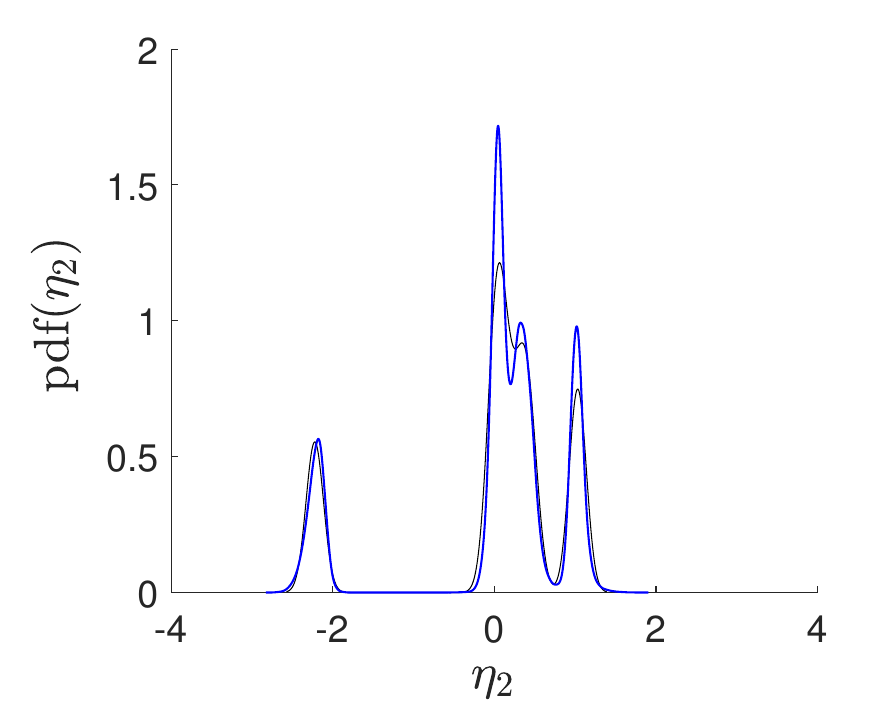}
        \caption{pdf of $H_2$ and $H_{\DB,2}$.}
        \label{fig:figure3e}
    \end{subfigure}
    \hfil
    \begin{subfigure}[b]{0.25\textwidth}
        \centering
        \includegraphics[width=\textwidth]{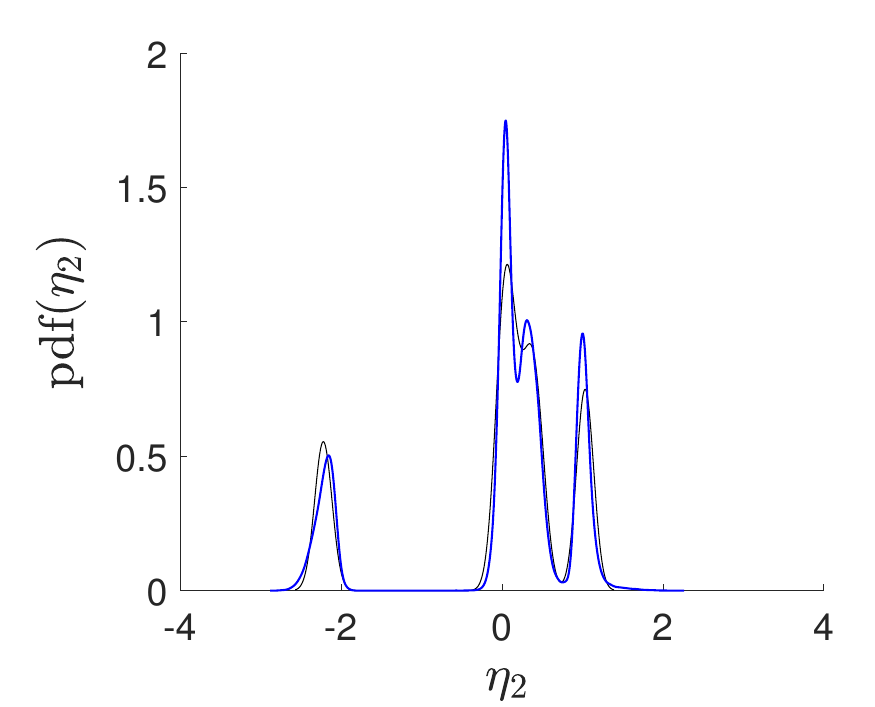}
        \caption{pdf of $H_2$ and $H_{\TB,2}$ at time $n_\optp\,\Delta t$.}
        \label{fig:figure3f}
    \end{subfigure}
    %
    \centering
    \begin{subfigure}[b]{0.25\textwidth}
    \centering
        \includegraphics[width=\textwidth]{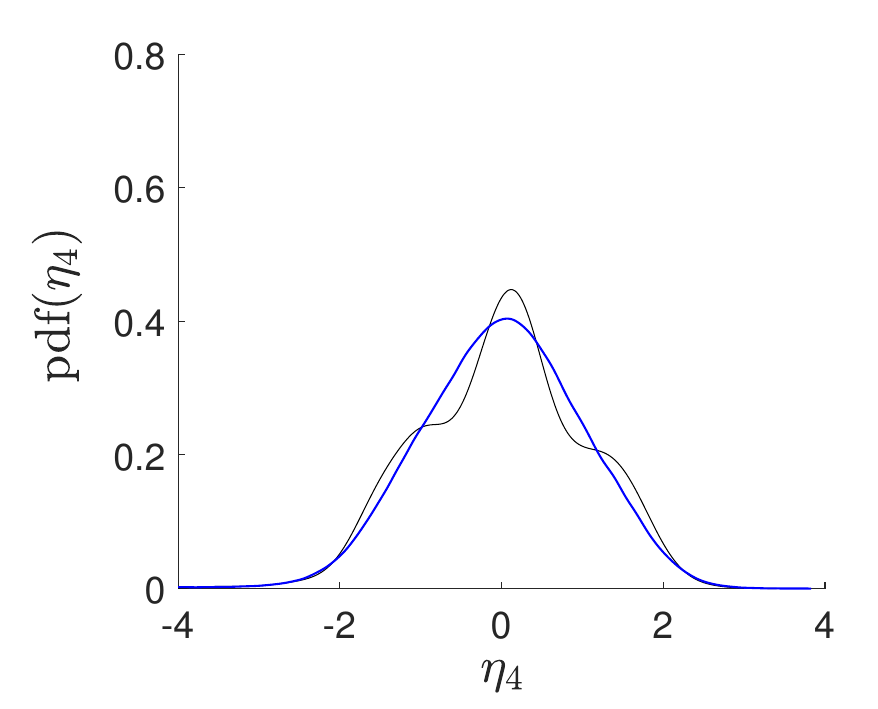}
         \caption{pdf of $H_4$ and $H_{\ar,4}$.}
        \label{fig:figure3g}
    \end{subfigure}
    \hfil
    \begin{subfigure}[b]{0.25\textwidth}
        \centering
        \includegraphics[width=\textwidth]{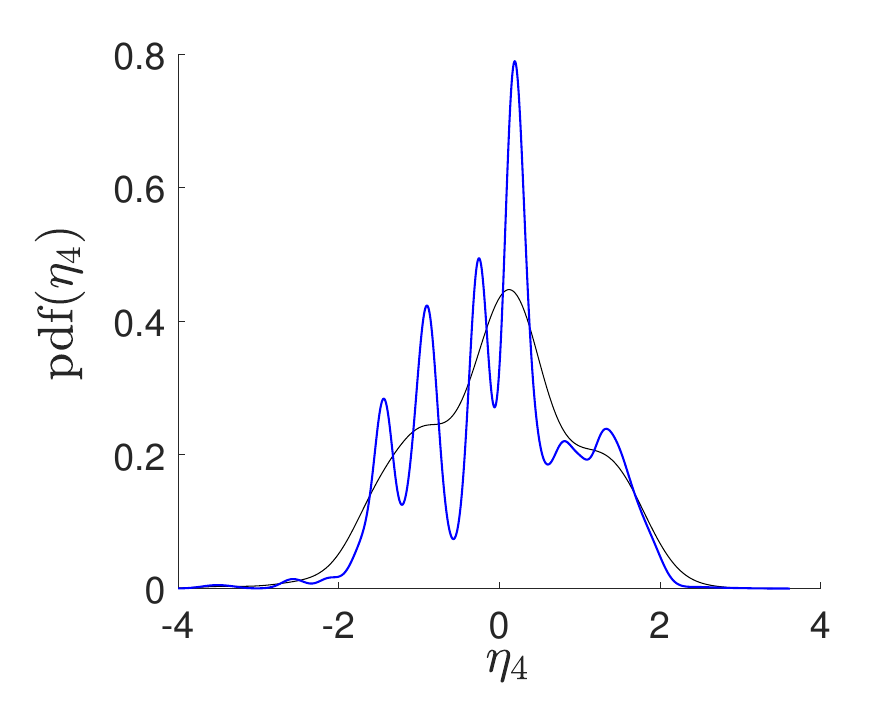}
        \caption{pdf of $H_4$ and $H_{\DB,4}$.}
        \label{fig:figure3h}
    \end{subfigure}
    \hfil
    \begin{subfigure}[b]{0.25\textwidth}
        \centering
        \includegraphics[width=\textwidth]{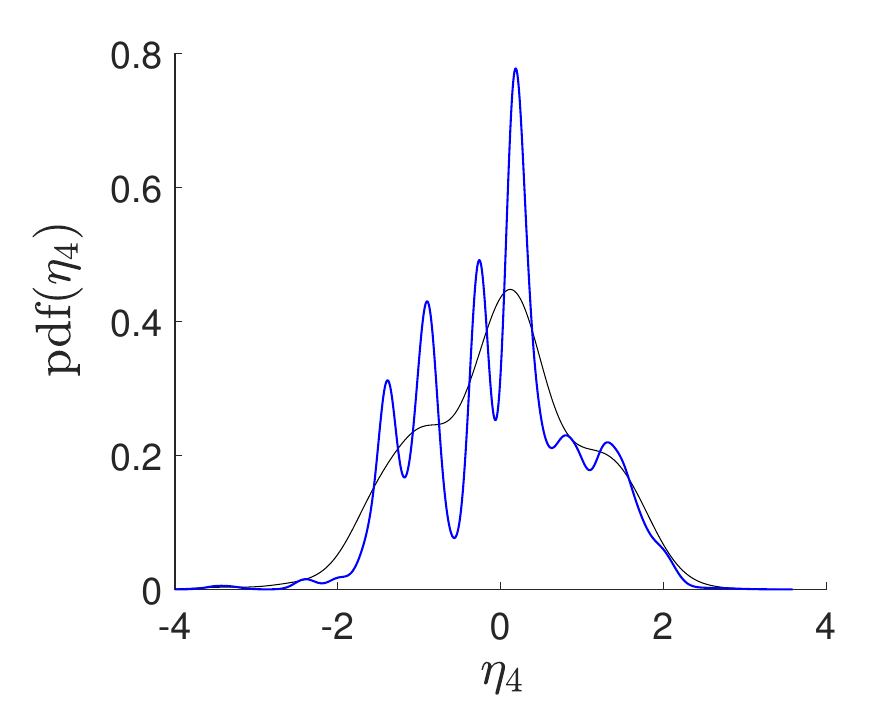}
        \caption{pdf of $H_4$ and $H_{\TB,4}$ at time $n_\optp\,\Delta t$.}
        \label{fig:figure3i}
    \end{subfigure}
    %
    \centering
    \begin{subfigure}[b]{0.25\textwidth}
    \centering
        \includegraphics[width=\textwidth]{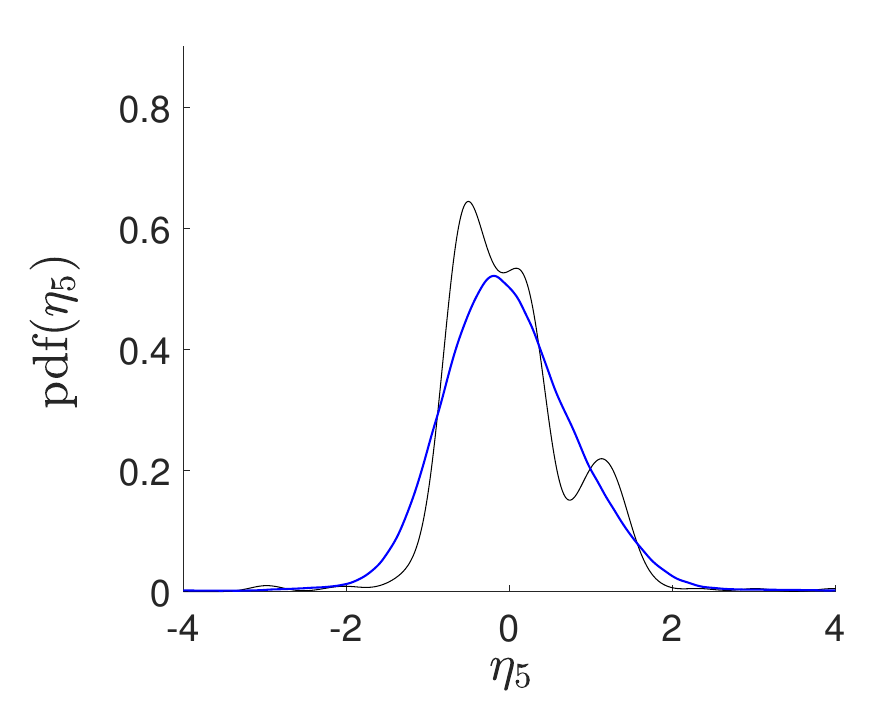}
        \caption{pdf of $H_5$ and $H_{\ar,5}$.}
        \label{fig:figure3j}
    \end{subfigure}
    \hfil
    \begin{subfigure}[b]{0.25\textwidth}
        \centering
        \includegraphics[width=\textwidth]{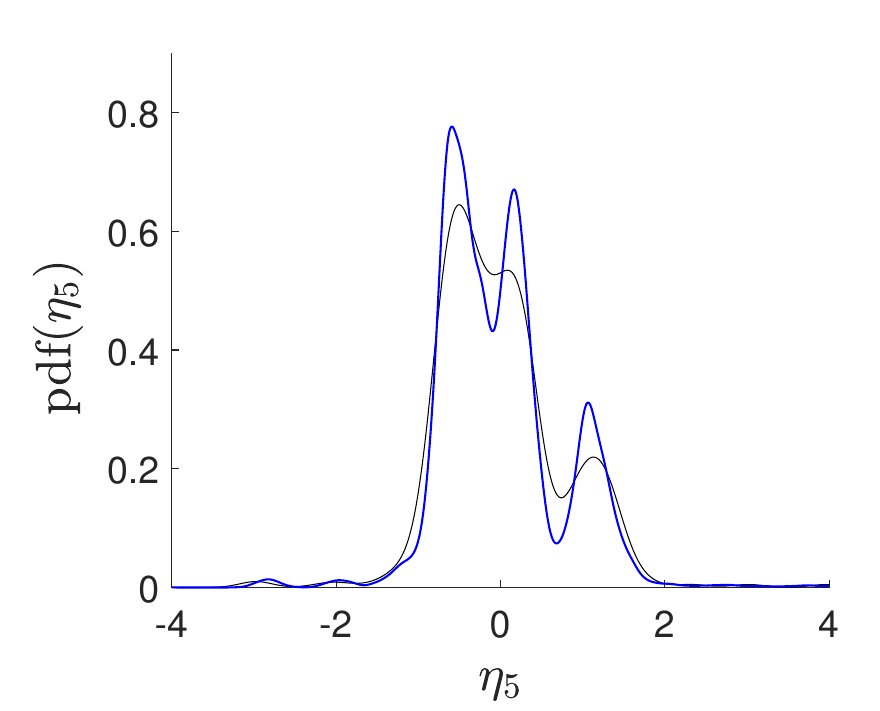}
        \caption{pdf of $H_5$ and $H_{\DB,5}$.}
        \label{fig:figure3k}
    \end{subfigure}
    \hfil
    \begin{subfigure}[b]{0.25\textwidth}
        \centering
        \includegraphics[width=\textwidth]{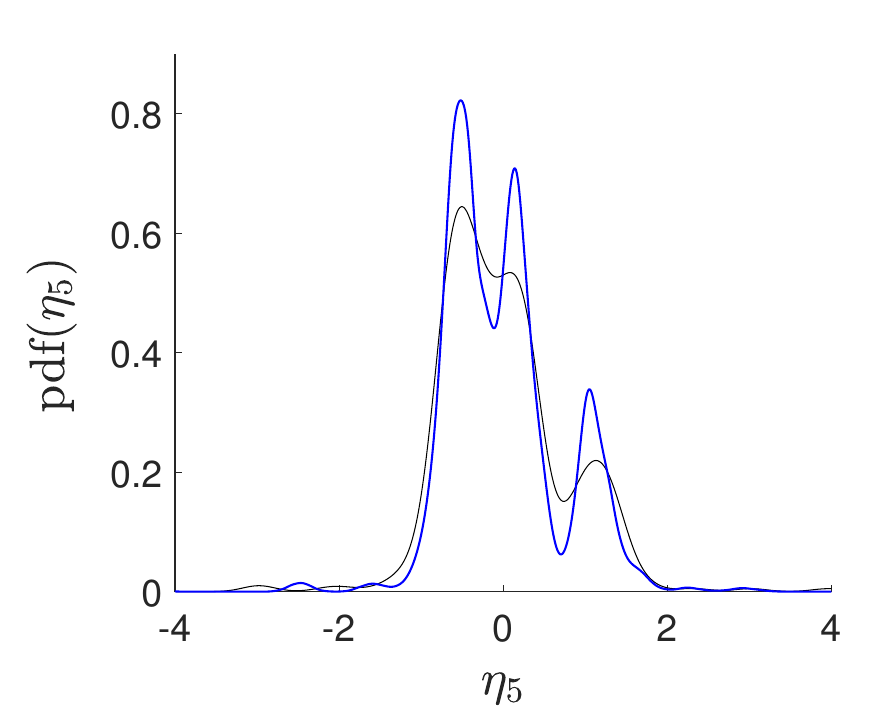}
        \caption{pdf of $H_5$ and $H_{\TB,5}$ at time $n_\optp\,\Delta t$.}
        \label{fig:figure3l}
    \end{subfigure}
    \caption{Application 1. Probability density function (pdf)  of components $1$, $2$, $4$, and $5$ for $\bfH$ estimated with the $n_d$ realizations of the training dataset (thin black line) and pdf estimated with $n_\ar$ learned realizations (thick blue line), for $\bfH_\ar$ using MCMC without PLoM (a,d,g,j), for $\bfH_\DB$ using  PLoM with RODB (b,e,h,k), and for $\bfH_\TB$ using PLoM with ROTB$(n_\optp \Delta t)$ (c, f, i, l).}
    \label{fig:figure3}
\end{figure}
%
\begin{figure}[h]
    \centering
    \begin{subfigure}[b]{0.24\textwidth}
    \centering
        \includegraphics[width=\textwidth]{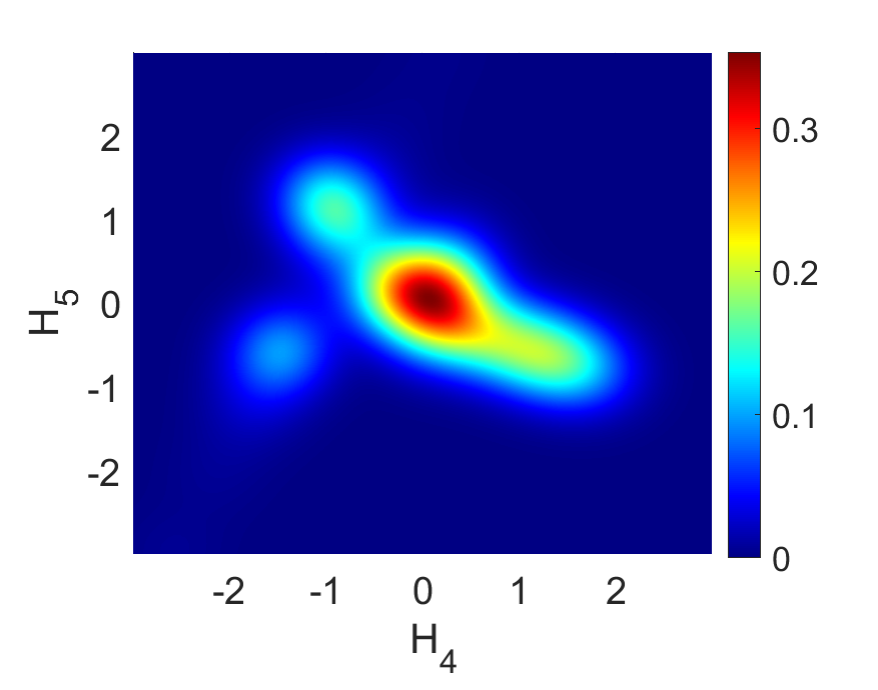}
        \caption{joint pdf of $H_{4}$ with $H_{5}$.}
         \vspace{0.3truecm}
        \label{fig:figure4a}
    \end{subfigure}
    \hfil
    \begin{subfigure}[b]{0.24\textwidth}
        \centering
        \includegraphics[width=\textwidth]{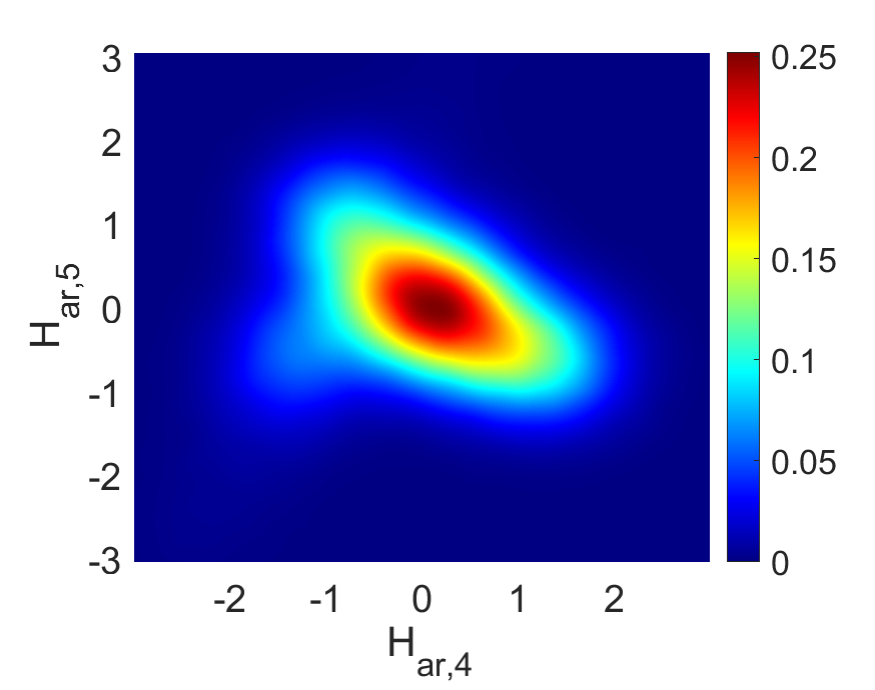}
        \caption{joint pdf of $H_{\ar,4}$ with $H_{\ar,5}$.}
         \vspace{0.3truecm}
        \label{fig:figure4b}
    \end{subfigure}
    \hfil
    \begin{subfigure}[b]{0.24\textwidth}
        \centering
        \includegraphics[width=\textwidth]{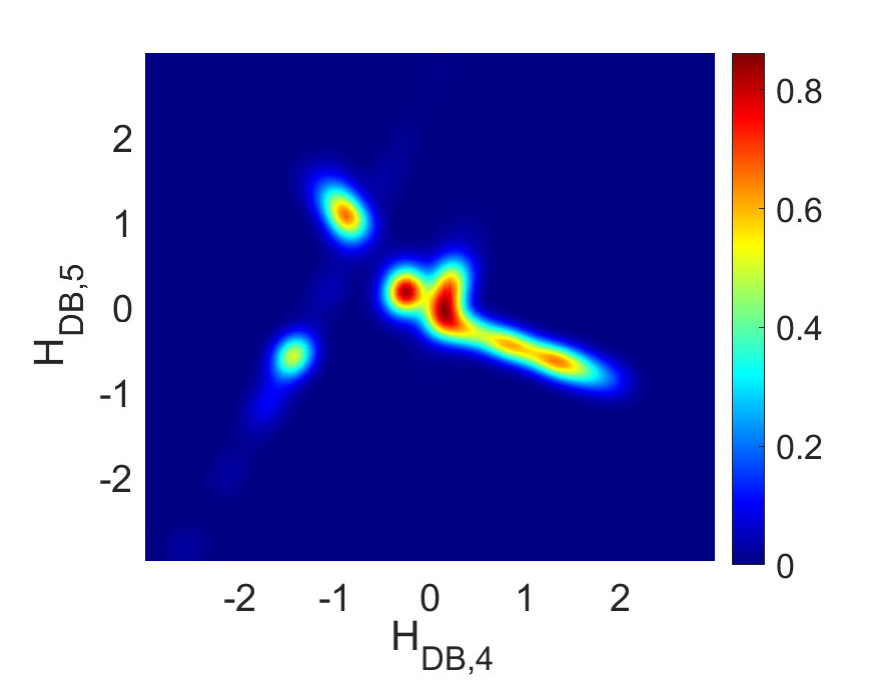}
        \caption{joint pdf of $H_{\DB,4}$ with $H_{\DB,5}$.}
         \vspace{0.3truecm}
        \label{fig:figure4c}
    \end{subfigure}
    \hfil
    \begin{subfigure}[b]{0.24\textwidth}
    \centering
        \includegraphics[width=\textwidth]{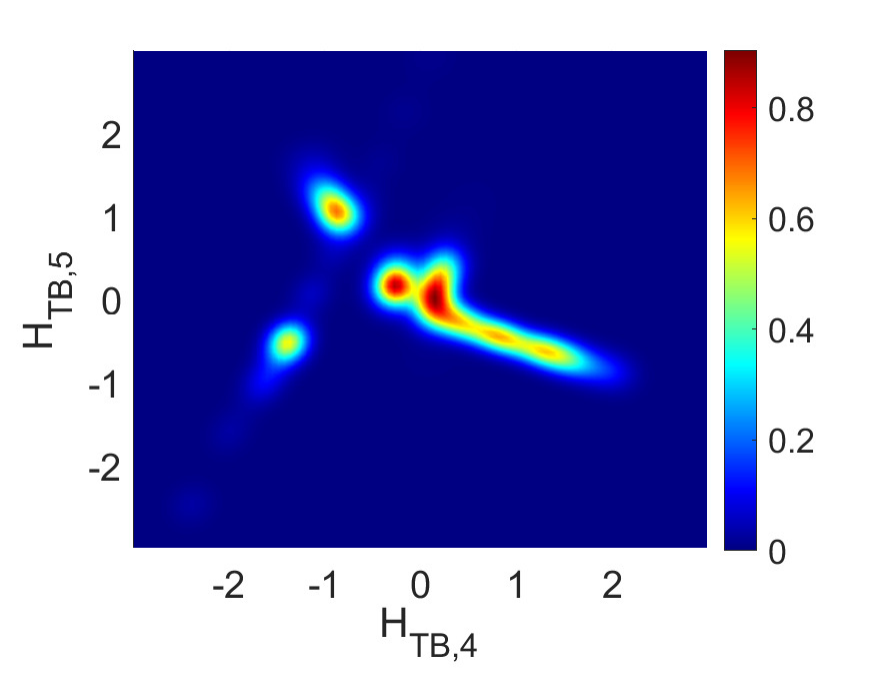}
        \caption{joint pdf of $H_{\TB,4}$ with $H_{\TB,5}$ at time $n_\optp\, \Delta t$.}
        \label{fig:figure4d}
    \end{subfigure}
    \caption{Application 1. Joint probability density function of components $4$ with $5$ of $\bfH$ estimated with the $n_d$ realizations of the training dataset (a) and estimated with $n_\ar$ learned realizations, for $\bfH_\ar$ using MCMC without PLoM (b), for $\bfH_\DB$ using PLoM with RODB (c), and for $\bfH_\TB$ using PLoM with ROTB$(n_\optp \Delta t)$ (d).}
    \label{fig:figure4}
\end{figure}
%
%
%
\begin{figure}[h]
    \centering
    \begin{subfigure}[b]{0.25\textwidth}
    \centering
        \includegraphics[width=\textwidth]{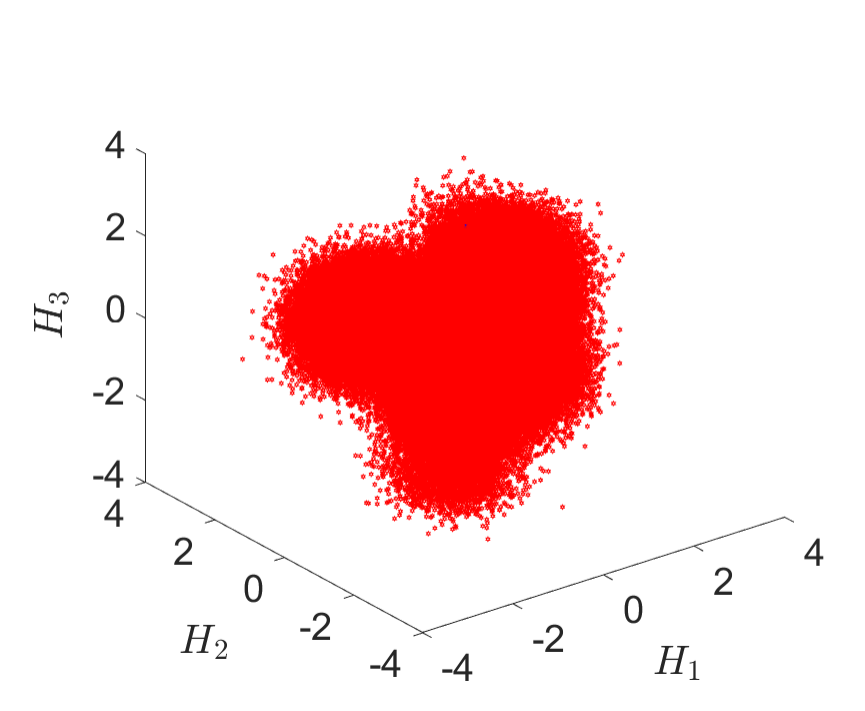}
        \caption{clouds for $(H_{\ar,1},H_{\ar,2},H_{\ar,3})$.}
         \vspace{0.3truecm}
        \label{fig:figure5a}
    \end{subfigure}
    \hfil
    \begin{subfigure}[b]{0.25\textwidth}
        \centering
        \includegraphics[width=\textwidth]{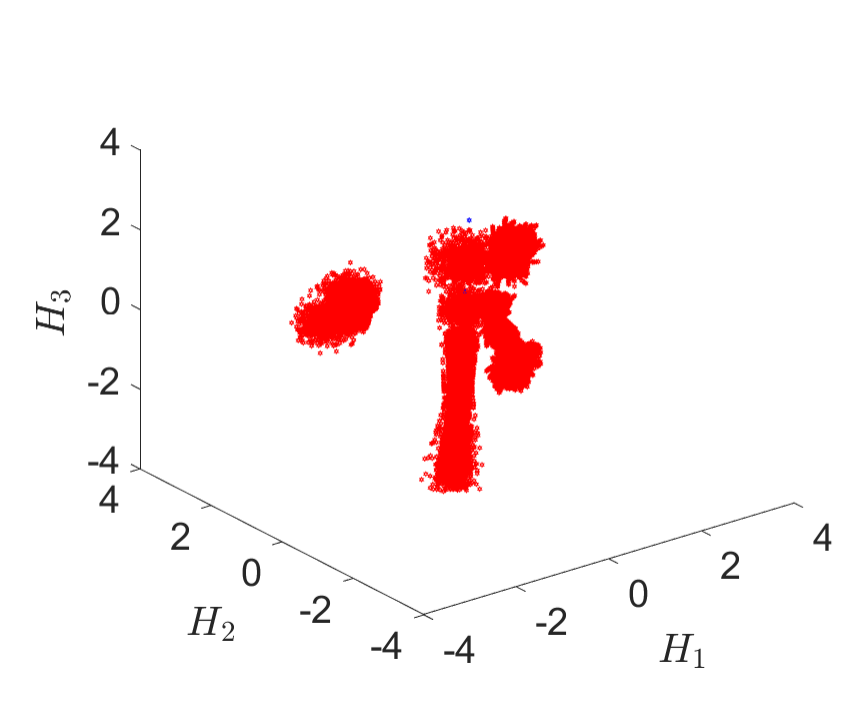}
         \caption{clouds for $(H_{\DB,1},H_{\DB,2},H_{\DB,3})$.}
          \vspace{0.3truecm}
        \label{fig:figure5b}
    \end{subfigure}
    \hfil
    \begin{subfigure}[b]{0.25\textwidth}
        \centering
        \includegraphics[width=\textwidth]{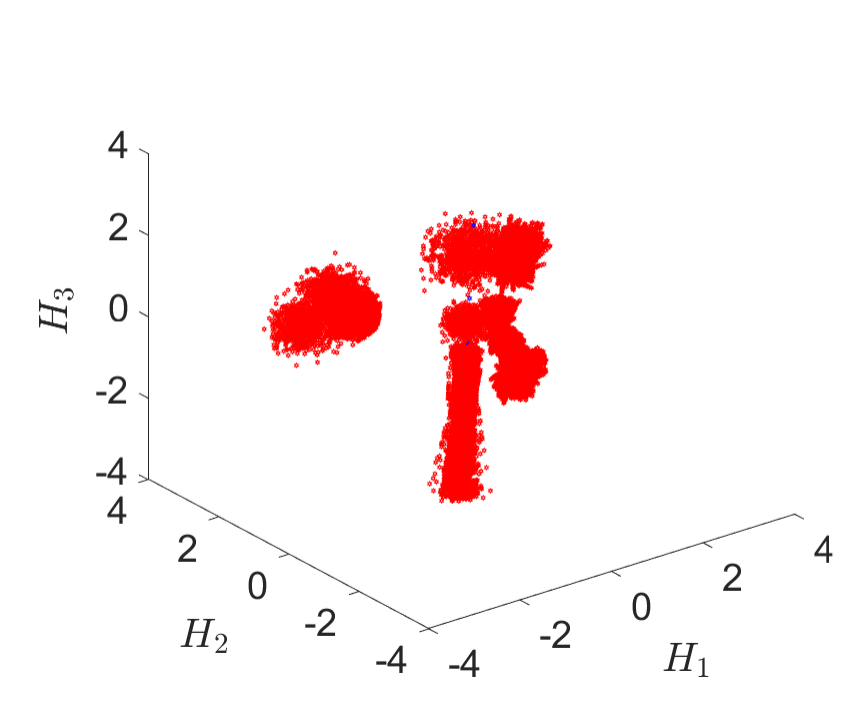}
         \caption{clouds for $(H_{\TB,1},H_{\TB,2},H_{\TB,3})$ at time $n_\optp\,\Delta t$.}
        \label{fig:figure5c}
    \end{subfigure}
    %
    \centering
    \begin{subfigure}[b]{0.25\textwidth}
    \centering
        \includegraphics[width=\textwidth]{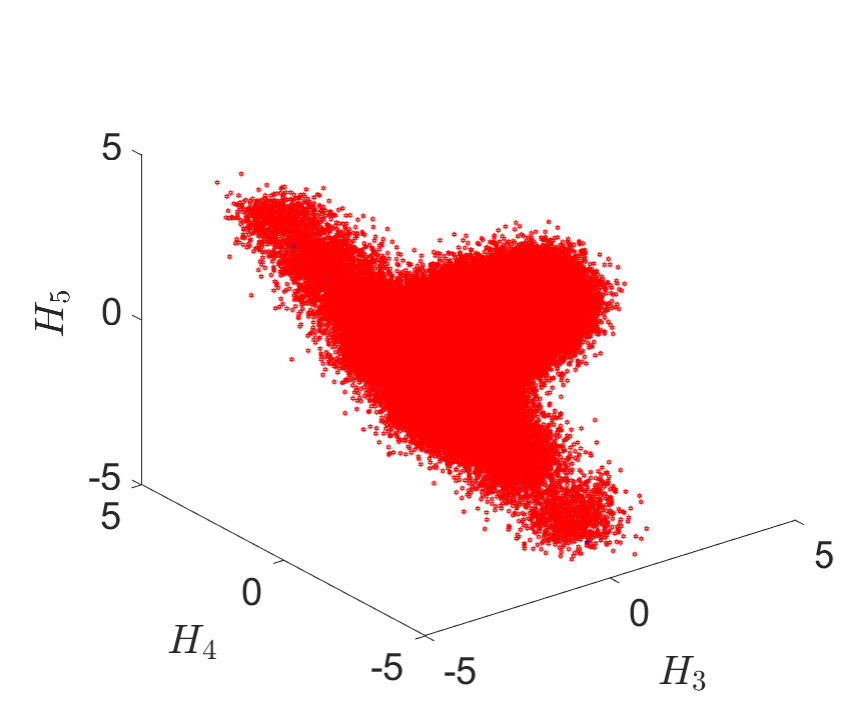}
        \caption{clouds for $(H_{\ar,3},H_{\ar,4},H_{\ar,5})$.}
        \vspace{0.3truecm}
        \label{fig:figure5d}
    \end{subfigure}
    \hfil
    \begin{subfigure}[b]{0.25\textwidth}
        \centering
        \includegraphics[width=\textwidth]{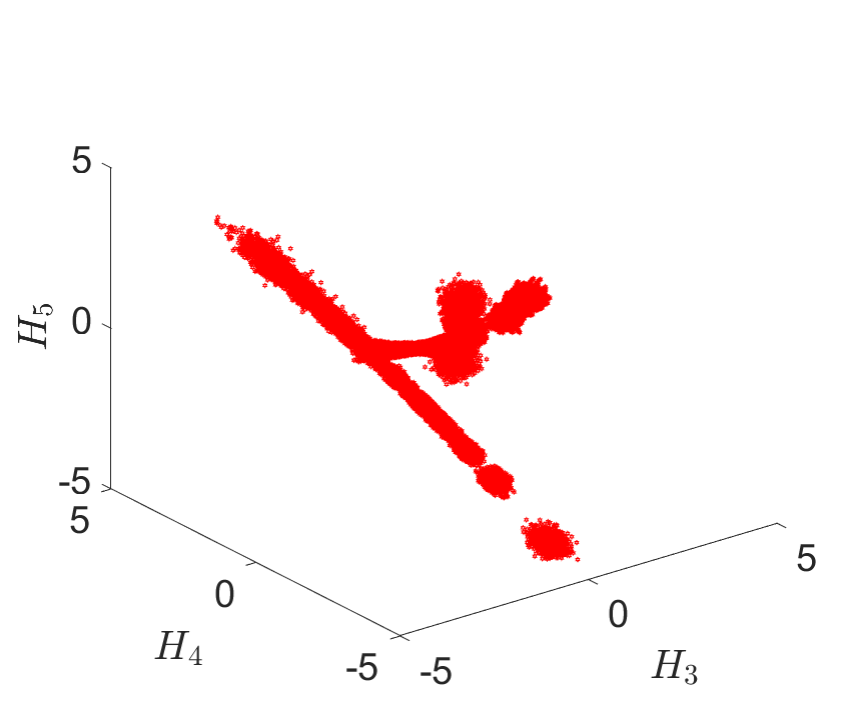}
        \caption{clouds for $(H_{\DB,3},H_{\DB,4},H_{\DB,5})$.}
         \vspace{0.3truecm}
        \label{fig:figure5e}
    \end{subfigure}
    \hfil
    \begin{subfigure}[b]{0.25\textwidth}
        \centering
        \includegraphics[width=\textwidth]{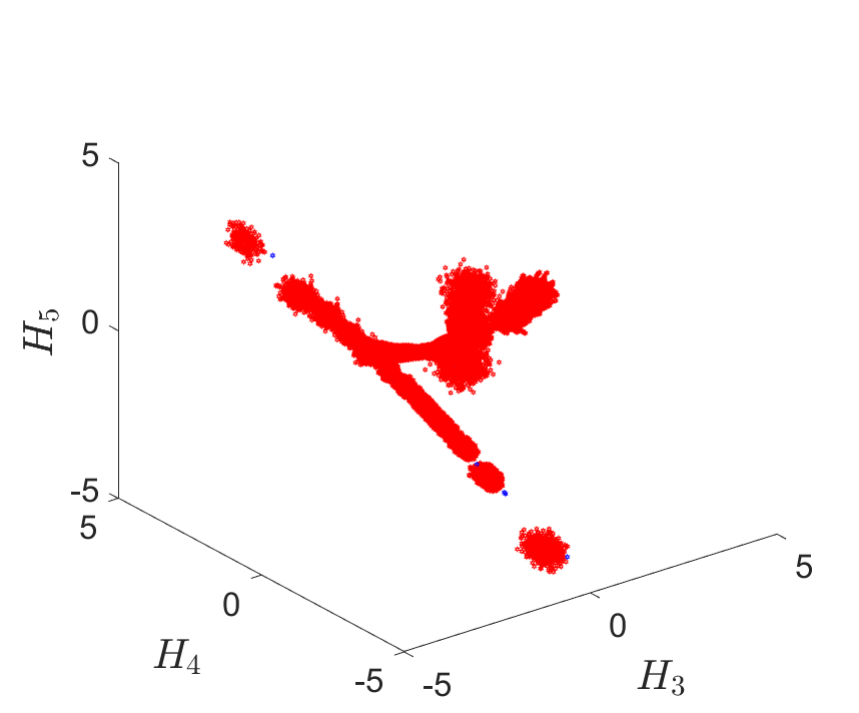}
        \caption{clouds for $(H_{\TB,3},H_{\TB,4},H_{\TB,5})$ at time $n_\optp\,\Delta t$.}
        \label{fig:figure5f}
    \end{subfigure}
    \caption{Application 1. Clouds of $n_\ar$ points corresponding to $n_\ar$ learned realizations, for components $1$, $2$, $3$ (a,b,c) and components $3$, $4$, $5$ (d,e,f), for $\bfH_\ar$ using MCMC without PLoM (a,d), for $\bfH_\DB$ using PLoM with RODB (b,e), and for $\bfH_\TB$ using PLoM with ROTB$(n_\optp \Delta t)$ (c,f).}
    \label{fig:figure5}
\end{figure}
%
%
%
\begin{figure}[h]
    \centering
    \begin{subfigure}[b]{0.25\textwidth}
    \centering
        \includegraphics[width=\textwidth]{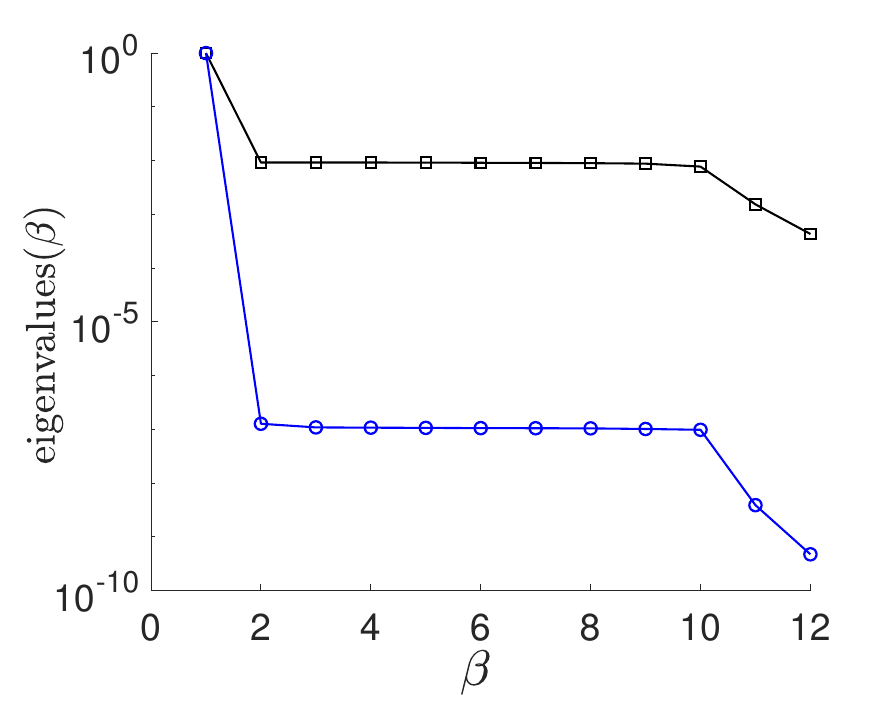}
        \caption{eigenvalues $\beta\mapsto \hat b_{\DM,\,\beta}$ (square) and $\beta\mapsto \tilde b_{\beta}(n_\optp\, \delta t)$ (circle).}
        \label{fig:figure6a}
    \end{subfigure}
    \hfil
    \begin{subfigure}[b]{0.25\textwidth}
        \centering
        \includegraphics[width=\textwidth]{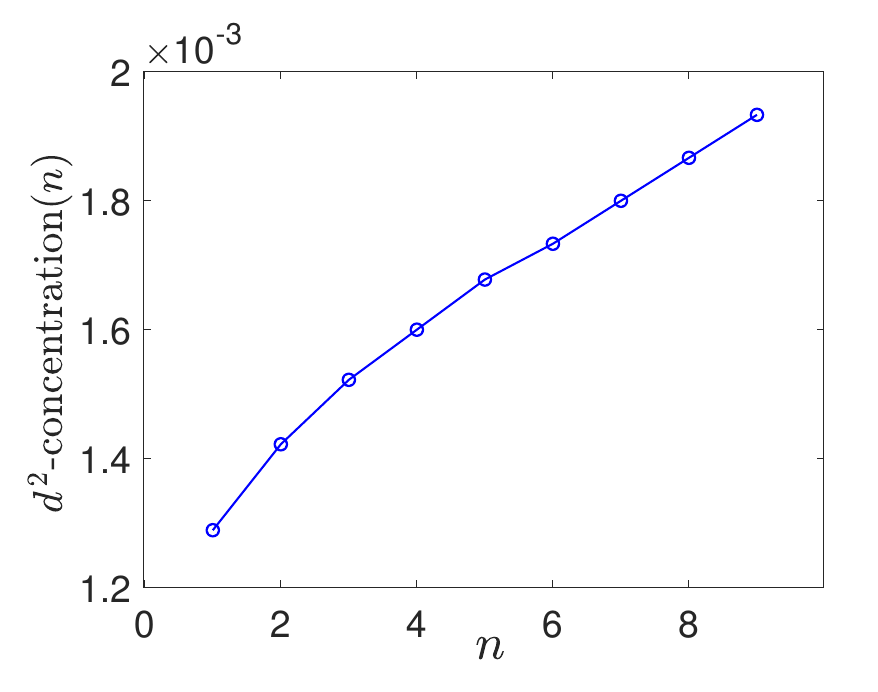}
         \caption{$n\mapsto \hat d^{\,2}(m_\optp;\nDeltat)/\nu$.}
         \vspace{0.3truecm}
        \label{fig:figure6b}
    \end{subfigure}
    \hfil
    \begin{subfigure}[b]{0.25\textwidth}
        \centering
        \includegraphics[width=\textwidth]{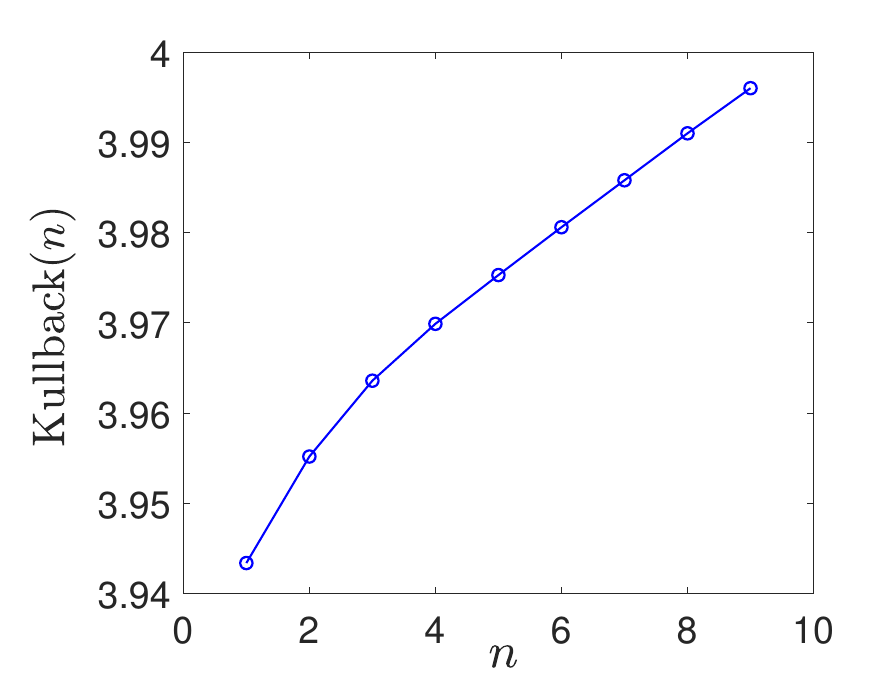}
        \caption{$n\mapsto \hat D(p_\TB(\cdot\, ,\nDeltat)\, \Vert \, p_\bfH)$ (Kullback).}
         \vspace{0.3truecm}
        \label{fig:figure6c}
    \end{subfigure}
    \centering
    \begin{subfigure}[b]{0.25\textwidth}
    \centering
        \includegraphics[width=\textwidth]{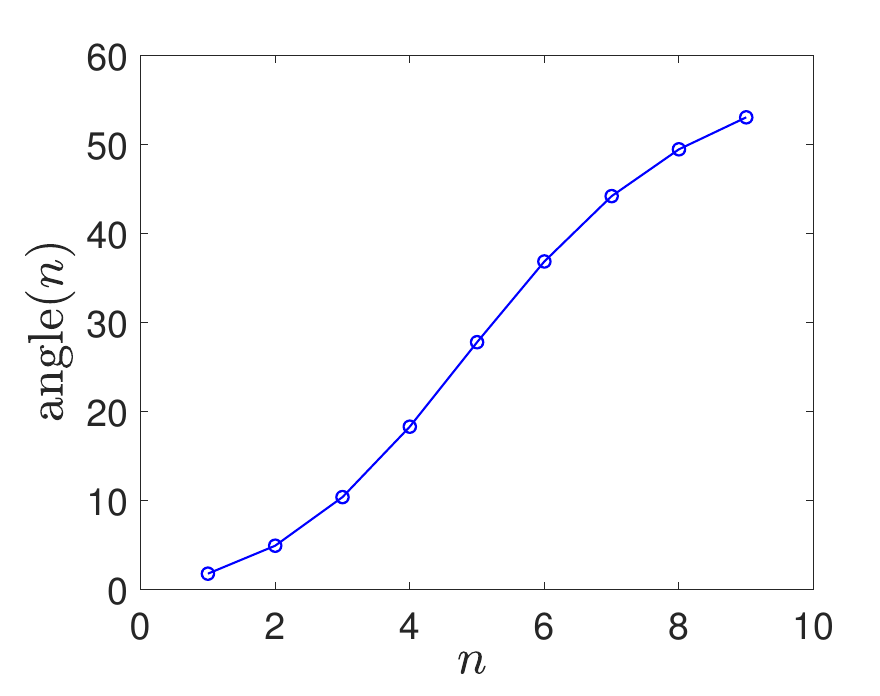}
          \caption{$n\mapsto \gamma(\nDeltat)$ (angle in degree).}
           \vspace{0.3truecm}
        \label{fig:figure6d}
    \end{subfigure}
    \hfil
    \begin{subfigure}[b]{0.25\textwidth}
        \centering
        \includegraphics[width=\textwidth]{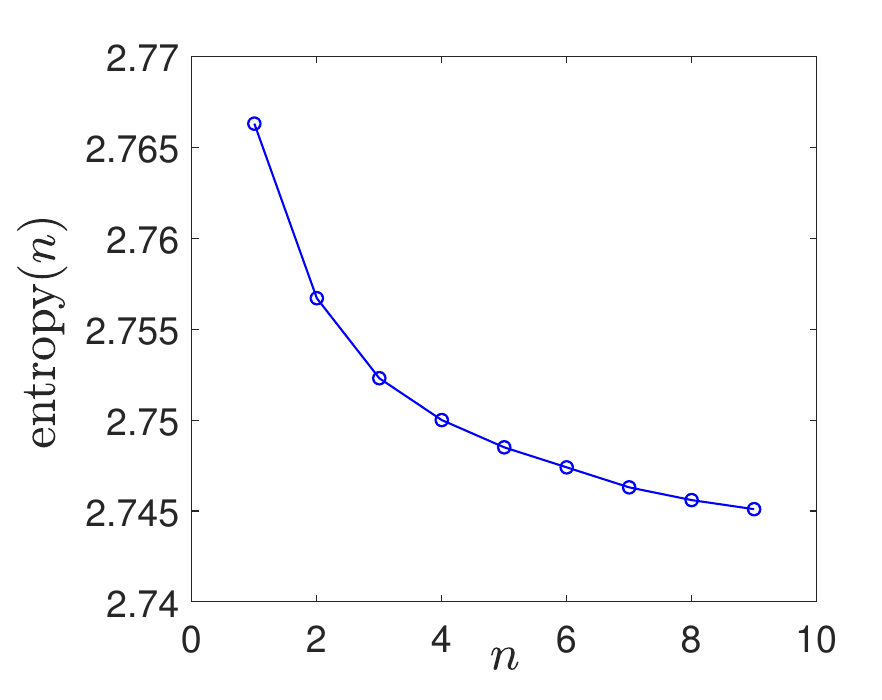}
        \caption{$n\mapsto \hat S_\TB(\nDeltat)$ (entropy).}
         \vspace{0.3truecm}
        \label{fig:figure6e}
    \end{subfigure}
    \hfil
    \begin{subfigure}[b]{0.25\textwidth}
        \centering
        \includegraphics[width=\textwidth]{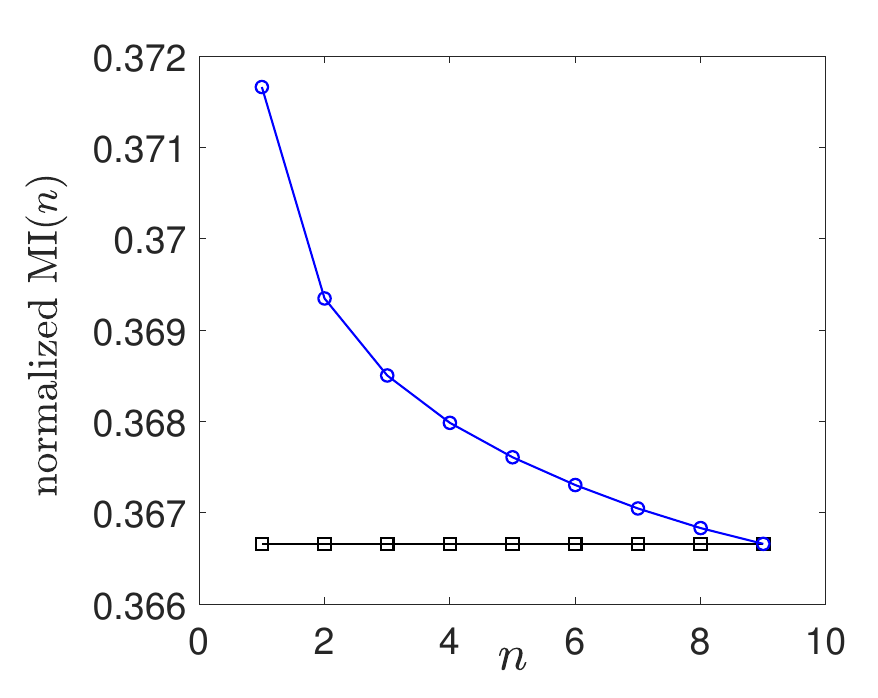}
        \caption{$n\mapsto \hat I_\normp(\bfH)$ (square) and $n\mapsto \hat I_\normp(\bfH_\TB;\nDeltat)$ (circle)(normalized MI).}
        \label{fig:figure6f}
    \end{subfigure}
    \caption{Application 1. Functions characterizing the reduced-order transient basis ROTB$(\nDeltat)$ as a function of time
    $\nDeltat$: eigenvalues of $[K_\DM]$ and of symmetrized $[\tilde K(\nDeltat)]$ (a); probability-measure concentration with
    $d^{\,2}/\nu$-criterion (b) and with Kullback criterion (c); angle between the subspaces spanned by RODB and ROTB$(\nDeltat)$ (d); entropy of pdf $p_\TB(\cdot\, ; \nDeltat)$ (e); normalized mutual information (MI) of pdf $p_\bfH$ and $p_\TB(\cdot\, ; \nDeltat)$ (f).}
    \label{fig:figure6}
\end{figure}
%
%
%
%
\subsection{Results for Application 2}
\label{Section8.7}

\noindent (i) Figure~\ref{fig:figure7} displays the graphs of the probability density function (pdf)  of components $3$, $5$, $6$, and $8$ for $\bfH$ estimated with the $n_d$ realizations of the training dataset, and the pdf estimated with $n_\ar$ learned realizations, for $\bfH_\ar$ using MCMC without PLoM (a,d,g,j), for $\bfH_\DB$ using  PLoM with RODB (b, e, h, k), and for $\bfH_\TB$ using PLoM with ROTB$(n_\optp \Delta t)$ (c, f, i, l).}

\noindent (ii)  Figure~\ref{fig:figure8} shows the joint probability density function of components $6$ and $8$ of $\bfH$ estimated with the $n_d$ realizations of the training dataset (a) and estimated with $n_\ar$ learned realizations for $\bfH_\ar$ using MCMC without PLoM (b), for $\bfH_\DB$ using PLoM with RODB (c), and for $\bfH_\TB$ using PLoM with ROTB$(n_\optp \Delta t)$ (d).

\noindent (iii)  In Fig.~\ref{fig:figure9}, the clouds of $n_\ar$ points corresponding to $n_\ar$ learned realizations can be seen for components $2$, $5$, $6$ (a,b,c) and components $3$, $6$, $8$ (d,e,f). These are shown for $\bfH_\ar$ using MCMC without PLoM (a,d), for $\bfH_\DB$ using PLoM with RODB (b,e), and for $\bfH_\TB$ using PLoM with ROTB$(n_\optp \Delta t)$ (c,f).

\noindent (iv) Figure~\ref{fig:figure10} plots the functions that characterize the reduced-order transient basis ROTB$(\nDeltat)$ as a function of time $\nDeltat$:
\begin{itemize}
\item The eigenvalues of matrix $[K_\DM]$ and those of the of symmetrized matrix $[\tilde K(\nDeltat)]$ are shown in
Fig.~\ref{fig:figure10a}.
\item The probability-measure concentration using the $d^{\,2}/\nu$-criterion is shown in Fig.~\ref{fig:figure10b}.
For the learning without PLoM, the $d^{\,2}$-concentration is $0.951$, which shows that the concentration is lost, and for the PLoM with the RODM, the concentration is $0.0091$, which shows that the concentration is preserved.
\item The other criterion of the probability-measure concentration is given by Kullback measure, shown in Fig.~\ref{fig:figure10c}.
For the learning without PLoM, Kullbach is $0.4594$, and for the PLoM with the RODM, Kullback is $3.4473$. Comparing Figs.~\ref{fig:figure10b} and \ref{fig:figure10c} shows, similarly to Application 1 that the two criteria are consistent and give the same analysis of the concentration.
\item The angle between the subspaces spanned by RODB and ROTB$(\nDeltat)$ is displayed in Fig.~\ref{fig:figure10d}. It can be seen that, for the optimal time $5\, \Delta t$, the angle is $17.1^\circ$, which is significant, although less than the optimal angle of Application~1. This shows that the two bases are different while the $d^{\,2}$-concentration remains small at $0.0153$.
\item The entropy of pdf $p_\TB(\cdot\, ; \nDeltat)$ is given in Fig.~\ref{fig:figure6e}.
\item The normalized mutual information (MI) of the pdfs $p_\bfH$ and $p_\TB(\cdot\, ; \nDeltat)$ is shown in Fig.~\ref{fig:figure10f}. This figure shows that the optimal value of $n$ is $n_\optp = 5$. Unlike Application~1, the normalized mutual information presents a local minimum, which is also a global minimum over the admissible set $\curC_N$.
    For the non-normalized estimation of the mutual information, we have $\hat I(\bfH) = 3.4994$, $\hat I(\bfH_\TB\, ; n_\optp \Delta t) = 5.8763$, and $\hat I(\bfH_\DM) = 5.9382$. For the normalized one, we have
    $\hat I_\normp(\bfH) = \hat I_\normp(\bfH_\TB\, ; n_\optp \Delta t) = 0.3441$, and $\hat I_\normp(\bfH_\DM) = 0.3477$.
\end{itemize}

\noindent (v) As for Application~1, examination of these figures shows that traditional learning without PLoM gives poor results compared to PLoM, which allows the concentration to be preserved and properly learns the geometry of the probability measure support. We also see that PLoM with the optimal ROTB provides an improvement in learning compared to PLoM with RODM. However, this improvement is less than in the case of Application~1 for which the data are much more heterogeneous (in correlation with the geometric complexity of the probability-measure support). Nevertheless, PLoM with the optimal ROTB is an improvement over PLoM with RODB and consequently, should improve the estimates of conditional statistics thanks to better learning of the joint probability measure.

\begin{figure}[!t]
    \centering
    \begin{subfigure}[b]{0.25\textwidth}
    \centering
        \includegraphics[width=\textwidth]{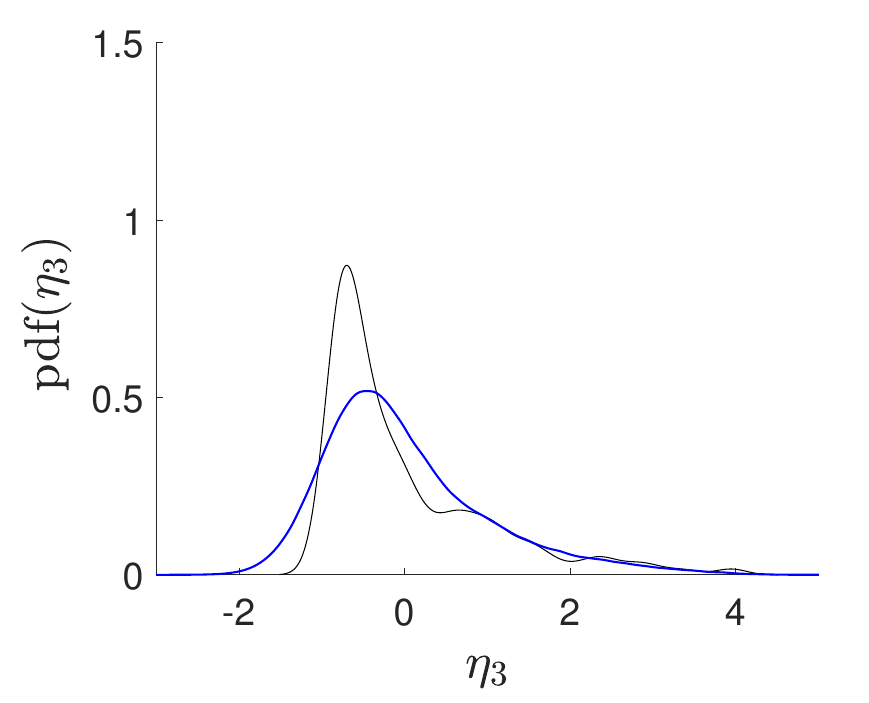}
        \caption{pdf of $H_3$ and $H_{\ar,3}$.}
        \label{fig:figure7a}
    \end{subfigure}
    \hfil
    \begin{subfigure}[b]{0.25\textwidth}
        \centering
        \includegraphics[width=\textwidth]{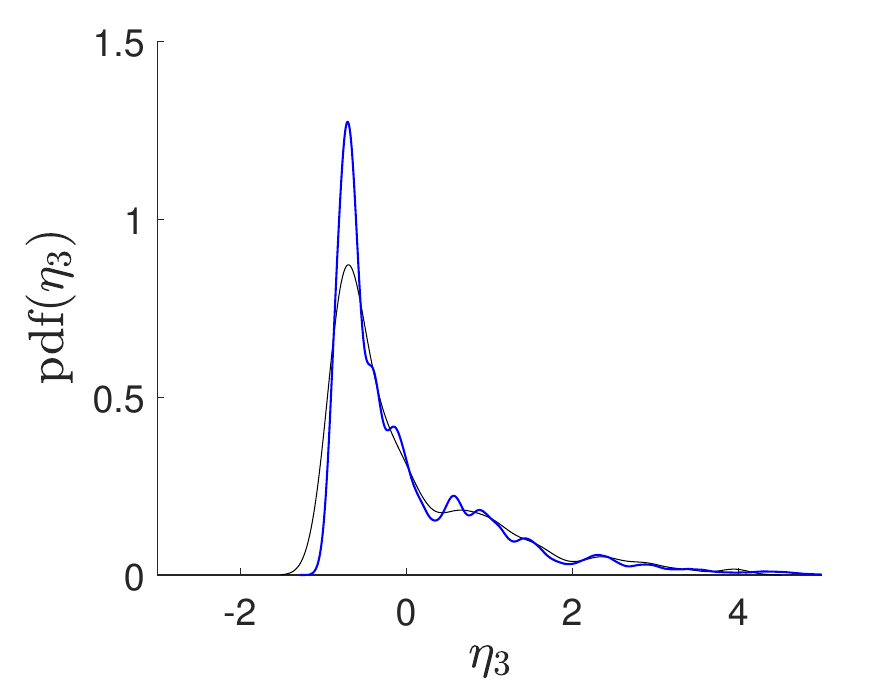}
        \caption{pdf of $H_3$ and $H_{\DB,3}$.}
        \label{fig:figure7b}
    \end{subfigure}
    \hfil
    \begin{subfigure}[b]{0.25\textwidth}
        \centering
        \includegraphics[width=\textwidth]{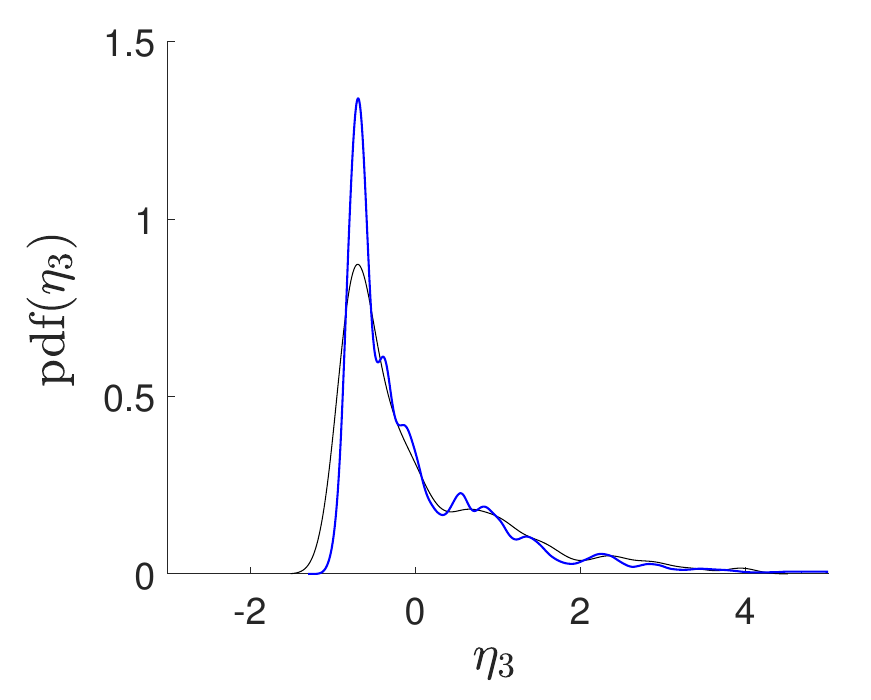}
        \caption{pdf of $H_3$ and $H_{\TB,3}$ at time $n_\optp\,\Delta t$.}
        \label{fig:figure7c}
    \end{subfigure}
    %
    \centering
    \begin{subfigure}[b]{0.25\textwidth}
    \centering
        \includegraphics[width=\textwidth]{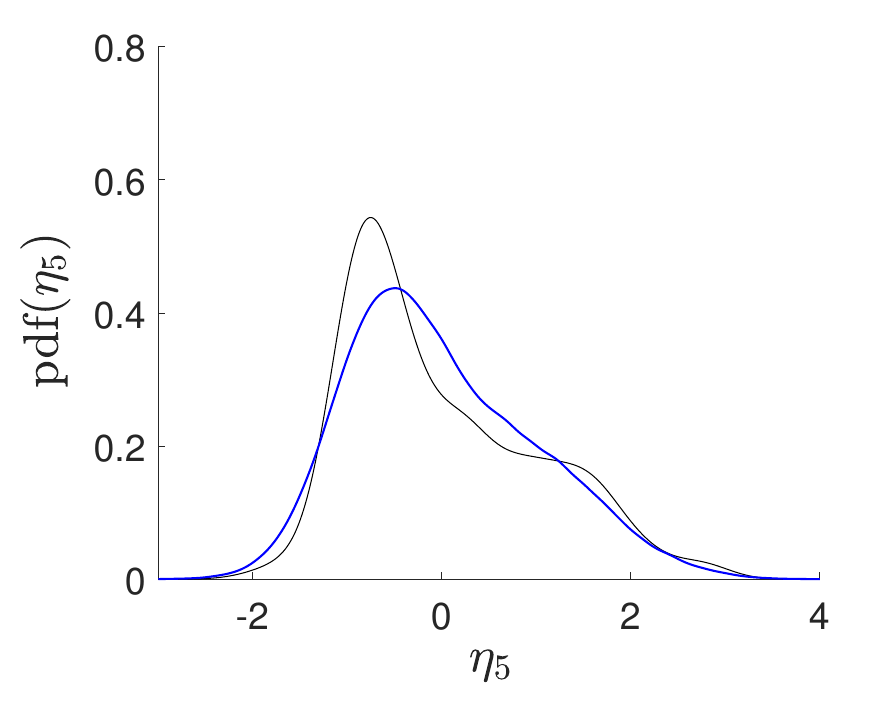}
         \caption{pdf of $H_5$ and $H_{\ar,5}$.}
        \label{fig:figure7d}
    \end{subfigure}
    \hfil
    \begin{subfigure}[b]{0.25\textwidth}
        \centering
        \includegraphics[width=\textwidth]{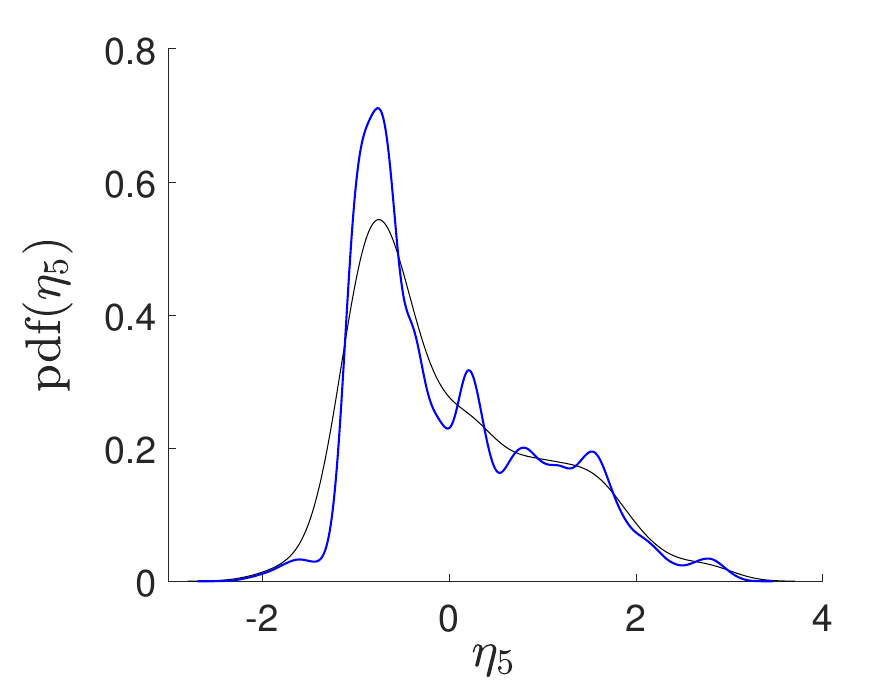}
        \caption{pdf of $H_5$ and $H_{\DB,5}$.}
        \label{fig:figure7e}
    \end{subfigure}
    \hfil
    \begin{subfigure}[b]{0.25\textwidth}
        \centering
        \includegraphics[width=\textwidth]{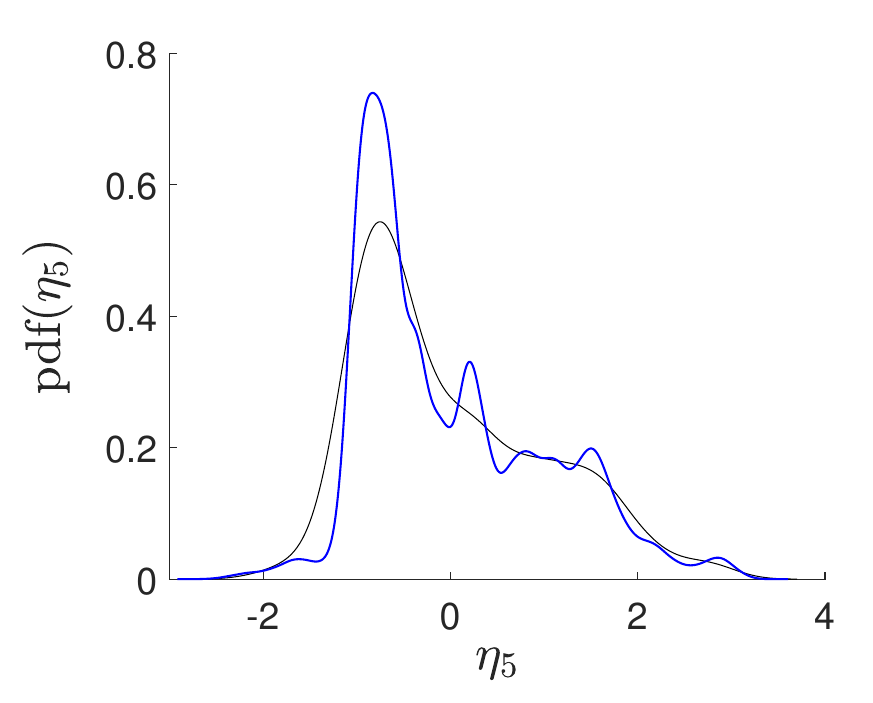}
        \caption{pdf of $H_5$ and $H_{\TB,5}$ at time $n_\optp\,\Delta t$.}
        \label{fig:figure7f}
    \end{subfigure}
    %
    \centering
    \begin{subfigure}[b]{0.25\textwidth}
    \centering
        \includegraphics[width=\textwidth]{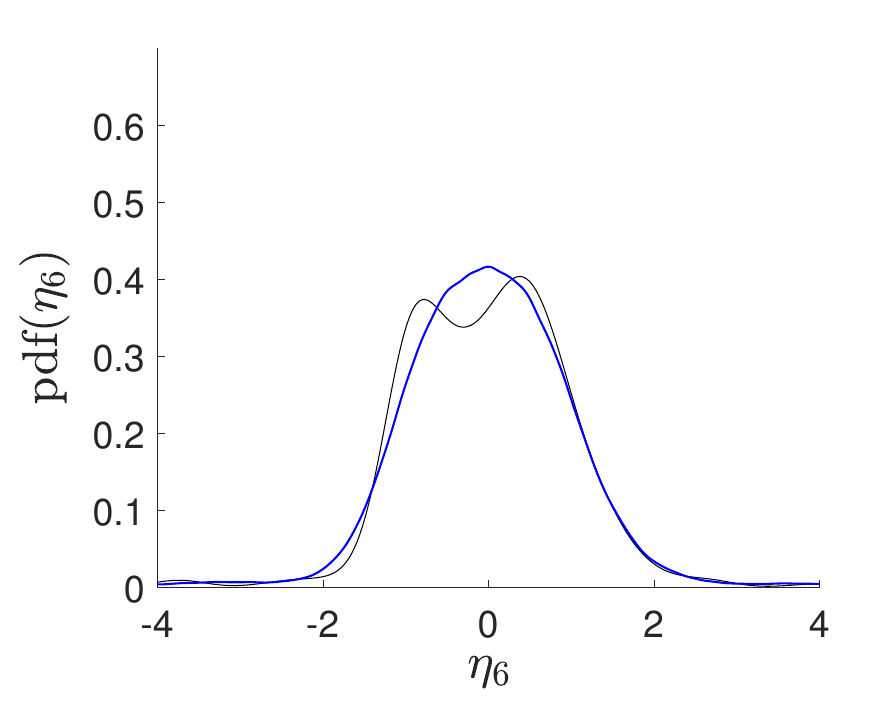}
         \caption{pdf of $H_6$ and $H_{\ar,6}$.}
        \label{fig:figure7g}
    \end{subfigure}
    \hfil
    \begin{subfigure}[b]{0.25\textwidth}
        \centering
        \includegraphics[width=\textwidth]{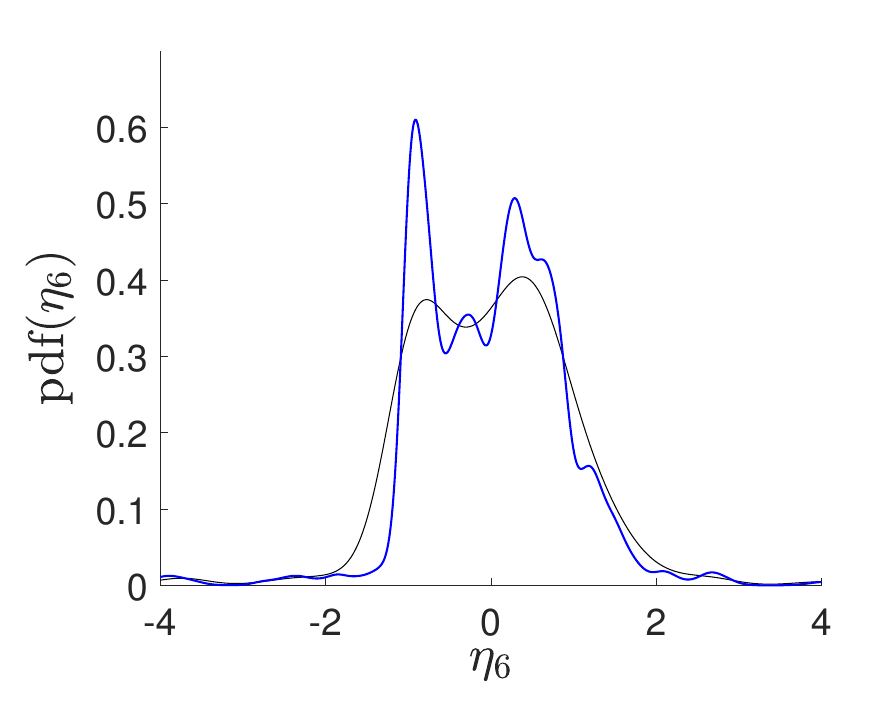}
        \caption{pdf of $H_6$ and $H_{\DB,6}$.}
        \label{fig:figure7h}
    \end{subfigure}
    \hfil
    \begin{subfigure}[b]{0.25\textwidth}
        \centering
        \includegraphics[width=\textwidth]{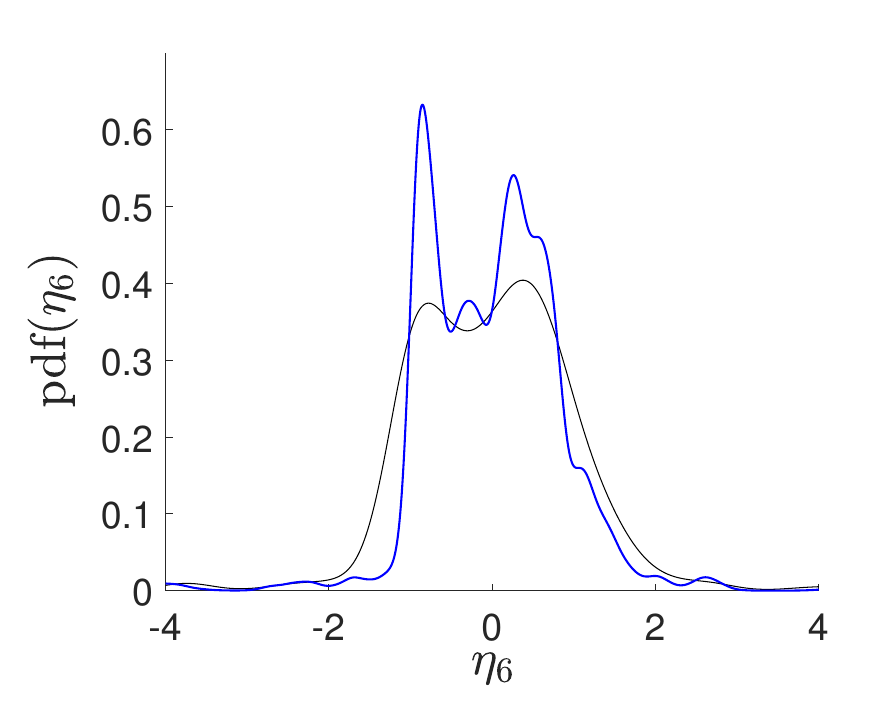}
        \caption{pdf of $H_6$ and $H_{\TB,6}$ at time $n_\optp\,\Delta t$.}
        \label{fig:figure7i}
    \end{subfigure}
    %
    \centering
    \begin{subfigure}[b]{0.25\textwidth}
    \centering
        \includegraphics[width=\textwidth]{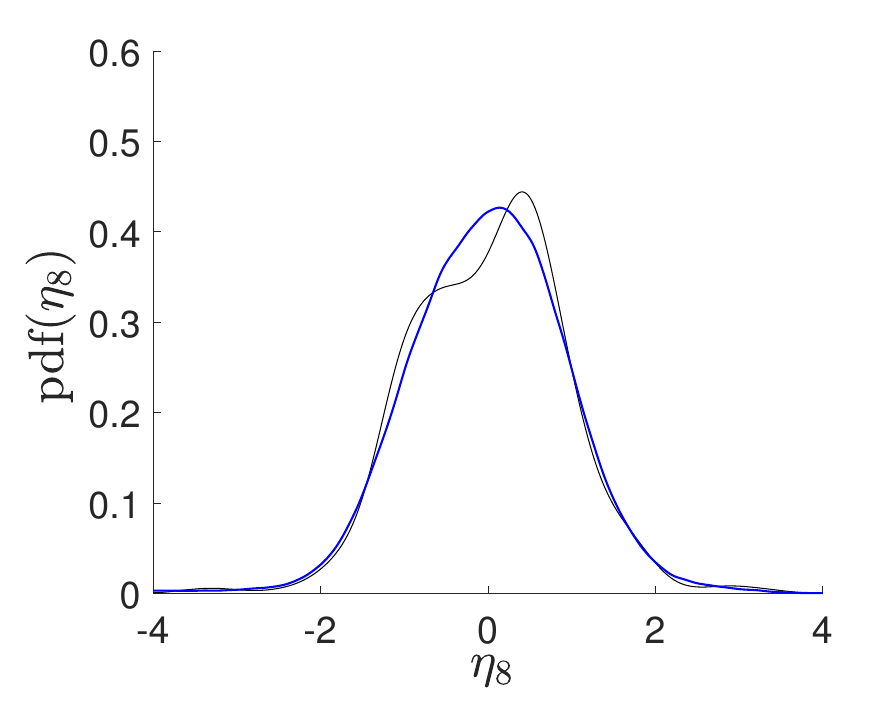}
        \caption{pdf of $H_8$ and $H_{\ar,8}$.}
        \label{fig:figure7j}
    \end{subfigure}
    \hfil
    \begin{subfigure}[b]{0.25\textwidth}
        \centering
        \includegraphics[width=\textwidth]{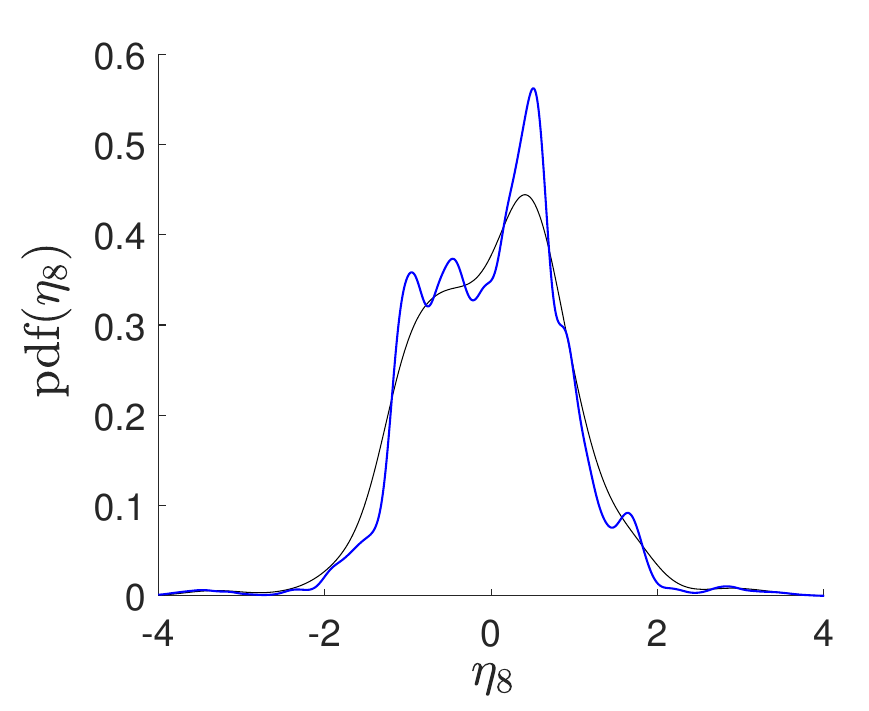}
        \caption{pdf of $H_8$ and $H_{\DB,8}$.}
        \label{fig:figure7k}
    \end{subfigure}
    \hfil
    \begin{subfigure}[b]{0.25\textwidth}
        \centering
        \includegraphics[width=\textwidth]{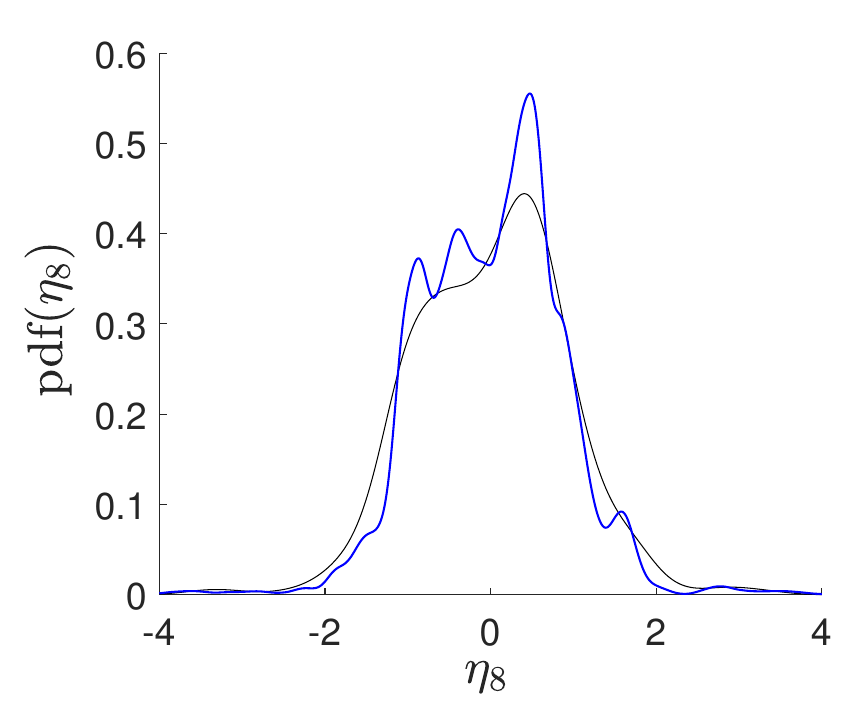}
        \caption{pdf of $H_8$ and $H_{\TB,8}$ at time $n_\optp\,\Delta t$.}
        \label{fig:figure7l}
    \end{subfigure}
    \caption{Application 2. Probability density function (pdf)  of components $3$, $5$, $6$, and $8$ for $\bfH$ estimated with the $n_d$ realizations of the training dataset (thin black line) and pdf estimated with $n_\ar$ learned realizations (thick blue line), for $\bfH_\ar$ using MCMC without PLoM (a,d,g,j), for $\bfH_\DB$ using  PLoM with RODB (b,e,h,k), and for $\bfH_\TB$ using PLoM with ROTB$(n_\optp \Delta t)$ (c, f, i, l).}
    \label{fig:figure7}
\end{figure}
%
\begin{figure}[h]
    \centering
    \begin{subfigure}[b]{0.24\textwidth}
    \centering
        \includegraphics[width=\textwidth]{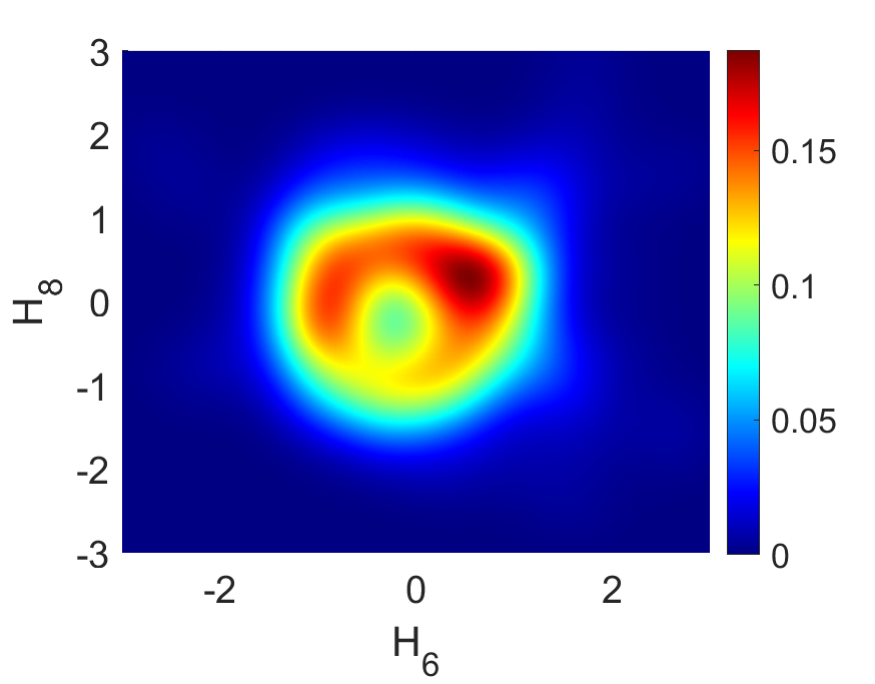}
        \caption{joint pdf of $H_{6}$ with $H_{8}$.}
         \vspace{0.3truecm}
        \label{fig:figure8a}
    \end{subfigure}
    \hfil
    \begin{subfigure}[b]{0.24\textwidth}
        \centering
        \includegraphics[width=\textwidth]{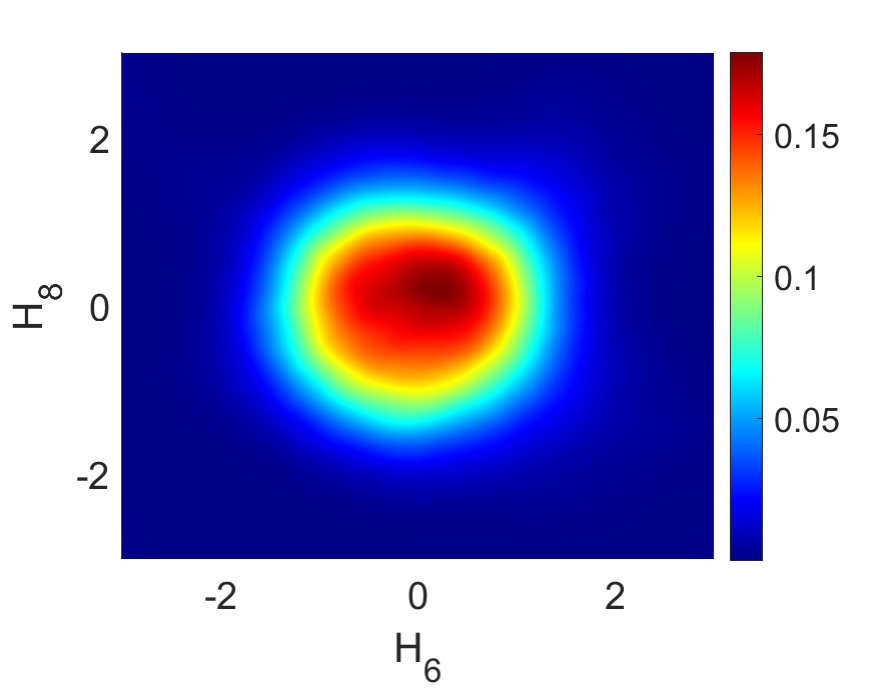}
        \caption{joint pdf of $H_{\ar,6}$ with $H_{\ar,8}$.}
         \vspace{0.3truecm}
        \label{fig:figure8b}
    \end{subfigure}
    \hfil
    \begin{subfigure}[b]{0.24\textwidth}
        \centering
        \includegraphics[width=\textwidth]{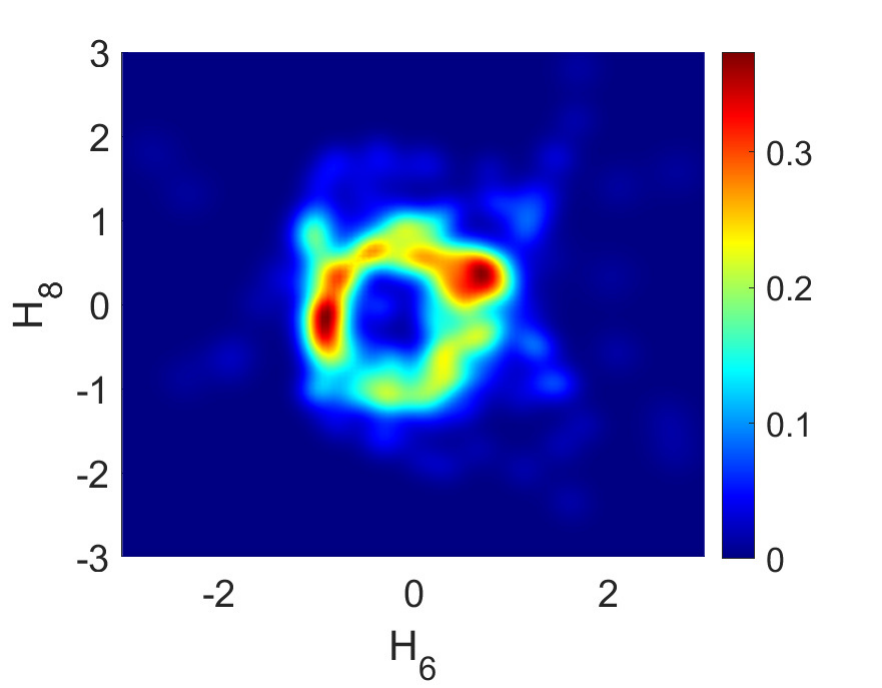}
        \caption{joint pdf of $H_{\DB,6}$ with $H_{\DB,8}$.}
         \vspace{0.3truecm}
        \label{fig:figure8c}
    \end{subfigure}
    \hfil
    \begin{subfigure}[b]{0.24\textwidth}
    \centering
        \includegraphics[width=\textwidth]{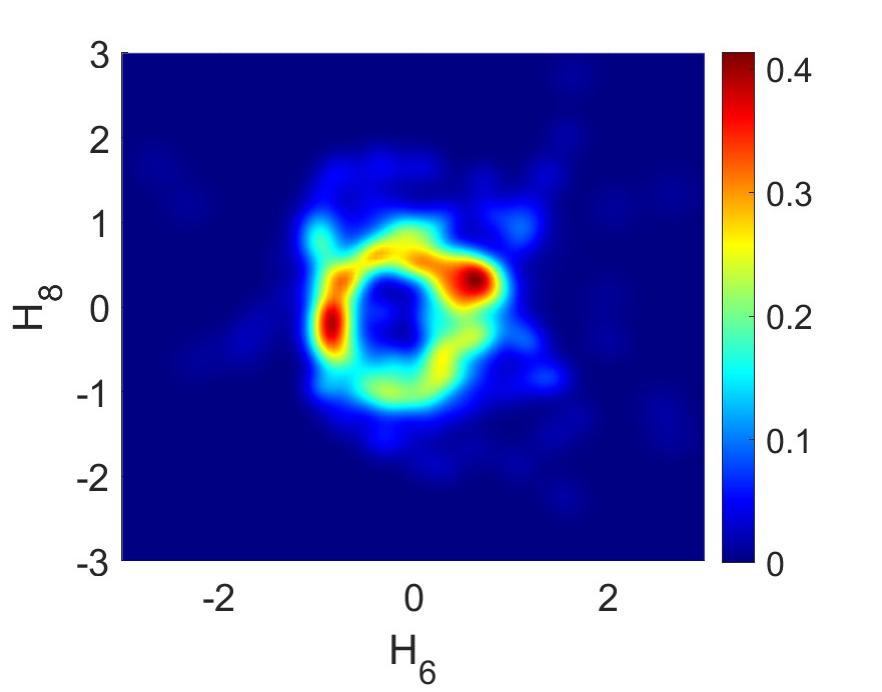}
        \caption{joint pdf of $H_{\TB,6}$ with $H_{\TB,8}$ at time $n_\optp\, \Delta t$.}
        \label{fig:figure8d}
    \end{subfigure}
    \caption{Application 2. Joint probability density function of components $6$ with $8$ of $\bfH$ estimated with the $n_d$ realizations of the training dataset (a) and estimated with $n_\ar$ learned realizations, for $\bfH_\ar$ using MCMC without PLoM (b), for $\bfH_\DB$ using PLoM with RODB (c), and for $\bfH_\TB$ using PLoM with ROTB$(n_\optp \Delta t)$ (d).}
    \label{fig:figure8}
\end{figure}
%
%
%
\begin{figure}[h]
    \centering
    \begin{subfigure}[b]{0.25\textwidth}
    \centering
        \includegraphics[width=\textwidth]{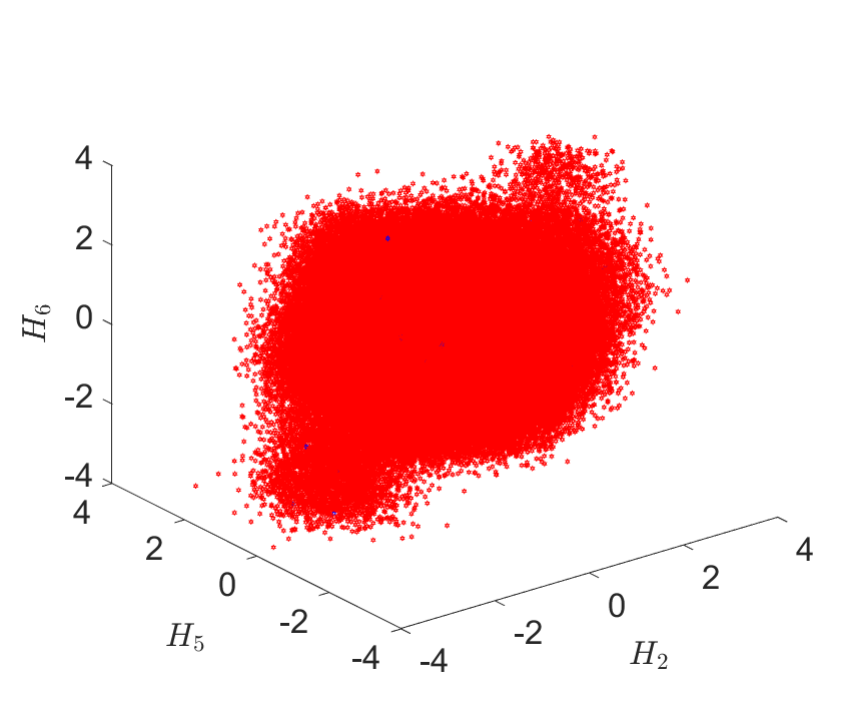}
        \caption{clouds for $(H_{\ar,2},H_{\ar,5},H_{\ar,6})$.}
         \vspace{0.3truecm}
        \label{fig:figure9a}
    \end{subfigure}
    \hfil
    \begin{subfigure}[b]{0.25\textwidth}
        \centering
        \includegraphics[width=\textwidth]{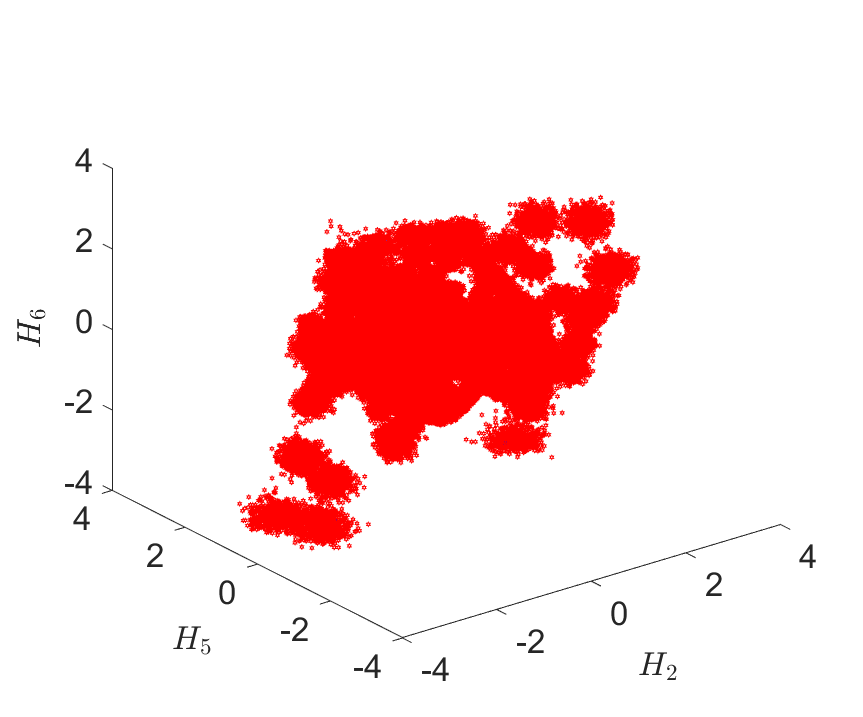}
         \caption{clouds for $(H_{\DB,2},H_{\DB,5},H_{\DB,6})$.}
          \vspace{0.3truecm}
        \label{fig:figure9b}
    \end{subfigure}
    \hfil
    \begin{subfigure}[b]{0.25\textwidth}
        \centering
        \includegraphics[width=\textwidth]{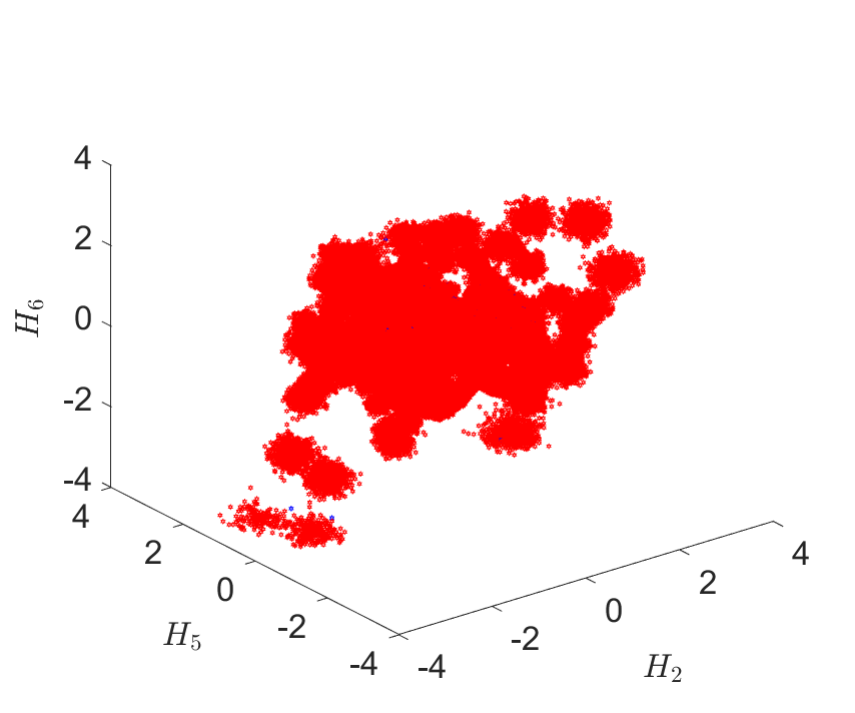}
         \caption{clouds for $(H_{\TB,2},H_{\TB,5},H_{\TB,6})$ at time $n_\optp\,\Delta t$.}
        \label{fig:figure9c}
    \end{subfigure}
    %
    \centering
    \begin{subfigure}[b]{0.25\textwidth}
    \centering
        \includegraphics[width=\textwidth]{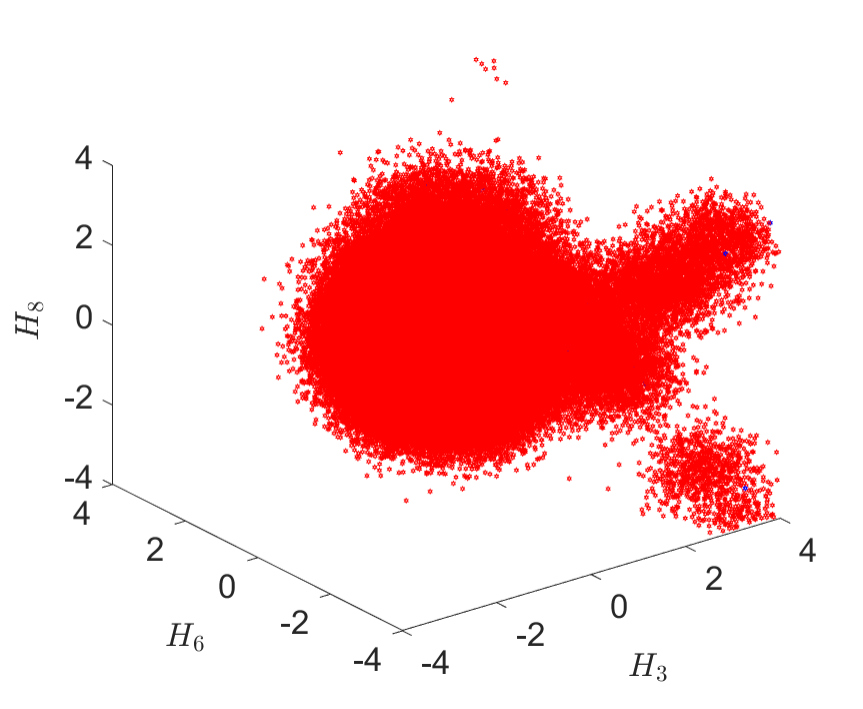}
        \caption{clouds for $(H_{\ar,3},H_{\ar,6},H_{\ar,8})$.}
        \vspace{0.3truecm}
        \label{fig:figure9d}
    \end{subfigure}
    \hfil
    \begin{subfigure}[b]{0.25\textwidth}
        \centering
        \includegraphics[width=\textwidth]{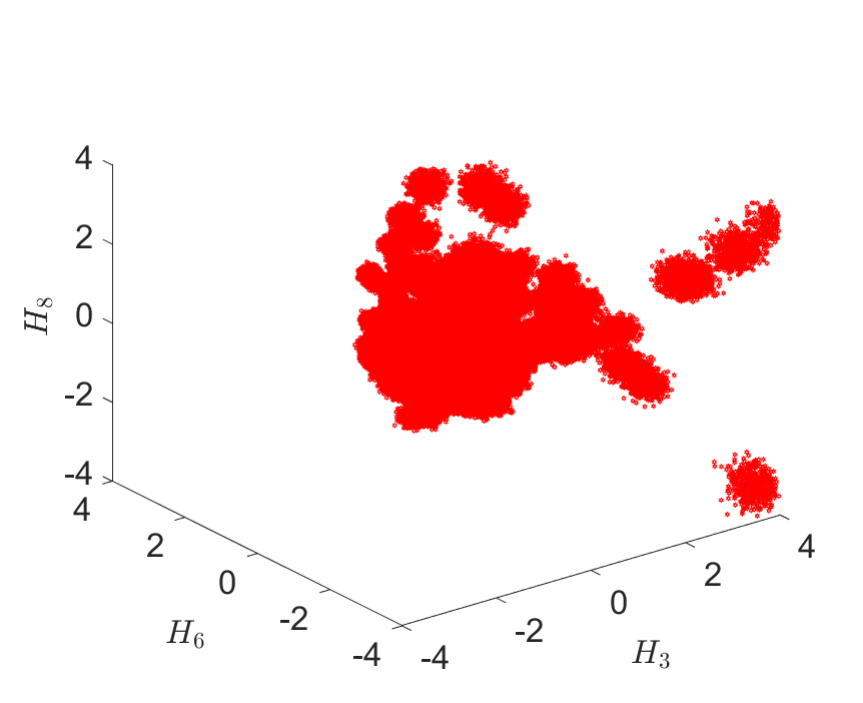}
        \caption{clouds for $(H_{\DB,3},H_{\DB,6},H_{\DB,8})$.}
         \vspace{0.3truecm}
        \label{fig:figure9e}
    \end{subfigure}
    \hfil
    \begin{subfigure}[b]{0.25\textwidth}
        \centering
        \includegraphics[width=\textwidth]{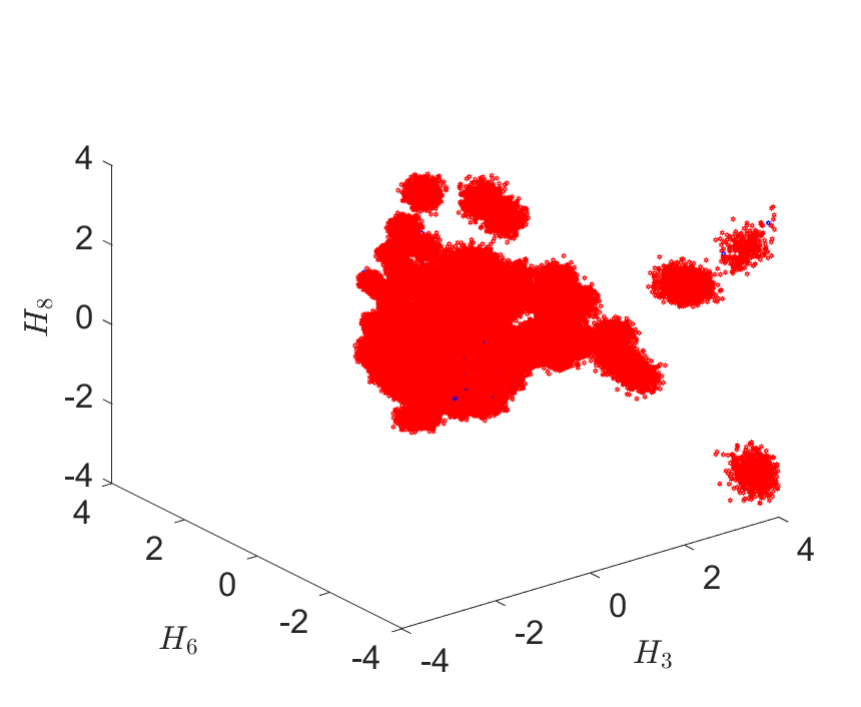}
        \caption{clouds for $(H_{\TB,3},H_{\TB,6},H_{\TB,8})$ at time $n_\optp\,\Delta t$.}
        \label{fig:figure9f}
    \end{subfigure}
    \caption{Application 2. Clouds of $n_\ar$ points corresponding to $n_\ar$ learned realizations, for components $2$, $5$, $6$ (a,b,c) and components $3$, $6$, $8$ (d,e,f), for $\bfH_\ar$ using MCMC without PLoM (a,d), for $\bfH_\DB$ using PLoM with RODB (b,e), and for $\bfH_\TB$ using PLoM with ROTB$(n_\optp \Delta t)$ (c,f).}
    \label{fig:figure9}
\end{figure}
%
%
%
\begin{figure}[h]
    \centering
    \begin{subfigure}[b]{0.25\textwidth}
    \centering
        \includegraphics[width=\textwidth]{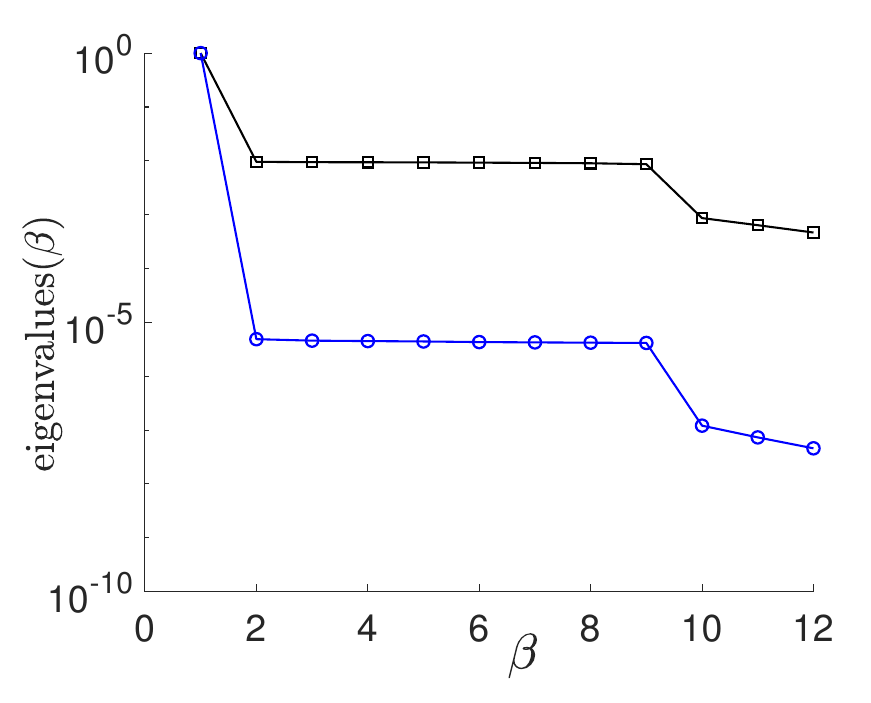}
        \caption{eigenvalues $\beta\mapsto \hat b_{\DM,\,\beta}$ (square) and $\beta\mapsto \tilde b_{\beta}(n_\optp\, \delta t)$ (circle).}
        \label{fig:figure10a}
    \end{subfigure}
    \hfil
    \begin{subfigure}[b]{0.25\textwidth}
        \centering
        \includegraphics[width=\textwidth]{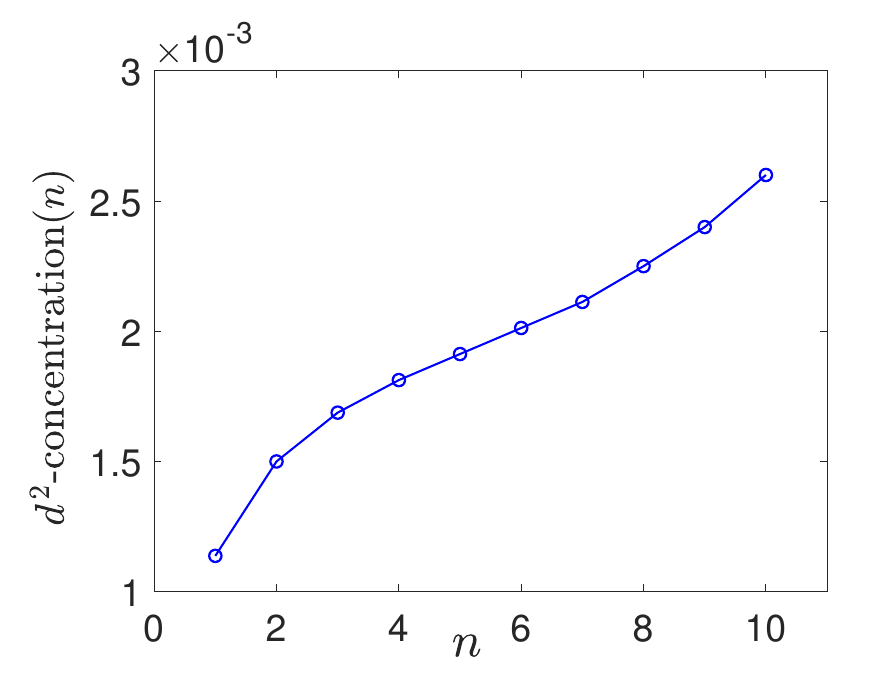}
         \caption{$n\mapsto \hat d^{\,2}(m_\optp;\nDeltat)/\nu$.}
         \vspace{0.3truecm}
        \label{fig:figure10b}
    \end{subfigure}
    \hfil
    \begin{subfigure}[b]{0.25\textwidth}
        \centering
        \includegraphics[width=\textwidth]{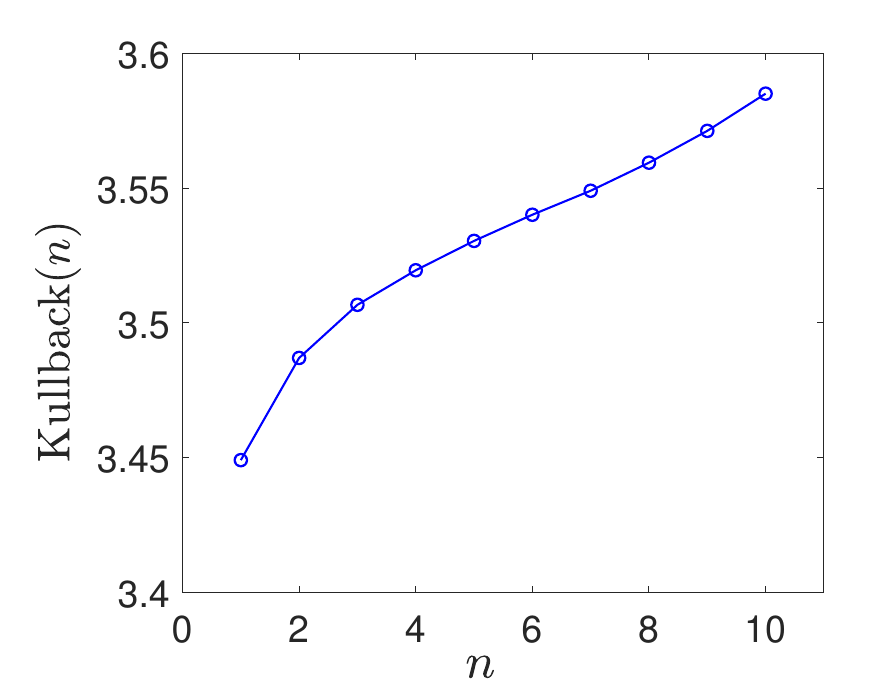}
        \caption{$n\mapsto \hat D(p_\TB(\cdot\, ,\nDeltat)\, \Vert \, p_\bfH)$ (Kullback).}
         \vspace{0.3truecm}
        \label{fig:figure10c}
    \end{subfigure}
    \centering
    \begin{subfigure}[b]{0.25\textwidth}
    \centering
        \includegraphics[width=\textwidth]{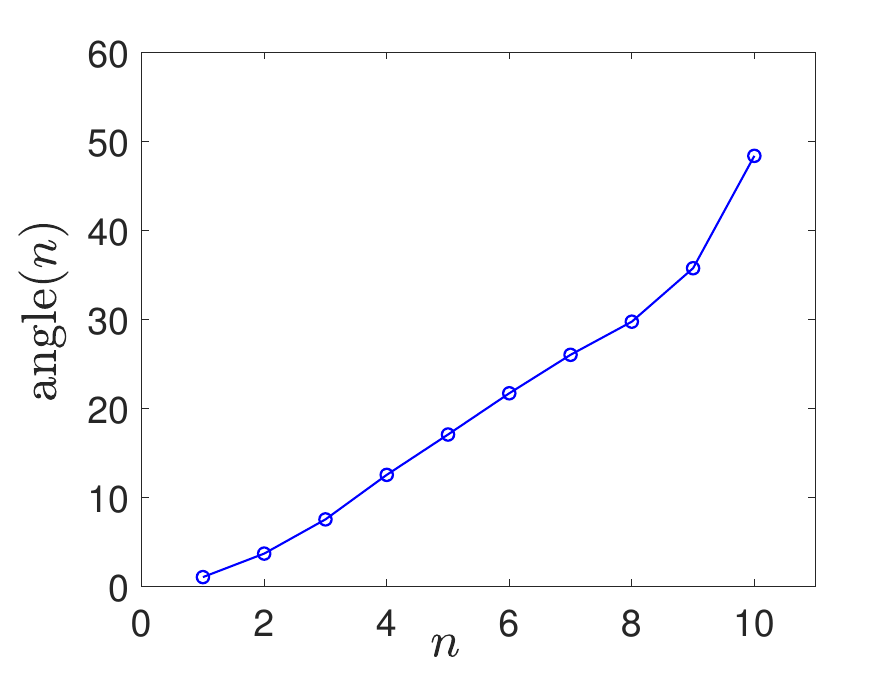}
          \caption{$n\mapsto \gamma(\nDeltat)$ (angle in degree).}
           \vspace{0.3truecm}
        \label{fig:figure10d}
    \end{subfigure}
    \hfil
    \begin{subfigure}[b]{0.25\textwidth}
        \centering
        \includegraphics[width=\textwidth]{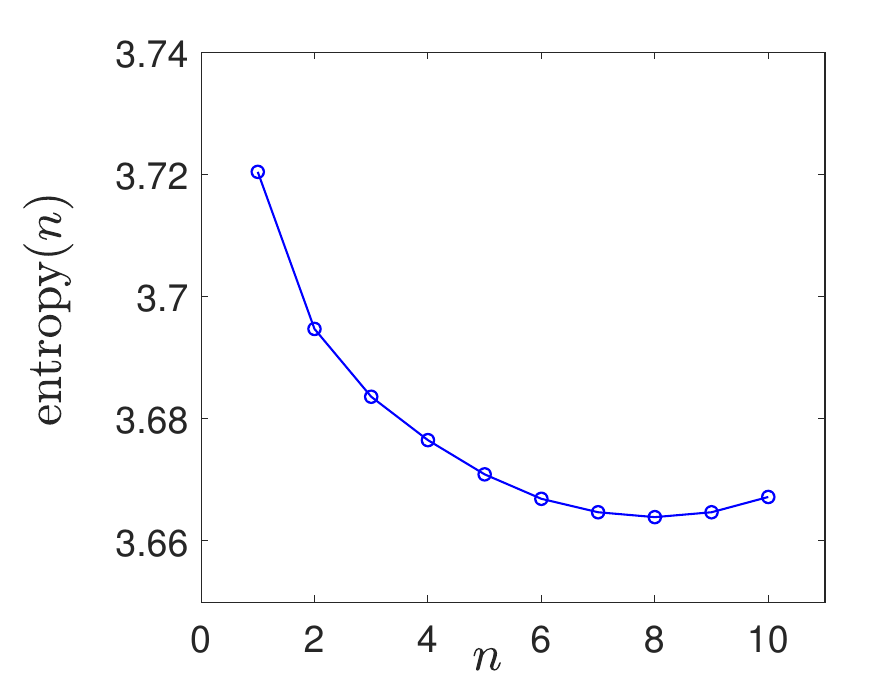}
        \caption{$n\mapsto \hat S_\TB(\nDeltat)$ (entropy).}
         \vspace{0.3truecm}
        \label{fig:figure10e}
    \end{subfigure}
    \hfil
    \begin{subfigure}[b]{0.25\textwidth}
        \centering
        \includegraphics[width=\textwidth]{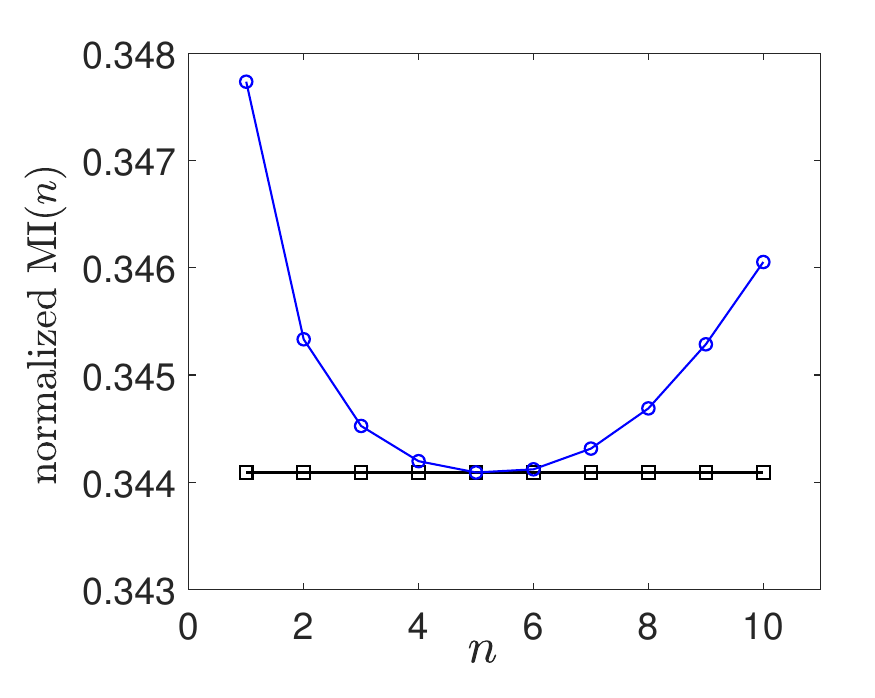}
        \caption{$n\mapsto \hat I_\normp(\bfH)$ (square) and $n\mapsto \hat I_\normp(\bfH_\TB;\nDeltat)$ (circle)(normalized MI).}
        \label{fig:figure10f}
    \end{subfigure}
    \caption{Application 2. Functions characterizing the reduced-order transient basis ROTB$(\nDeltat)$ as a function of time
    $\nDeltat$: eigenvalues of $[K_\DM]$ and of symmetrized $[\tilde K(\nDeltat)]$ (a); measure concentration with $d^{\,2}/\nu$-criterion (b) and with Kullback criteria (c); angle between the subspaces spanned by RODB and ROTB$(\nDeltat)$ (d); entropy of pdf $p_\TB(\cdot\, ; \nDeltat)$ (e); normalized mutual information (MI) of pdf $p_\bfH$ and$p_\TB(\cdot\, ; \nDeltat)$ (f).}
    \label{fig:figure10}
\end{figure}
%
%
%
%
\subsection{Results for Application 3}
\label{Section8.8}

\noindent (i) Figure~\ref{fig:figure11} displays the graphs of the probability density function (pdf)  of components $1$, $6$, $25$, and $40$ for $\bfH$ estimated with the $n_d$ realizations of the training dataset, and the pdf estimated with $n_\ar$ learned realizations, for $\bfH_\ar$ using MCMC without PLoM (a,d,g,j), for $\bfH_\DB$ using  PLoM with RODB (b, e, h, k), and for $\bfH_\TB$ using PLoM with ROTB$(n_\optp \Delta t)$ (c, f, i, l).}

\noindent (ii)  Figure~\ref{fig:figure12} shows the joint probability density function of components $25$ and $40$ of $\bfH$ estimated with the $n_d$ realizations of the training dataset (a) and estimated with $n_\ar$ learned realizations for $\bfH_\ar$ using MCMC without PLoM (b), for $\bfH_\DB$ using PLoM with RODB (c), and for $\bfH_\TB$ using PLoM with ROTB$(n_\optp \Delta t)$ (d).

\noindent (iii)  In Fig.~\ref{fig:figure13}, the clouds of $n_\ar$ points corresponding to $n_\ar$ learned realizations can be seen for components $1$, $6$, $12$ (a,b,c) and components $12$, $25$, $40$ (d,e,f). These are shown for $\bfH_\ar$ using MCMC without PLoM (a,d), for $\bfH_\DB$ using PLoM with RODB (b,e), and for $\bfH_\TB$ using PLoM with ROTB$(n_\optp \Delta t)$ (c,f).

\noindent (iv) Figure~\ref{fig:figure14} plots the functions that characterize the reduced-order transient basis ROTB$(\nDeltat)$ as a function of time $\nDeltat$:
\begin{itemize}
\item The eigenvalues of matrix $[K_\DM]$ and those of the of symmetrized matrix $[\tilde K(\nDeltat)]$ are shown in
Fig.~\ref{fig:figure14a}.
\item The probability-measure concentration using the $d^{\,2}/\nu$-criterion is shown in Fig.~\ref{fig:figure14b}.
For the learning without PLoM, the $d^{\,2}$-concentration is $0.574$, which shows that the concentration is lost, and for the PLoM with the RODM, the concentration is $0.067$, which shows that the concentration is preserved.
\item The other criterion of the probability-measure concentration is given by Kullback measure, shown in Fig.~\ref{fig:figure14c}.
For the learning without PLoM, Kullbach is $14.35$, and for the PLoM with the RODM, Kullback is $5.72$. Comparing Figs.~\ref{fig:figure14b} and \ref{fig:figure14c} shows, similarly to Applications~1 and ~2 that the two criteria are consistent and give the same analysis of the concentration.
\item The angle between the subspaces spanned by RODB and ROTB$(\nDeltat)$ is displayed in Fig.~\ref{fig:figure14d}. It can be seen that, for the optimal time $9\, \Delta t$, the angle is $9.7^\circ$, which is significant, although less than the optimal angle of Applications~1 and ~2. This shows that the two bases are different while the $d^~{\,2}$-concentration remains small at $0.078$.
\item The entropy of pdf $p_\TB(\cdot\, ; \nDeltat)$ is given in Fig.~\ref{fig:figure14e}.
\item The normalized mutual information (MI) of the pdfs $p_\bfH$ and $p_\TB(\cdot\, ; \nDeltat)$ is shown in Fig.~\ref{fig:figure14f}. This figure shows that the optimal value of $n$ is $n_\optp = 9$. The behavior of the normalized mutual information is similar to that of Application~1 and does not present a local minimum as in Application~2.
    For the non-normalized estimation of the mutual information, we have $\hat I(\bfH) = 24.167$, $\hat I(\bfH_\TB\, ; n_\optp \Delta t) = 25.115$, and $\hat I(\bfH_\DM) = 25.673$. For the normalized one, we have
    $\hat I_\normp(\bfH) = \hat I_\normp(\bfH_\TB\, ; n_\optp \Delta t) = 0.1418$ and $\hat I_\normp(\bfH_\DM) = 0.1450$.
\end{itemize}

\noindent (v) As for Applications~1 and~2, examination of these figures shows that traditional learning without PLoM gives poor results compared to PLoM, which allows the concentration to be preserved and properly learns the geometry of the probability measure support. We also see that PLoM with the optimal ROTB provides an improvement in learning compared to PLoM with RODM. However, this improvement is less than in the case of Applications~1 and~2. For this application, relative  to a relatively high dimension of  $\nu = 45$, the data are more homogeneous than for the other applications (in correlation with the geometric complexity of the probability-measure support). Nevertheless, PLoM with the optimal
ROTB is an improvement over PLoM with RODB and, consequently, should improve the estimates of conditional statistics thanks to better learning of the joint probability measure.

\begin{figure}[!t]
    \centering
    \begin{subfigure}[b]{0.25\textwidth}
    \centering
        \includegraphics[width=\textwidth]{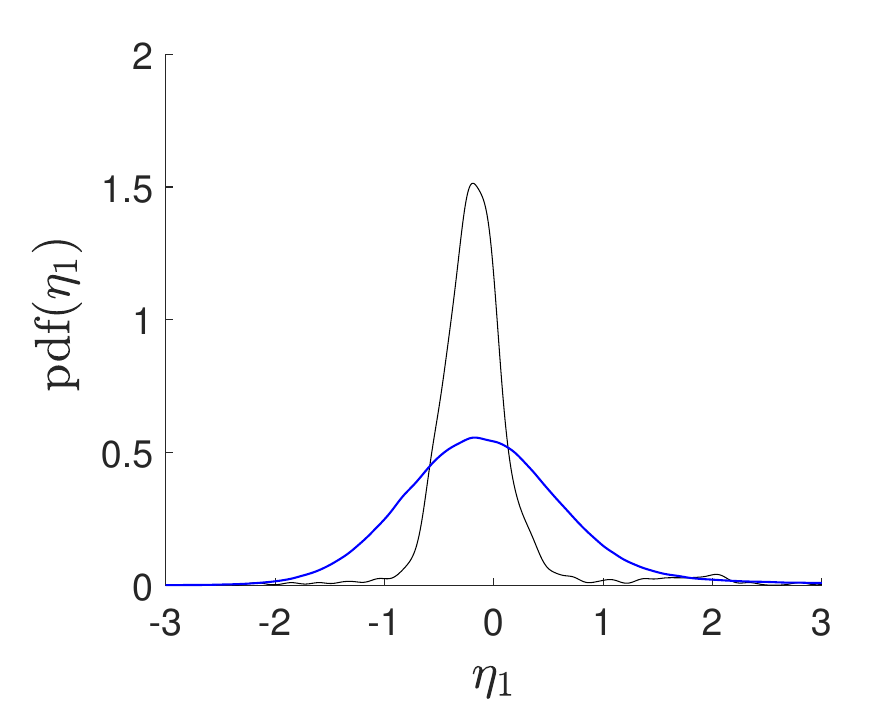}
        \caption{pdf of $H_1$ and $H_{\ar,1}$.}
        \label{fig:figure11a}
    \end{subfigure}
    \hfil
    \begin{subfigure}[b]{0.25\textwidth}
        \centering
        \includegraphics[width=\textwidth]{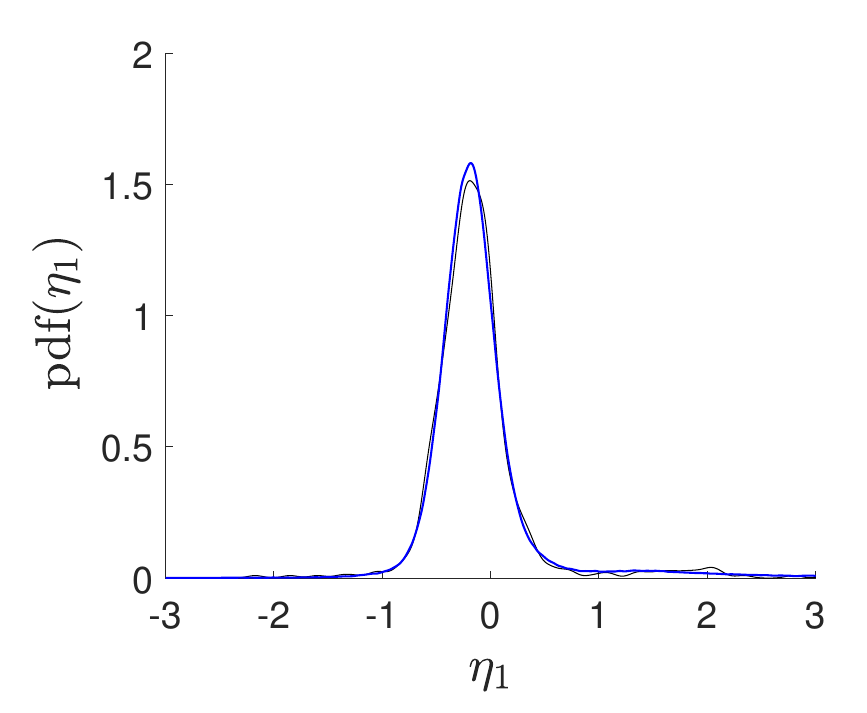}
        \caption{pdf of $H_1$ and $H_{\DB,1}$.}
        \label{fig:figure11b}
    \end{subfigure}
    \hfil
    \begin{subfigure}[b]{0.25\textwidth}
        \centering
        \includegraphics[width=\textwidth]{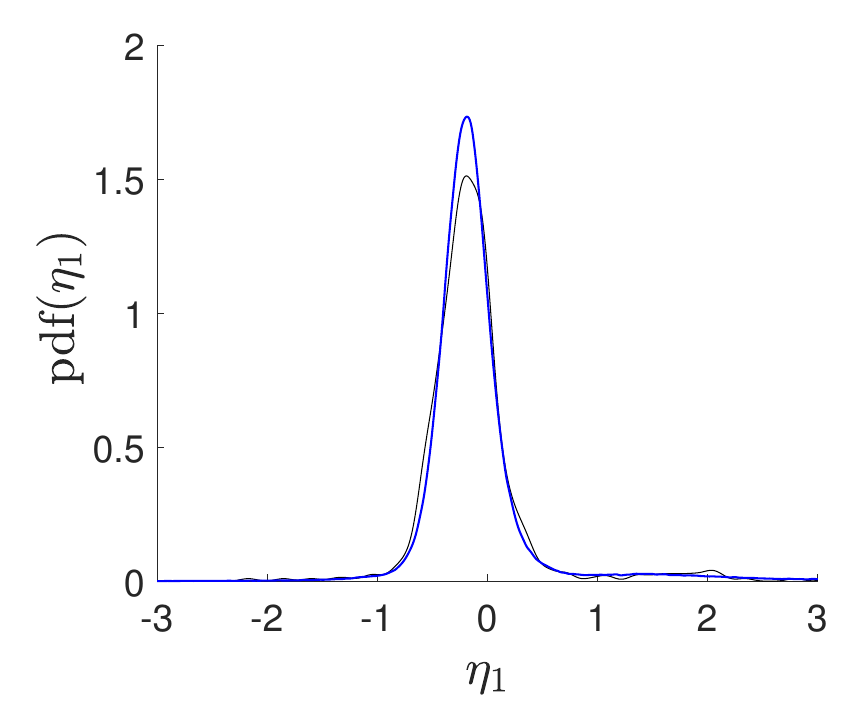}
        \caption{pdf of $H_1$ and $H_{\TB,1}$ at time $n_\optp\,\Delta t$.}
        \label{fig:figure11c}
    \end{subfigure}
    %
    \centering
    \begin{subfigure}[b]{0.25\textwidth}
    \centering
        \includegraphics[width=\textwidth]{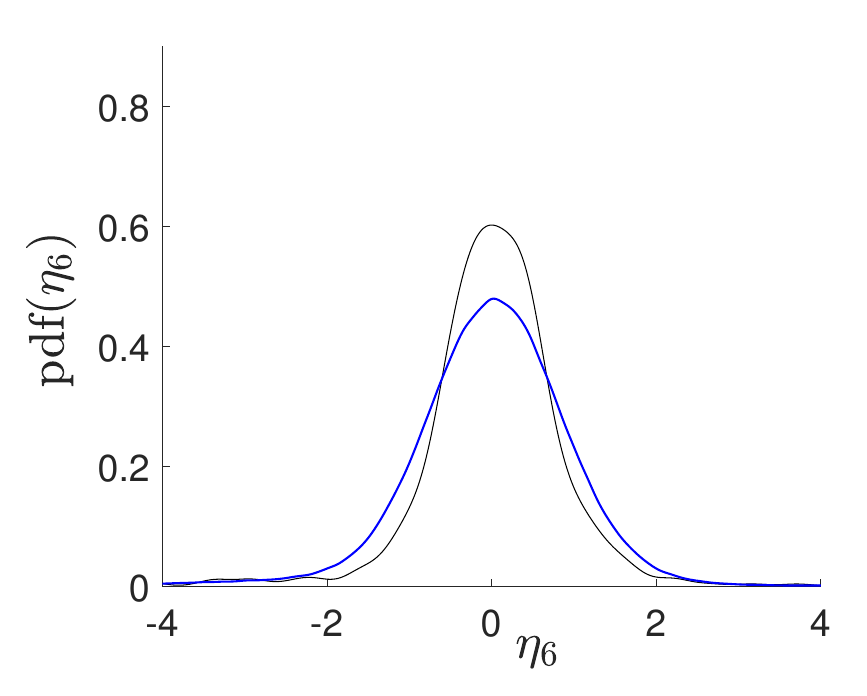}
         \caption{pdf of $H_6$ and $H_{\ar,6}$.}
        \label{fig:figure11d}
    \end{subfigure}
    \hfil
    \begin{subfigure}[b]{0.25\textwidth}
        \centering
        \includegraphics[width=\textwidth]{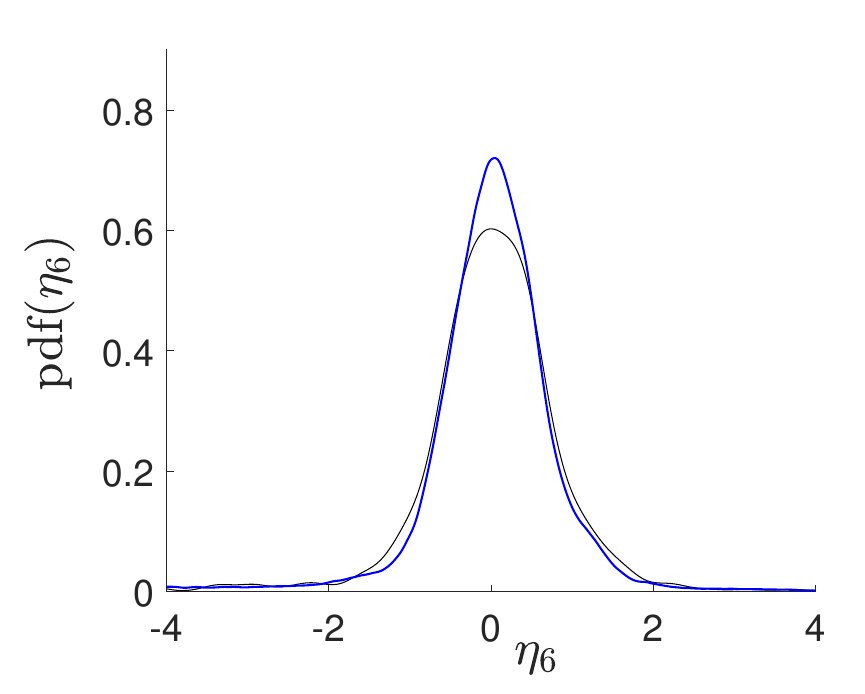}
        \caption{pdf of $H_6$ and $H_{\DB,6}$.}
        \label{fig:figure11e}
    \end{subfigure}
    \hfil
    \begin{subfigure}[b]{0.25\textwidth}
        \centering
        \includegraphics[width=\textwidth]{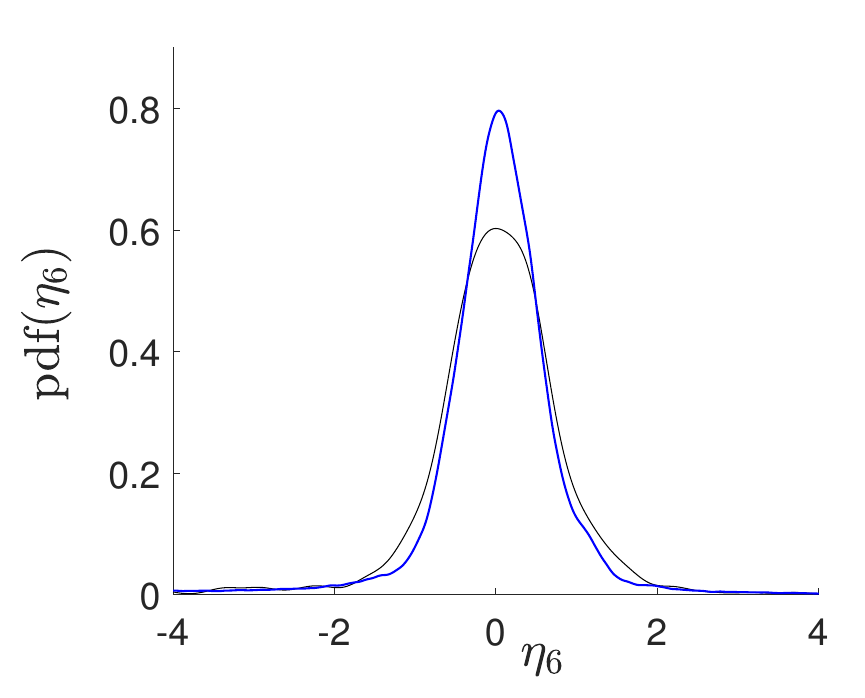}
        \caption{pdf of $H_6$ and $H_{\TB,6}$ at time $n_\optp\,\Delta t$.}
        \label{fig:figure11f}
    \end{subfigure}
    %
    \centering
    \begin{subfigure}[b]{0.25\textwidth}
    \centering
        \includegraphics[width=\textwidth]{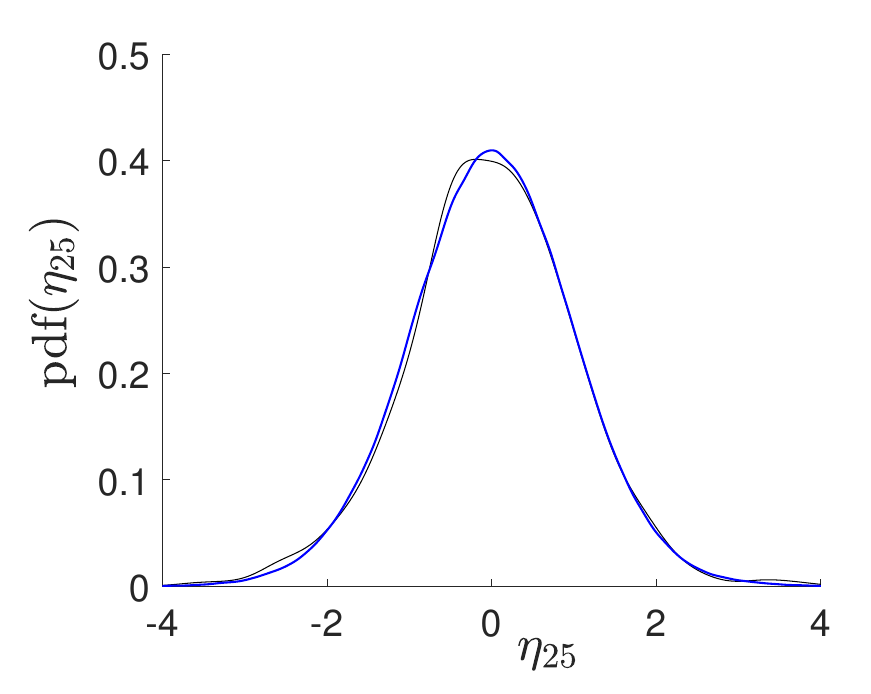}
         \caption{pdf of $H_{25}$ and $H_{\ar,25}$.}
        \label{fig:figure11g}
    \end{subfigure}
    \hfil
    \begin{subfigure}[b]{0.25\textwidth}
        \centering
        \includegraphics[width=\textwidth]{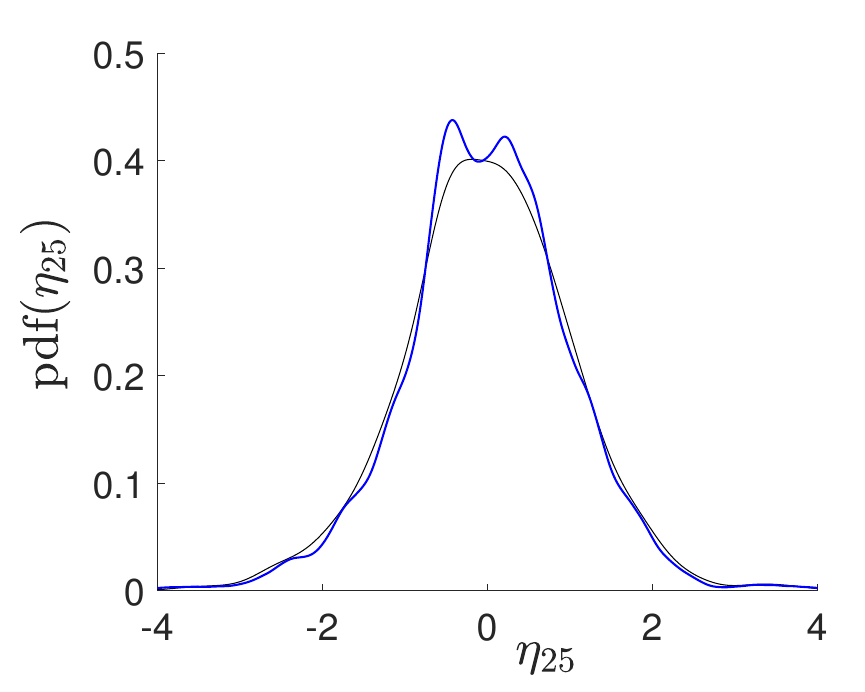}
        \caption{pdf of $H_{25}$ and $H_{\DB,25}$.}
        \label{fig:figure11h}
    \end{subfigure}
    \hfil
    \begin{subfigure}[b]{0.25\textwidth}
        \centering
        \includegraphics[width=\textwidth]{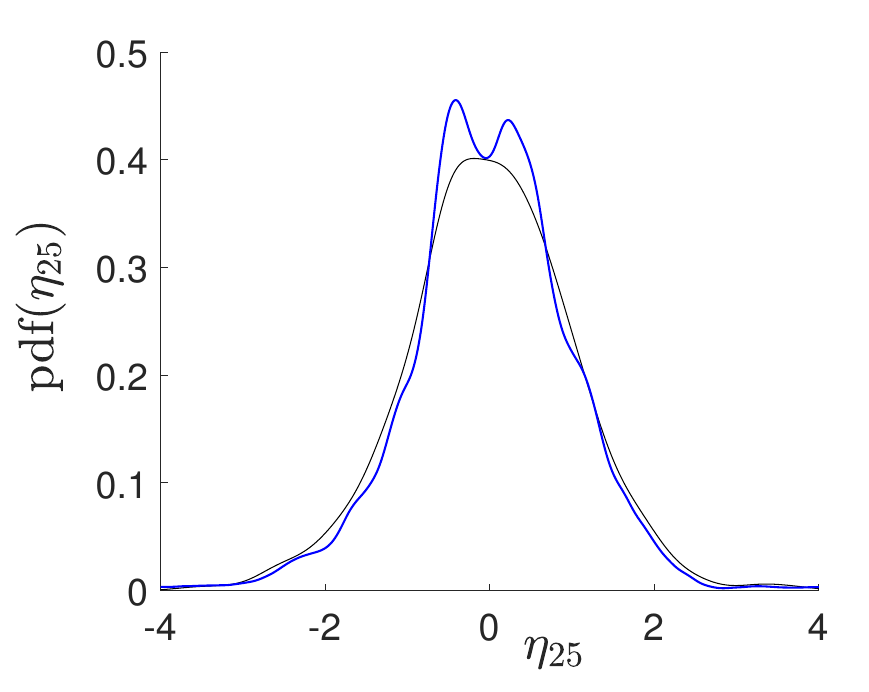}
        \caption{pdf of $H_{25}$ and $H_{\TB,25}$ at time $n_\optp\,\Delta t$.}
        \label{fig:figure11i}
    \end{subfigure}
    %
    \centering
    \begin{subfigure}[b]{0.25\textwidth}
    \centering
        \includegraphics[width=\textwidth]{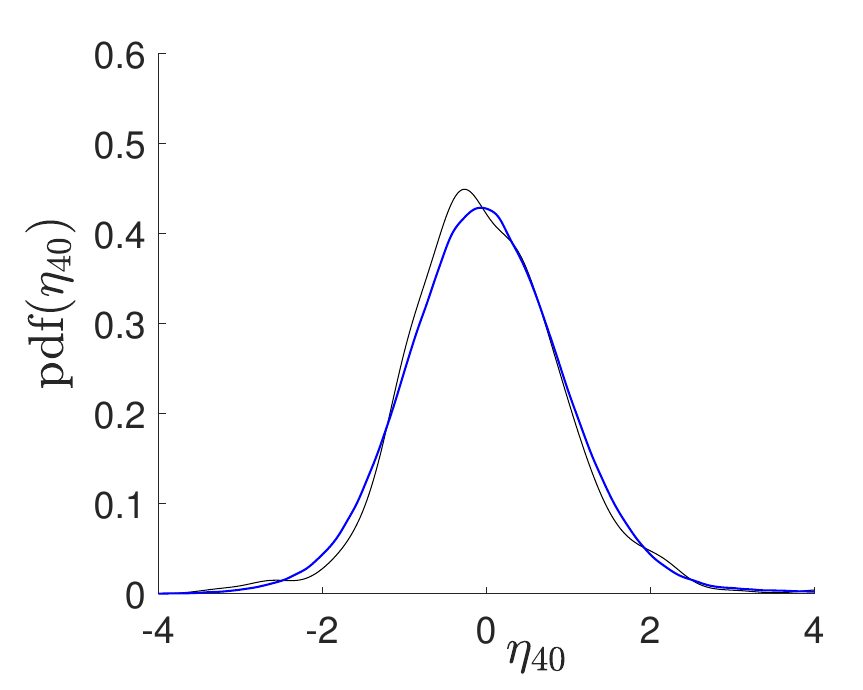}
        \caption{pdf of $H_{40}$ and $H_{\ar,40}$.}
        \label{fig:figure11j}
    \end{subfigure}
    \hfil
    \begin{subfigure}[b]{0.25\textwidth}
        \centering
        \includegraphics[width=\textwidth]{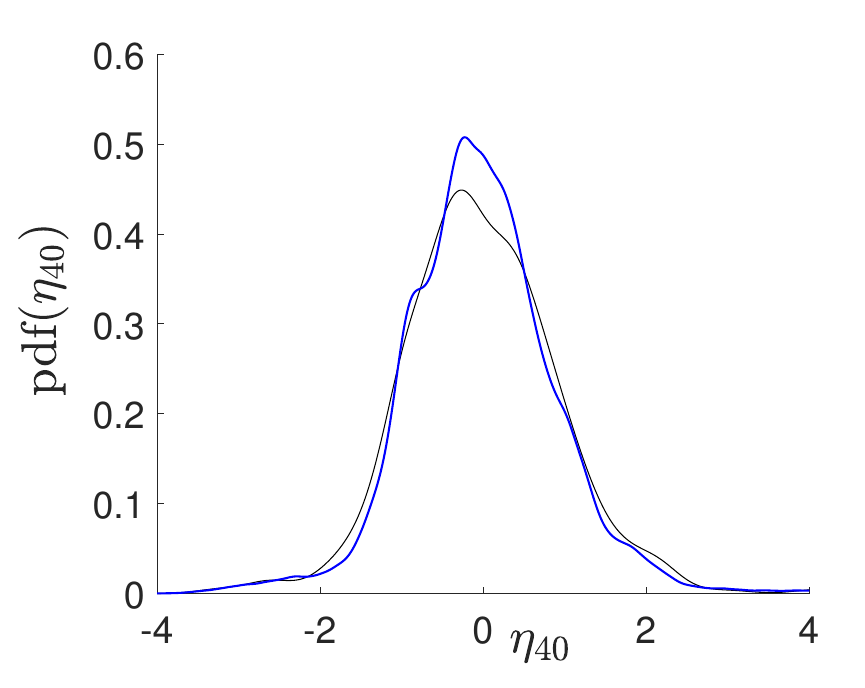}
        \caption{pdf of $H_{40}$ and $H_{\DB,40}$.}
        \label{fig:figure11k}
    \end{subfigure}
    \hfil
    \begin{subfigure}[b]{0.25\textwidth}
        \centering
        \includegraphics[width=\textwidth]{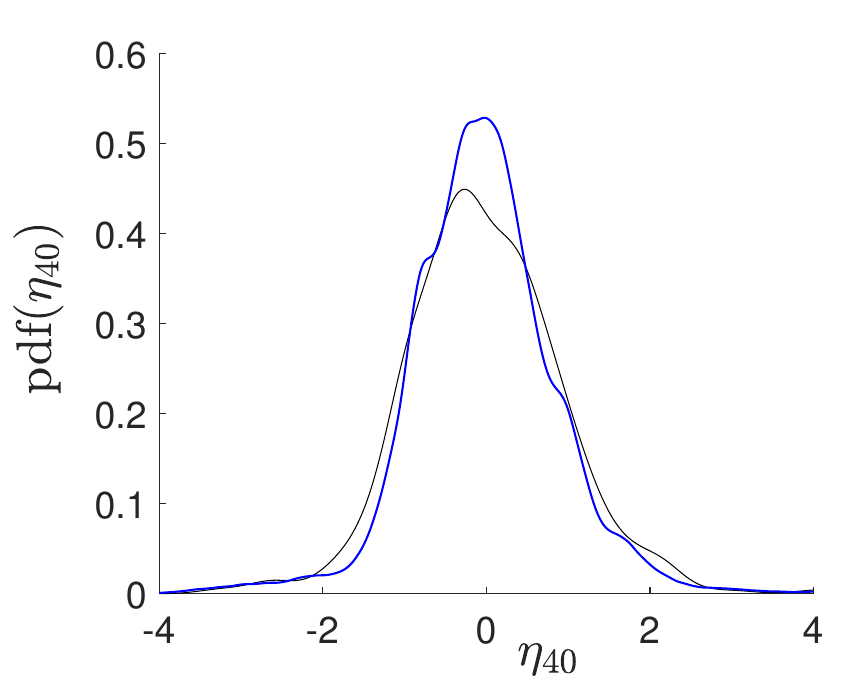}
        \caption{pdf of $H_{40}$ and $H_{\TB,40}$ at time $n_\optp\,\Delta t$.}
        \label{fig:figure11l}
    \end{subfigure}
    \caption{Application 3. Probability density function (pdf)  of components $1$, $6$, $25$, and $40$ for $\bfH$ estimated with the $n_d$ realizations of the training dataset (thin black line) and pdf estimated with $n_\ar$ learned realizations (thick blue line), for $\bfH_\ar$ using MCMC without PLoM (a,d,g,j), for $\bfH_\DB$ using  PLoM with RODB (b,e,h,k), and for $\bfH_\TB$ using PLoM with ROTB$(n_\optp \Delta t)$ (c, f, i, l).}
    \label{fig:figure11}
\end{figure}
%
\begin{figure}[h]
    \centering
    \begin{subfigure}[b]{0.24\textwidth}
    \centering
        \includegraphics[width=\textwidth]{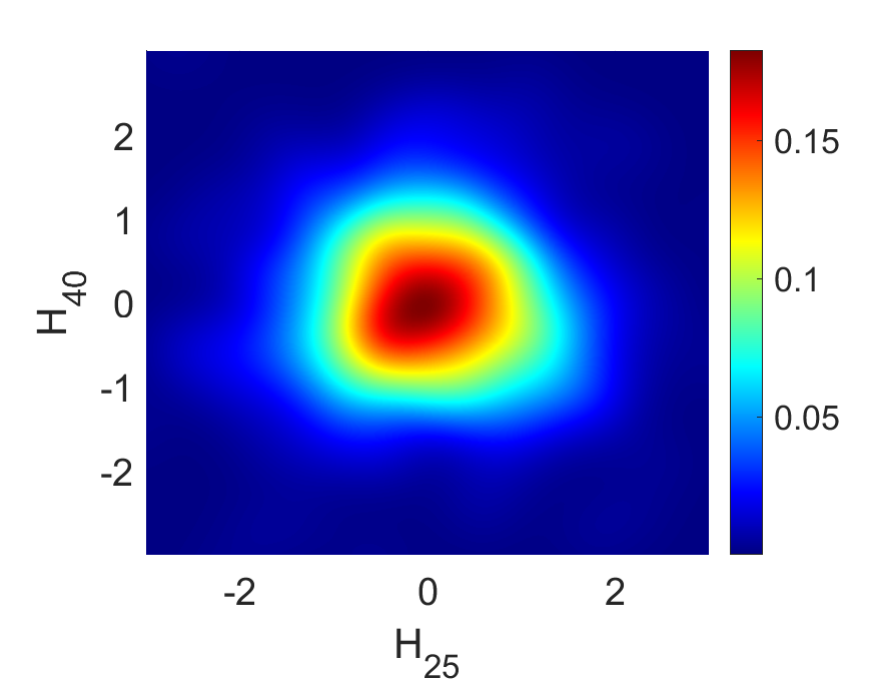}
        \caption{joint pdf of $H_{25}$ with $H_{40}$.}
         \vspace{0.3truecm}
        \label{fig:figure12a}
    \end{subfigure}
    \hfil
    \begin{subfigure}[b]{0.24\textwidth}
        \centering
        \includegraphics[width=\textwidth]{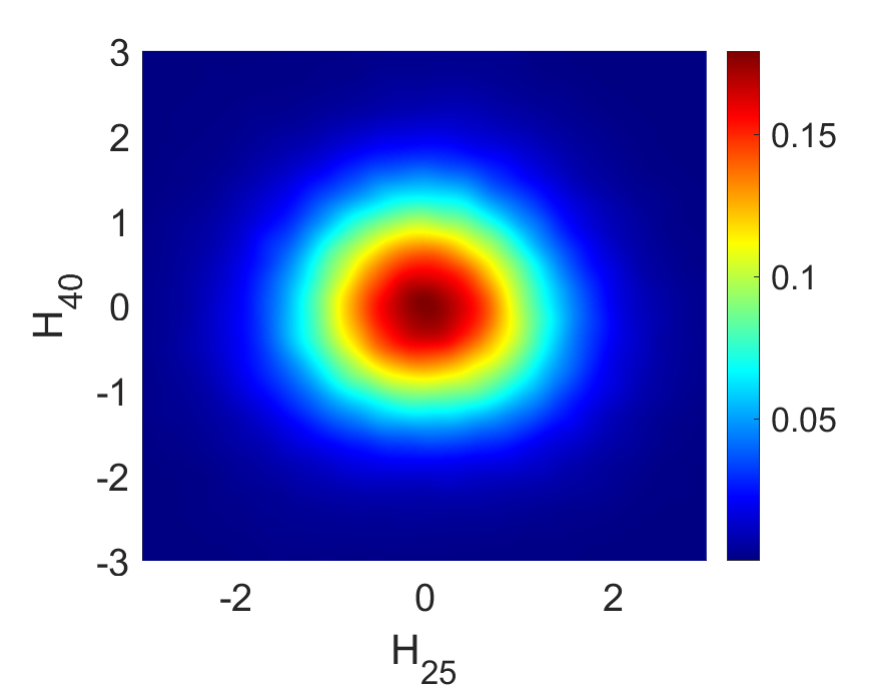}
        \caption{joint pdf of $H_{\ar,25}$ with $H_{\ar,40}$.}
         \vspace{0.3truecm}
        \label{fig:figure12b}
    \end{subfigure}
    \hfil
    \begin{subfigure}[b]{0.24\textwidth}
        \centering
        \includegraphics[width=\textwidth]{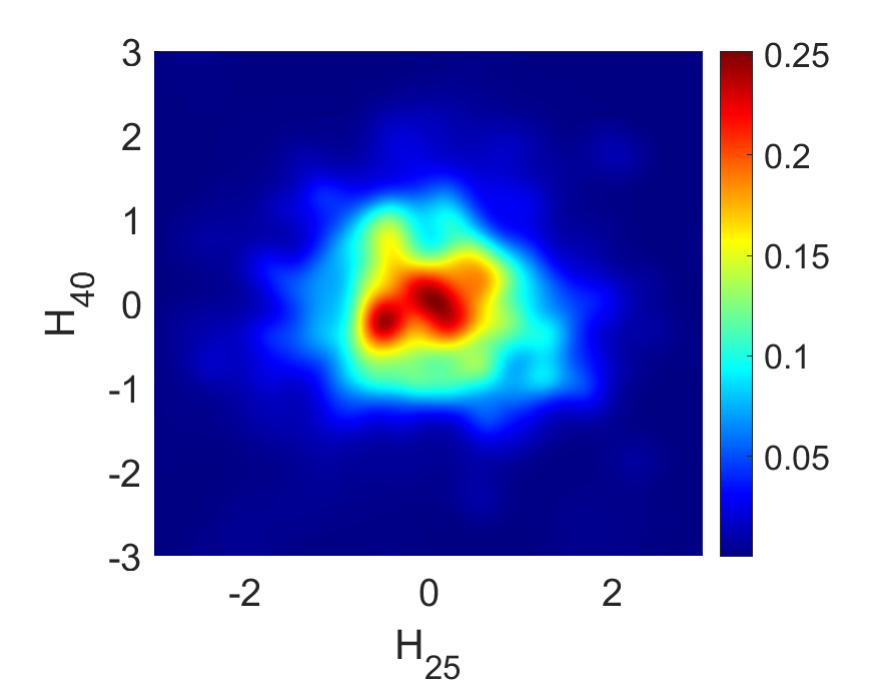}
        \caption{joint pdf of $H_{\DB,25}$ with $H_{\DB,40}$.}
         \vspace{0.3truecm}
        \label{fig:figure12c}
    \end{subfigure}
    \hfil
    \begin{subfigure}[b]{0.24\textwidth}
    \centering
        \includegraphics[width=\textwidth]{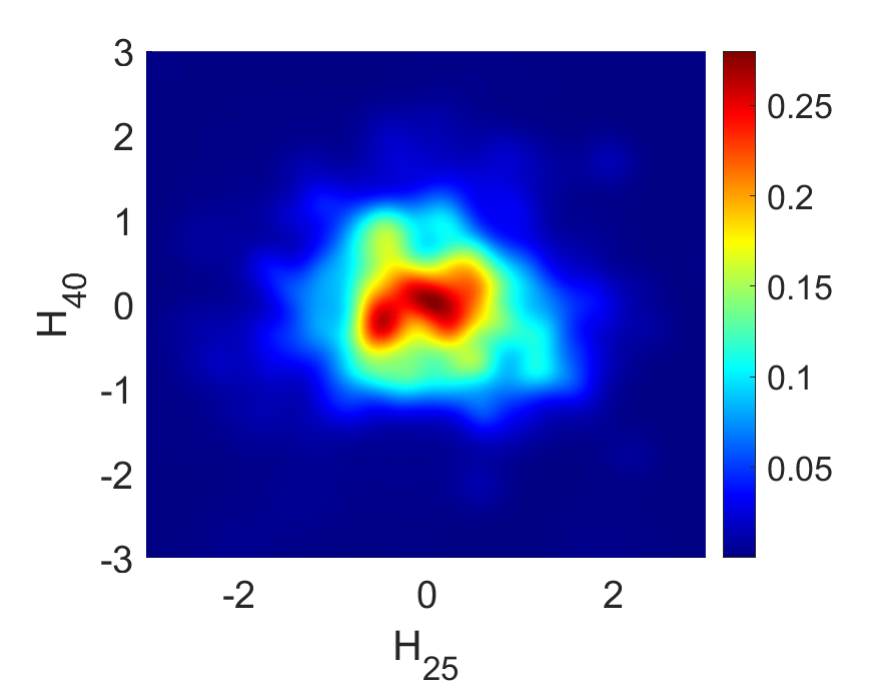}
        \caption{joint pdf of $H_{\TB,25}$ with $H_{\TB,40}$ at time $n_\optp\, \Delta t$.}
        \label{fig:figure12d}
    \end{subfigure}
    \caption{Application 3. Joint probability density function of components $24$ with $40$ of $\bfH$ estimated with the $n_d$ realizations of the training dataset (a) and estimated with $n_\ar$ learned realizations, for $\bfH_\ar$ using MCMC without PLoM (b), for $\bfH_\DB$ using PLoM with RODB (c), and for $\bfH_\TB$ using PLoM with ROTB$(n_\optp \Delta t)$ (d).}
    \label{fig:figure12}
\end{figure}
%
%
%
\begin{figure}[h]
    \centering
    \begin{subfigure}[b]{0.25\textwidth}
    \centering
        \includegraphics[width=\textwidth]{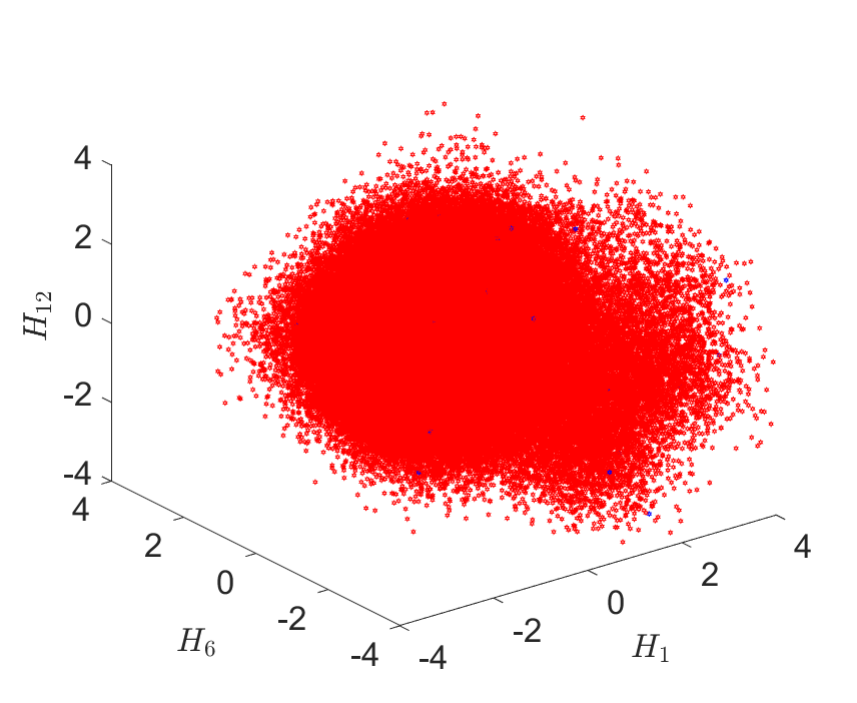}
        \caption{clouds for $(H_{\ar,1},H_{\ar,6},H_{\ar,12})$.}
         \vspace{0.3truecm}
        \label{fig:figure13a}
    \end{subfigure}
    \hfil
    \begin{subfigure}[b]{0.25\textwidth}
        \centering
        \includegraphics[width=\textwidth]{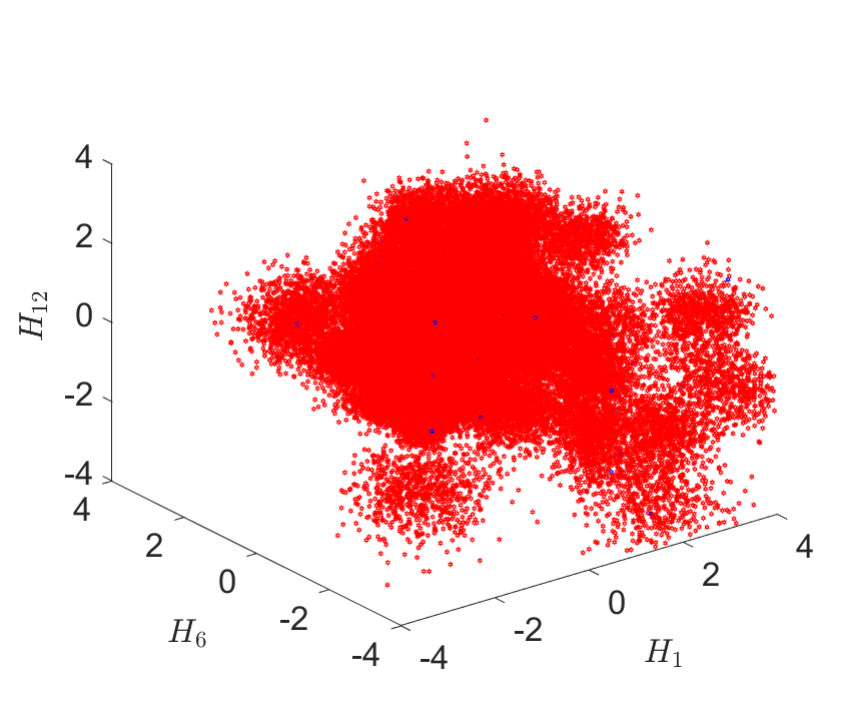}
         \caption{clouds for $(H_{\DB,1},H_{\DB,6},H_{\DB,12})$.}
          \vspace{0.3truecm}
        \label{fig:figure13b}
    \end{subfigure}
    \hfil
    \begin{subfigure}[b]{0.25\textwidth}
        \centering
        \includegraphics[width=\textwidth]{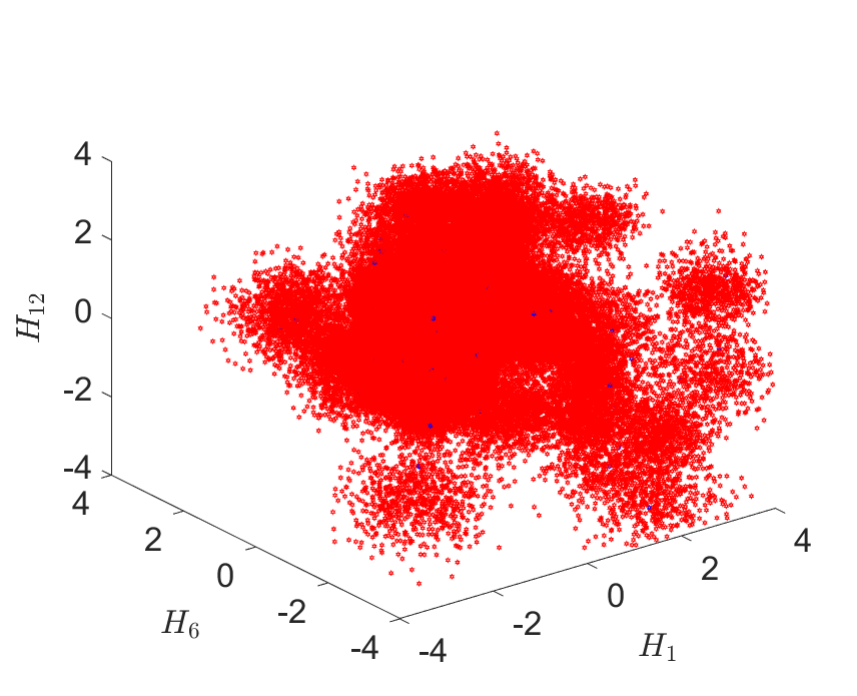}
         \caption{clouds for $(H_{\TB,1},H_{\TB,6},H_{\TB,12})$ at time $n_\optp\,\Delta t$.}
        \label{fig:figure13c}
    \end{subfigure}
    %
    \centering
    \begin{subfigure}[b]{0.25\textwidth}
    \centering
        \includegraphics[width=\textwidth]{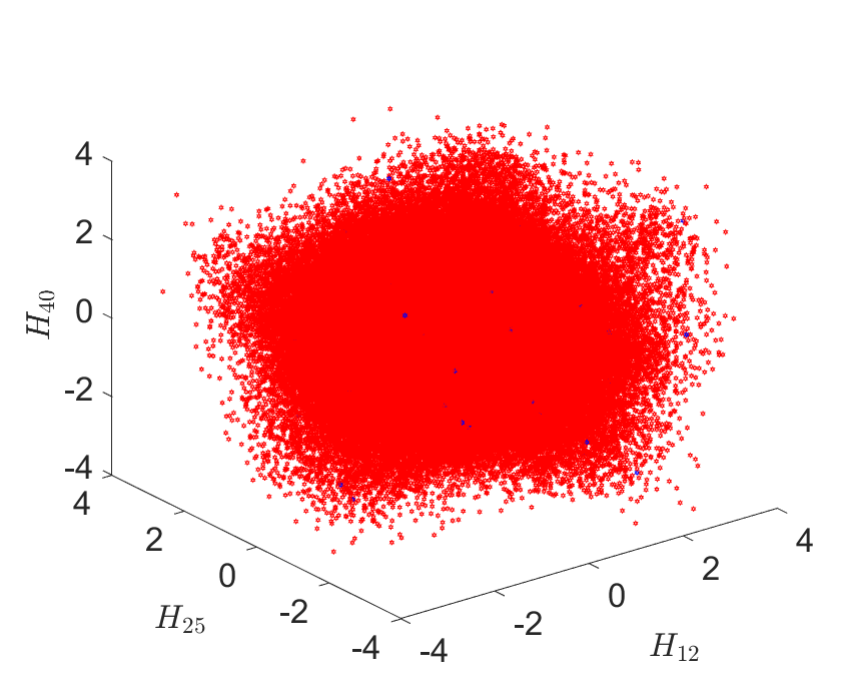}
        \caption{clouds for $(H_{\ar,12},H_{\ar,25},H_{\ar,40})$.}
        \vspace{0.3truecm}
        \label{fig:figure13d}
    \end{subfigure}
    \hfil
    \begin{subfigure}[b]{0.25\textwidth}
        \centering
        \includegraphics[width=\textwidth]{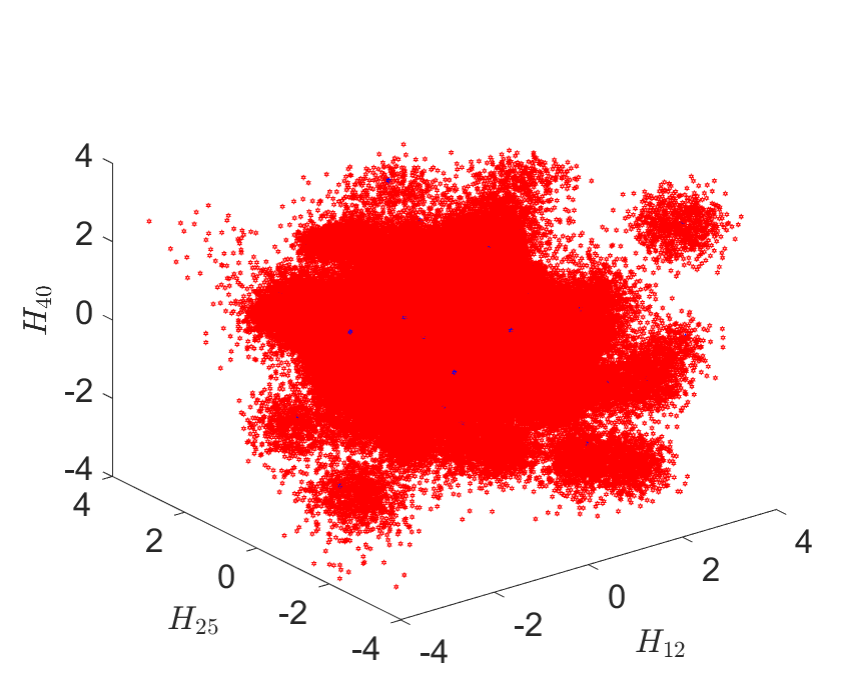}
        \caption{clouds for $(H_{\DB,12},H_{\DB,25},H_{\DB,40})$.}
         \vspace{0.3truecm}
        \label{fig:figure13e}
    \end{subfigure}
    \hfil
    \begin{subfigure}[b]{0.25\textwidth}
        \centering
        \includegraphics[width=\textwidth]{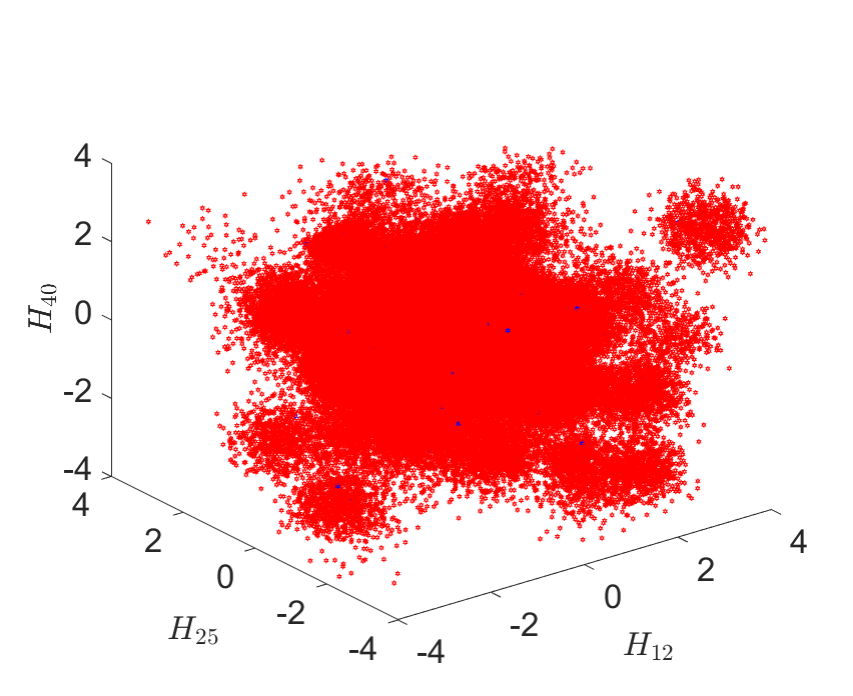}
        \caption{clouds for $(H_{\TB,12},H_{\TB,25},H_{\TB,40})$ at time $n_\optp\,\Delta t$.}
        \label{fig:figure13f}
    \end{subfigure}
    \caption{Application 3. Clouds of $n_\ar$ points corresponding to $n_\ar$ learned realizations, for components $1$, $6$, $12$ (a,b,c) and components $12$, $25$, $40$ (d,e,f), for $\bfH_\ar$ using MCMC without PLoM (a,d), for $\bfH_\DB$ using PLoM with RODB (b,e), and for $\bfH_\TB$ using PLoM with ROTB$(n_\optp \Delta t)$ (c,f).}
    \label{fig:figure13}
\end{figure}
%
%
%
\begin{figure}[h]
    \centering
    \begin{subfigure}[b]{0.25\textwidth}
    \centering
        \includegraphics[width=\textwidth]{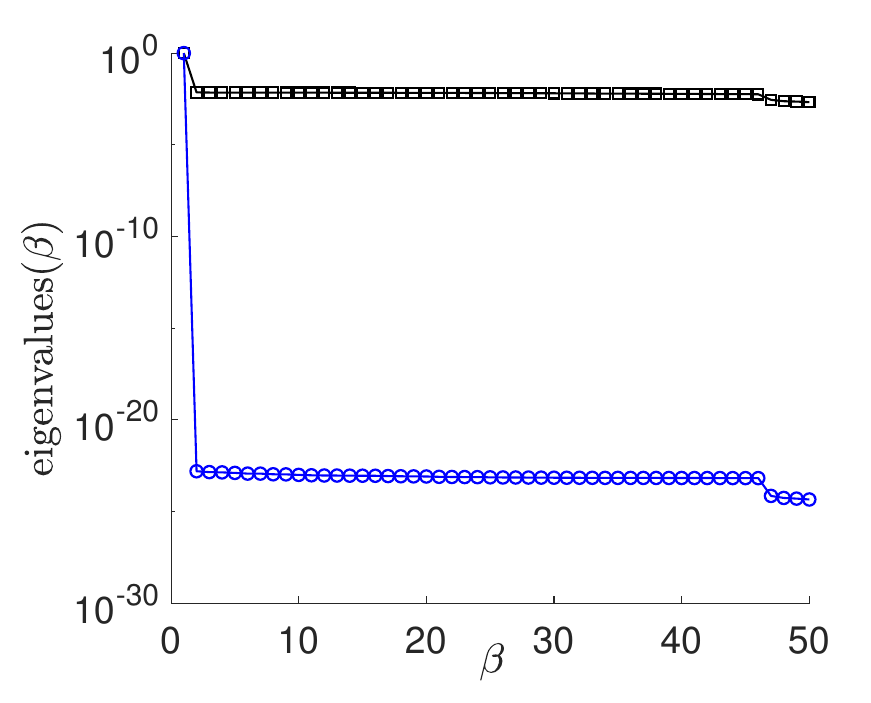}
        \caption{eigenvalues $\beta\mapsto \hat b_{\DM,\,\beta}$ (square) and $\beta\mapsto \tilde b_{\beta}(n_\optp\, \delta t)$ (circle).}
        \label{fig:figure14a}
    \end{subfigure}
    \hfil
    \begin{subfigure}[b]{0.25\textwidth}
        \centering
        \includegraphics[width=\textwidth]{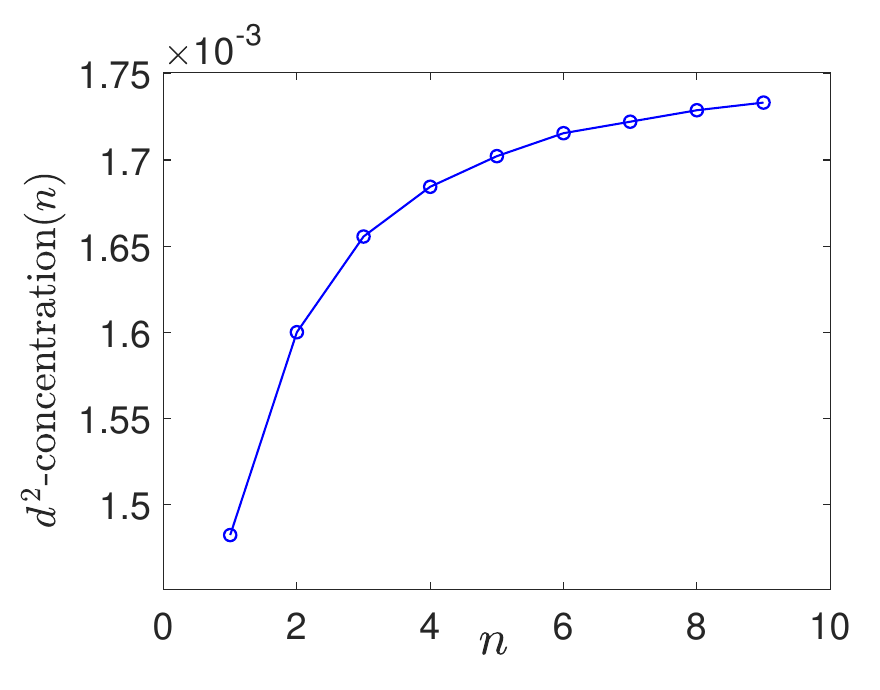}
         \caption{$n\mapsto \hat d^{\,2}(m_\optp;\nDeltat)/\nu$.}
         \vspace{0.3truecm}
        \label{fig:figure14b}
    \end{subfigure}
    \hfil
    \begin{subfigure}[b]{0.25\textwidth}
        \centering
        \includegraphics[width=\textwidth]{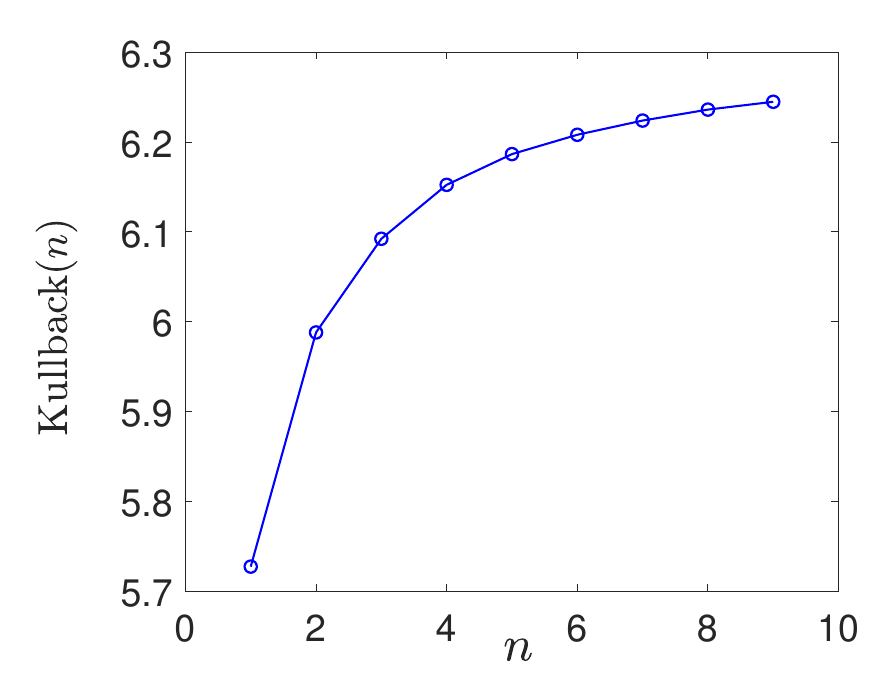}
        \caption{$n\mapsto \hat D(p_\TB(\cdot\, ,\nDeltat)\, \Vert \, p_\bfH)$ (Kullback).}
         \vspace{0.3truecm}
        \label{fig:figure14c}
    \end{subfigure}
    \centering
    \begin{subfigure}[b]{0.25\textwidth}
    \centering
        \includegraphics[width=\textwidth]{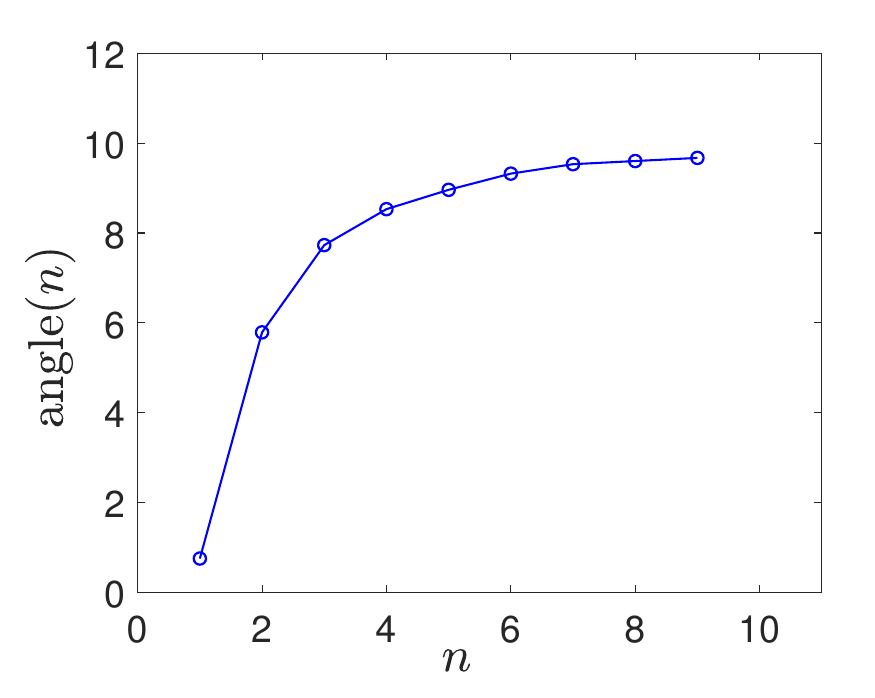}
          \caption{$n\mapsto \gamma(\nDeltat)$ (angle in degree).}
           \vspace{0.3truecm}
        \label{fig:figure14d}
    \end{subfigure}
    \hfil
    \begin{subfigure}[b]{0.25\textwidth}
        \centering
        \includegraphics[width=\textwidth]{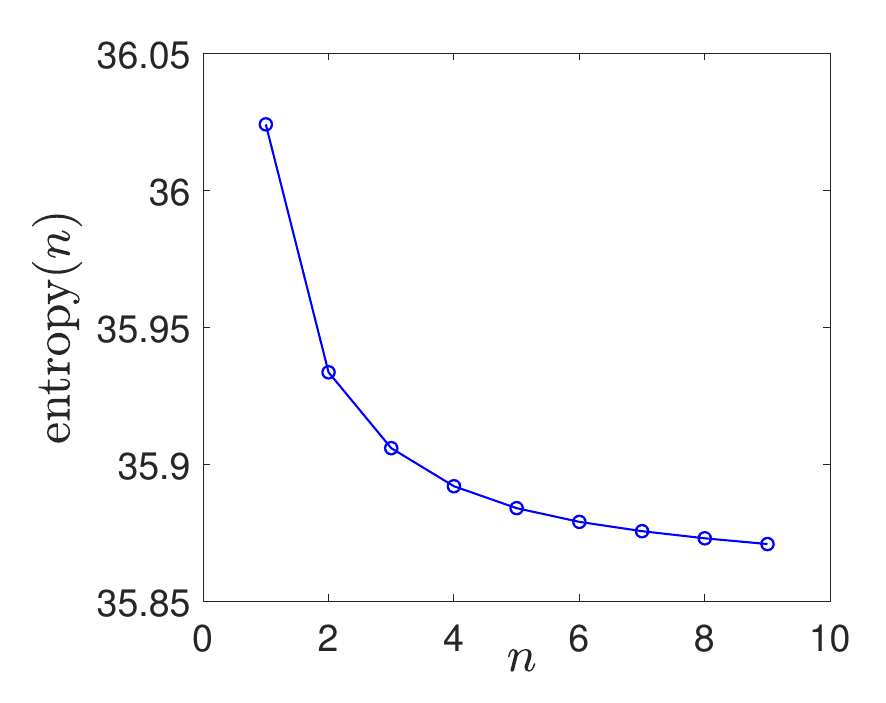}
        \caption{$n\mapsto \hat S_\TB(\nDeltat)$ (entropy).}
         \vspace{0.3truecm}
        \label{fig:figure14e}
    \end{subfigure}
    \hfil
    \begin{subfigure}[b]{0.25\textwidth}
        \centering
        \includegraphics[width=\textwidth]{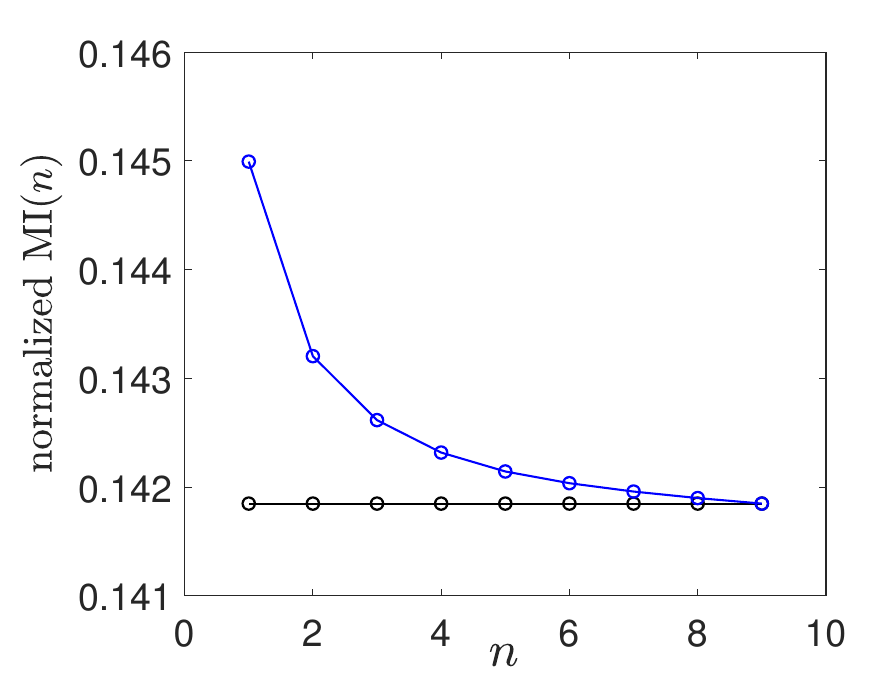}
        \caption{$n\mapsto \hat I_\normp(\bfH)$ (square) and $n\mapsto \hat I_\normp(\bfH_\TB;\nDeltat)$ (circle)(normalized MI).}
        \label{fig:figure14f}
    \end{subfigure}
    \caption{Application 3. Functions characterizing the reduced-order transient basis ROTB$(\nDeltat)$ as a function of time
    $\nDeltat$: eigenvalues of $[K_\DM]$ and of symmetrized $[\tilde K(\nDeltat)]$ (a); measure concentration with $d^{\,2}/\nu$-criterion (b) and with Kullback criteria (c); angle between the subspaces spanned by RODB and ROTB$(\nDeltat)$ (d); entropy of pdf $p_\TB(\cdot\, ; \nDeltat)$ (e); normalized mutual information (MI) of pdf $p_\bfH$ and $p_\TB(\cdot\, ; \nDeltat)$ (f).}
    \label{fig:figure14}
\end{figure}
%

\section{Conclusion}
%
In this paper, we have presented the theoretical elements of constructing a time-dependent anisotropic kernel, which allows us to create a data projection basis for PLoM. This basis serves as an alternative to the DMAPS basis built using a time-independent isotropic kernel used by PLoM. We have demonstrated that an optimal time can be determined to obtain an optimal transient basis, best respecting the statistical dependence between the components for the learned joint probability measure.

The proposed theory has been  developed to improve PLoM in cases of highly heterogeneous data. The improvement of the learned joint probability measure is quantified by estimating an objective criterion from information theory, namely the mutual information, which we have normalized relative to the number of realizations using entropy.

This theory is consistent in the sense that, for a time close to the initial time, the DMAPS basis constructed with the time-independent isotropic kernel coincides with the transient basis constructed with the time-dependent anisotropic kernel. Thus, we can characterize the difference between the two bases by the angle of the vector subspaces they generate.

The theory is illustrated through three applications with decreasing levels of data heterogeneity. The three applications confirm that PLoM with the DMAPS basis (time-independent isotropic kernel) always results in learning that preserves the concentration of the measure, unlike the classic MCMC approach. The applications show that it is possible to improve the learned joint probability measure with the transient anisotropic kernel, which a priori allows for better estimates of conditional statistics.

\section*{Acknowledgments}
The authors acknowledge partial funding from DOE SciDAC FASTMath  Institute, and an ONR MURI on Modeling Turbulence and Chemistry in High Speed Reactive Flows.

\appendix
%
%
\section{Overview of the probabilistic learning on manifolds (PLoM) algorithm and its parameterization}
\label{SectionA}

The PLoM approach~\cite{Soize2016,Soize2020c,Soize2022a}, which has specifically been developed for small
data (as opposed to big data) starts from a training dataset $\curD_d$ made up of a relatively small number
$n_d$ of points. It is assumed that $\curD_d$ is generated with an underlying stochastic manifold related
to a $\RR^{n_x}$-valued random variable $\bfX = (\bfQ,\bfW)$, defined on a probability space $(\Theta,\curT,\curP)$,
in which $\bfQ$ is the quantity of interest that is a $\RR^{n_q}$-random variable, where $\bfW$ is the
control parameter that is a $\RR^{n_w}$-random variable, and where $n_x = n_q + n_w$. Another
$\RR^{n_u}$- valued random variable $\bfU$ defined on $(\Theta,\curT,\curP)$ is also be considered,
which is an uncontrolled parameter and/or a noise.  Random variable $\bfQ$ is assumed to be
written as $\bfQ = \bff(\bfU,\bfW)$ in which the measurable mapping $\bff$ is not explicitly known.
The joint probability distribution $P_{\bfW,\bfU}(d\bfw,d\bfu)$ of  $\bfW$ and $\bfU$ is assumed to be given.
The non-Gaussian probability measure $P_\bfX(\bfx) = P_{\bfQ,\bfW}(d\bfq,d\bfw)$ of $\bfX =(\bfQ,\bfW)$
is concentrated in a region of $\RR^{n_x}$ for which the only available information is the cloud of
the points of training dataset $\curD_d$. The PLoM method makes it possible to generate the learned dataset
$\curD_\ar$ for $\bfX$ whose $n_\ar \gg n_d$ points (learned realizations) are generated by the
non-Gaussian probability measure that is estimated using the training dataset. The concentration of
the probability measure is preserved thanks to the use of a diffusion-maps basis that allows to
enrich the available information from the training dataset.
The training dataset $\curD_d$ is made up of the $n_d$ independent realizations
$\bfx_d^j= (\bfq_d^j,\bfw_d^j)$ in $\RR^{n_x} = \RR^{n_q}\times \RR^{n_w}$ for $j\in\{1,\ldots , n_d\}$
of random variable $\bfX=(\bfQ,\bfW)$. The PLoM method allows for generating the learned dataset
$\curD_\ar$ made up of $n_\ar\gg n_d$  learned realizations $\{\bfx_\ar^{\ell},\ell = 1, \ldots ,n_\ar\}$
of random vector $\bfX$. As soon as the learned dataset has been constructed, the learned realizations
for $\bfQ$ and $\bfW$ can be extracted as $(\bfq_\ar^\ell,\bfw_\ar^\ell) = \bfx_\ar^\ell$ for $\ell = 1, \ldots ,n_\ar$.
Using the learned dataset $\curD_\ar$, PLoM allows for carrying out any conditional statistics such as $\bfw\mapsto E\{\bfxi(\bfQ)\vert\bfW=\bfw\}$ from $\curC_w$ in $\RR^{n_\xi}$, in which $\bfxi$ is a given measurable mapping from $\RR^{n_q}$ into $\RR^{n_\xi}$, that is to say to construct statistical surrogate models (metamodels) in a probabilistic framework.
\subsection{Reduced representation}
\label{SectionA.1}
The $n_d$ independent realizations $\{\bfx_d^j , j=1,\ldots , n_d\}$ are represented by the matrix $[x_d]= [\bfx_d^1 \ldots \bfx_d^{n_d}]$ in $\MM_{n_x,n_d}$. Let $[\bfX]= [\bfX^1,\ldots ,\bfX^{n_d}]$ be the random
matrix with values in $\MM_{n_x,n_d}$, whose columns are $n_d$ independent copies of random vector $\bfX$. Using the PCA of $\bfX$, random matrix $[\bfX]$ is written as,
\begin{equation}\label{eqA1}
[\bfX] = [\underline x] + [\varphi]\, [\zeta]^{1/2}\, [\bfH]\, ,
\end{equation}
in which $[\bfH] = [\bfH^1,\ldots ,$ $\bfH^{n_d}]$ is a $\MM_{\nu,n_d}$-valued random matrix, where $\nu\leq n_x$, and where $[\zeta]$ is the $(\nu\times\nu)$ diagonal matrix of the $\nu$ positive eigenvalues of the empirical estimate of the covariance matrix of $\bfX$. The $(n_x\times\nu)$ matrix $[\varphi]$ is made up of the associated eigenvectors such $[\varphi]^T\,[\varphi]= [I_{\nu}]$. The matrix $[\underline x]$ in $\MM_{n_x,n_d}$ has identical columns, each one being equal to the empirical estimate $\underline\bfx\in\RR^{n_x}$ of the mean value of random vector $\bfX$.
The columns of $[\bfH]$ are $n_d$ independent copies of a random vector $\bfH$ with values in $\RR^{\nu}$.
The realization $[\eta_d] = [\bfeta^{1} \ldots \bfeta^{n_d}] \in \MM_{\nu,n_d}$ of $[\bfH]$  is computed by $[\eta_d] =  [\zeta]^{-1/2} [\varphi]^T\, ([x_d] - [\underline x])$.
The value $\nu$ is classically calculated in order that the $L^2$- error function $\nu\mapsto \err_\bfX(\nu)$ defined by
\begin{equation}\label{eqA2}
\err_\bfX(\nu) = 1 - \frac{\sum_{\alpha=1}^\nu \zeta_\alpha}{E\{\Vert\bfX\Vert^2\}} \, ,
\end{equation}
be smaller than $\varepsilon_\PCA$. If $\nu < n_x$, then there is a statistical reduction.

\subsection{Construction of a reduced-order diffusion-maps basis (RODB) and reduced-order transient basis (ROTB$(\nDeltat)$)}
\label{SectionA.2}
 In this section, we begin with the construction of the RODB that is the basis initially used in the PLoM algorithm (see \cite{Soize2016}).
Concerning the construction of the ROTB$(\nDeltat)$, we refer the read to \ref{Section7.2}-(ii).\\

\noindent \textit{(i) Construction of RODB}.
This construction corresponds to the one initailly proposed in the PLoM algorithm.
For preserving the concentration of the learned realizations in the region in which the points of the training dataset are concentrated, the PLoM relies on the diffusion-maps method \cite{Coifman2006,Lafon2006}. This is an algebraic basis of vector space $\RR^{n_d}$, which  is constructed using the diffusion maps.
Let $[\curK_\DM]$  and  $[B]$ be the matrices such that, for all $i$ and $j$ in $\{1,\ldots , n_d\}$,
$[\curK_\DM]_{ij} = \exp\{-(4\,\varepsilon_\DM)^{-1} \Vert\bfeta^i-\bfeta^j\Vert^2\}$ and
$[B]_{ij} = \delta_{ij} \, \exp\{-(4\,\varepsilon_\DM)^{-1} \Vert\bfeta^i-\bfeta^j\Vert^2\}$, in which $\varepsilon_\DM >0$ is a smoothing parameter. The eigenvalues $b_{\DM,1},\ldots,b_{\DM,n_d}$ and the associated eigenvectors $\bfg_\DM^1,\ldots,\bfg_\DM^{n_d}$ of the right-eigenvalue problem $[B]^{-1}[\curK_\DM]\, \bfg_\DM^\beta = b_{\DM,\beta}\, \bfg_\DM^\beta$ are such that $ 1=b_{\DM,1} > b_{\DM,2} \geq \ldots \geq b_{\DM,n_d}$ and are computed by solving the eigenvalue problem $[B]^{-1/2}[\curK_\DM]\,[B]^{-1/2} \, \bfphi^\beta = b_{\DM,\beta}\, \bfphi^\beta$ with the normalization $\langle \bfphi^\beta,\bfphi^{\beta'}\rangle =\delta_{\beta\beta'}$, and $\bfg_\DM^\beta = [B]^{-1/2}\bfphi^\beta $. The eigenvector $\bfg_\DM^1$ associated with $b_{\DM,1}=1$ is a constant vector.
The diffusion-maps basis $\{\bfg_\DM^1,\ldots,\bfg_\DM^\alpha,\ldots, \bfg_\DM^{n_d}\}$ is a vector basis of $\RR^{n_d}$.
For a given integer $m < n_d$, the reduced-order diffusion-maps basis of order $m$ is defined as the family
$\{\bfg_\DM^1,\ldots,\bfg_\DM^{m}\}$.
This basis depends on two parameters, $\varepsilon_\DM$ and $m$, which have to be identified.
As explained in \cite{Soize2022a}, the optimal value $m_\optp$ of $m_\DM$ is chosen as $m_\optp = \nu +1$, and the optimal value $\varepsilon_\optp$ of $\varepsilon_\DM$ is such that
\begin{equation} \label{eqA3}
1 = b_{\DM,1} > b_{\DM,2} \simeq \ldots \simeq b_{\DM,m_\optpp}
 \gg b_{\DM,m_\optpp + 1}\geq \ldots \geq b_{\DM,n_d} > 0\, ,
\end{equation}
with the jump amplitude $J_\DM = b_{\DM,m_\optpp + 1}  /b_{\DM,m_\optpp}$, which is $J_\DM=0.1$ (following \cite{Soize2020c}), but which can also be chosen in the interval $[0.1\, , 0.5]$.
Consequently, the RODB is defined for $m=m_\optp$ and is represented by the matrix
\begin{equation} \label{eqA4}
[g_\DM] = [\bfg_\DM^1 \ldots \bfg_\DM^{m_\optpp}]\in\MM_{n_d,m_\optpp}\, .
\end{equation}
%

\noindent \textit{(ii) Construction of ROTB$(\nDeltat)$}.
Because PLoM will also use the reduced-order transient basis to quantify its efficiency relative to the reduced-order DMAPS basis, we introduce this basis  in this Appendix.
For $n$ fixed in $\{1,\ldots , N\}$, the reduced-order transient basis, ROTB$(\nDeltat)$, is represented by the matrix
\begin{equation} \label{eqA5}
[g(\nDeltat)] = [\bfg^1(\nDeltat) \ldots \bfg^{m_\optpp}(\nDeltat)]\in\MM_{n_d,m_\optpp}\, ,
\end{equation}
in which $m_\optp$ is the optimal value identified in \ref{SectionA.2}-(i), and where $[g(\nDeltat)]$ is constructed in Section~\ref{Section7.2}-(ii) (see Eq.~\eqref{eq7.16}.\\

\noindent \textit{(iii) Reduced-order basis for PLoM}.
In this Appendix, the PLoM reduced-order basis will be represented by the matrix $[g_\optp]\in\MM_{n_d,m_\optpp}$. Depending on the context of its use, this matrix will either be $[g_\DM]$, representing the reduced-order DMAPS basis (RODB) as used in the initial construction of the PLoM, or $[g(\nDeltat)]$ for a fixed $n$, representing the reduced-order transient basis (ROTB$(\nDeltat)$) as proposed in this paper. The latter is introduced with the goal of comparing the efficiency of the two reduced-order vector bases.
\subsection{Reduced-order representation of the random matrices}
\label{SectionA.3}
The reduced-order basis represented by matrix $[g_\optp]\in \MM_{n_d,m_\optpp}$  spans a subspace of $\RR^{n_d}$ that characterizes, for the optimal values $m_\optp$ and $\varepsilon_\optp$, the local geometry structure of  dataset $\{\bfeta^j, j=1,\ldots , n_d\}$. So the PLoM method introduces the $\MM_{\nu,n_d}$-valued random matrix $[\bfH_\ar] = [\bfZ]\, [g_\optp]^T$ with $m_\optp < n_d$, corresponding to a data-reduction representation of random matrix $[\bfH]$, in which  $[\bfZ]$ is a $\MM_{\nu,m_\optpp}$-valued random matrix. The MCMC generator of random matrix $[\bfZ]$ belongs to the class of Hamiltonian Monte Carlo methods, is explicitly described in \cite{Soize2016}, and is mathematically detailed in Theorem~6.3 of \cite{Soize2020c}. This  generator allows for computing $n_\MCH$ realizations $\{[\bfz_\ar^{\ell}],\ell=1,\ldots , n_\MCH\}$ of  $[\bfZ]$ and therefore, for deducing the $n_\MCH$ realizations $\{[\bfeta_\ar^\ell],\ell=1,\ldots ,n_\MCH\}$ of $[\bfH_\ar]$. The reshaping of matrix $[\bfeta_\ar^\ell] \in \MM_{\nu,n_d}$ allows for obtaining $n_\ar = n_\MCH \times n_d$ learned realizations
$\{\bfeta_\ar^{\ell'},\ell' =1,\ldots ,n_\ar\}$ of $\bfH_\ar$. These learned realizations allow for estimating converged statistics on $\bfH_\ar$ and then on $\bfX_\ar = \underline x + [\varphi]\, [\zeta]^{1/2}\, \bfH_\ar$, such as pdf, moments, or conditional expectation of the type $E\{\bfxi(\bfQ) \,\vert \,\bfW = \bfw\}$ for $\bfw$ given in $\RR^{n_w}$ and for any given vector-valued function $\bfxi$ defined on $\RR^{n_q}$.

\subsection{Criterion for quantifying the concentration of the probability measure of random matrix $[\bfH_\ar]$}
\label{SectionA.4}
The concentration of the probability measure of random matrix $[\bfH_\ar]$ is defined (see \cite{Soize2020c}) by
\begin{equation} \label{eqA6}
d^{\,2}(m_\optp) = E\{\Vert [\bfH_\ar] - [\eta_d]\Vert^2\} /\Vert [\eta_d]\Vert^2\, .
\end{equation}
Let $\curM =\{m_\optp,m_\optp+1,\ldots , n_d\}$ in which $m_\optp$ is the optimal value of $m$.
Theorem~7.8 of \cite{Soize2020c} shows that
$\min_{m\in\curM} d^{\,2}(m)  \leq 1 + m_\optp/(n_d-1) < d^2(n_d)$, which means that the PLoM method, for $m=m_\optp$ and $[g_\optp]$ is a better method than the usual one corresponding to $d^2(n_d)= 1+n_d/(n_d-1)\simeq 2$.
Using the $n_\MCH$ realizations $\{[\bfeta_\ar^\ell],\ell=1,\ldots ,n_\MCH\}$ of $[\bfH_\ar]$, we have the estimate,
\begin{equation} \label{eqA7}
\hat d^{\,2}(m_\optp) = (1/n_\MCH)\sum_{\ell=1}^{n_\ppMCH}\{\Vert [\bfeta_\ar^\ell]  - [\eta_d]\Vert^2\} /\Vert [\eta_d]\Vert^2 \, .
\end{equation}
\subsection{Generation of learned realizations $\{\bfeta^{\ell'}_\ar, \ell'=1,\ldots,$ $n_\ar\}$ of random vector $\bfH_\ar$}
\label{SectionA.5}
Let  $\{ ([\bfcurZ (t)],$ $[\bfcurY(t)]),$ $t\in \RR^+ \}$ be  the unique asymptotic (for $t\rightarrow +\infty$) stationary diffusion stochastic process with values in $\MM_{\nu,m_\optpp}\times\MM_{\nu,m_\optpp}$, of the following reduced-order ISDE (stochastic nonlinear second-order dissipative Hamiltonian dynamic system), for $t >0$,
\begin{align}
 d[{\bfcurZ}(t)]  & =  [{\bfcurY}(t)] \, dt \, , \nonumber \\
  d[{\bfcurY}(t)] & =  [\curL([{\bfcurZ(t)}])]\, dt -\frac{1}{2} f_0\,[{\bfcurY}(t)]\, dt + \sqrt{f_0}\, [d{\bfcurW^\wien}(t)] \, ,  \nonumber
\end{align}
with $[\bfcurZ(0)] = [\eta_d]\, [a]$ and $[\bfcurY(0)] = [\bfcurN\,] \, [a]$, in which
\begin{equation}
[a]  = [g_\optp]\, ([g_\optp]^T\, [g_\optp])^{-1} \in \MM_{n_d,m_\optpp} \, . \nonumber
\end{equation}

\noindent (1) $[\curL([\bfcurZ(t)])]= [L ( [\bfcurZ(t)] \, [g_\optp]^T ) ] \, [a]$ is a random matrix  with values in $\MM_{\nu,m_\optpp}$.
For all $[u] = [\bfu^1 \ldots \bfu^{n_d}]$ in $\MM_{\nu,n_d}$ with $\bfu^j=(u^j_1,\ldots ,u^j_{\nu})$ in $\RR^{\nu}$, the
matrix $[L([u])]$ in $\MM_{\nu,n_d}$ is defined, for all $k = 1,\ldots ,\nu$ and for all $j=1,\ldots , n_d$, by
\begin{align}
[L([u])]_{kj} & = \frac{1}{p(\bfu^j)} \, \{{\boldsymbol{\nabla}}_{\!\!\bfu^j}\, p(\bfu^j) \}_k \, , \label{eqA8} \\                                            p(\bfu^j)     & = \frac{1}{n_d} \sum_{j'=1}^{n_d}
        \exp\{ -\frac{1}{2 {\hat s}^{\, 2}}\Vert \frac{\hat s}{s}\bfeta^{j'}-\bfu^j\Vert^2 \}  \, , \nonumber \\
{\boldsymbol{\nabla}}_{\!\!\bfu^j}\, p(\bfu^j) \!& = \! \frac{1}{\hat s^{\,2} \, n_d} \sum_{j'=1}^{n_d} (\frac{\hat s}{s}\bfeta^{j'} \!\! -\bfu^j)
              \,\exp\{ -\frac{1}{2 \hat s^{\,2}}\Vert
             \frac{\hat s}{s}\bfeta^{j'}\!\! -\bfu^j\Vert^2 \} \, ,  \nonumber
\end{align}
in which $\hat s$  is the modified Silverman bandwidth  $s$, which has been introduced in \cite{Soize2015},
\begin{equation}
\hat s =   \frac{s}{\sqrt{s^2 +\frac{n_d -1}{n_d}}} \quad , \quad
s = \left\{\frac{4}{n_d(2+\nu)} \right\}^{1/(\nu +4)}   \, .    \nonumber
\end{equation}

\noindent (2) $[\bfcurW^\wien(t)] = [\WW^\wien(t)] \, [a]$ where $\{[\WW^\wien(t)],$ $t\in \RR^+\}$ is the $\MM_{\nu,n_d}$-valued normalized Wiener process.

\noindent (3)  $[\bfcurN\,]$ is the $\MM_{\nu,n_d}$-valued normalized Gaussian random matrix that is independent of process $[\WW^\wien]$.

\noindent (4)  The free parameter $f_0$, such that $0 < f_0 < 4/\hat s$, allows the dissipation term of the nonlinear second-order dynamic system (dissipative Hamiltonian system)  to be controlled in order to kill the transient part induced by the initial conditions. A common value is $f_0=4$ (note that $\hat s< 1$).

\noindent (5)  We then have ${[\bfZ]} = \lim_{t\rightarrow +\infty} {[\bfcurZ(t)]}$ in probability distribution. The St\"{o}rmer-Verlet scheme is used \cite{Soize2016} for solving the reduced-order ISDE, which allows for generating the learned realizations, $[z_\ar^1], \ldots ,$ $[z_\ar^{n_\pMCH}]$, and then generating the learned realizations $[\eta_\ar^1],\ldots ,$ $[\eta_\ar^{n_\pMCH}]$ such that $[\eta_\ar^\ell] = [z_\ar^\ell]\, [g_\optp]^T$. The implementation of the St\"{o}rmer-Verlet scheme is detailed, for instance, in the Appendix of \cite{Soize2021a} for parallel computation, introducing the following parameters: the integration time step $\Delta t_\SV$, the initial time $t_i =0$, and the final integration time $t_f = M_0\times \Delta t_\SV$, at which the stationary solution is reached.

\noindent (6) The learned realizations $\{\bfx_\ar^{\ell'},\ell' =1,\ldots,n_\ar \}$ of random vector $\bfX$ are then calculated (see Eq.~\eqref{eqA1}) by $\bfx_\ar^{\ell'} = \underline \bfx + [\varphi]\, [\mu]^{1/2}$ $\bfeta_\ar^{\ell'}$.

\subsection{Constraints on the second-order moments of the components of $\bfH_\ar$}
\label{SectionA.6}
In general, the mean value of $\bfH_\ar$  estimated using the $n_\ar$ learned realizations
$\{\bfeta_\ar^{\ell'},\ell'=1,\ldots,n_\ar\}$, is sufficiently close to zero.
Likewise, the estimate of the covariance matrix of $\bfH_\ar$, which must be the identity matrix, is
sufficiently close to a diagonal matrix. However, sometimes  the
diagonal entries of the estimated covariance matrix can be lower than $1$.
Normalization can be recovered by imposing constraints
\begin{equation} \label{eqA9}
 E\{( H_{\ar,k} )^2 \} = 1 \,\, , \,\, k=1,\ldots, \nu \, ,
\end{equation}
in the algorithm presented in \ref{SectionA.5}.
For that, we use the method and the iterative algorithm presented in \cite{Soize2022a} (that is based on Sections~5.5 and 5.6 of \cite{Soize2020a}). The constraints are imposed by using the Kullback-Leibler minimum cross-entropy principle. The resulting optimization problem is formulated using a Lagrange multiplier $\bflambda = (\lambda_1,\ldots,\lambda_\nu)$ associated with the constraints. The optimal solution of the Lagrange multiplier is computed using an efficient iterative algorithm. At each iteration, the MCMC generator detailed in \ref{SectionA.5} is used.
The  constraints are rewritten as
\begin{equation} \label{eqA10}
E\{\bfh(\bfH_\ar)\} = \bfb \, ,
\end{equation}
in which the function $\bfh = (h_1,\ldots ,h_{\nu})$ and the vector $\bfb = (b_1,\ldots , b_{\nu})$ are such that
$h_k(\bfH_\ar) = (H_{\ar,k})^2$ and $b_k =1$ for $k$ in $\{1,\ldots , \nu\}$.
To take into account the constraints in the algorithm presented in \ref{SectionA.5}, Eq.~\eqref{eqA8} is replaced by the following one,
\begin{equation} \label{eqA11}
[L_\bflambda([u])]_{kj}  = \frac{1}{p(\bfu^j)} \, \{{\boldsymbol{\nabla}}_{\!\!\bfu^j}\, p(\bfu^j) \}_k  -2\, \lambda_k u_k^j\, .
\end{equation}
It should be noted that Eqs.~\eqref{eqA9} to \eqref{eqA11} can be straightforwardly extended to the case in which the constraint defined by
Eq.~\eqref{eqA9} is replaced by the full second-order moment constraints  $E\{\bfH_\ar\} = \bfzero_\nu$ and $E\{\bfH_\ar\otimes\bfH_\ar\} = [I_\nu]$, that is to say,
\begin{equation} \label{eqA12}
E\{ H_{\ar,k} \}  = 0 \, , \, E\{( H_{\ar,k} )^2\}  = 1 \,\, ,\,\, k=1,\ldots, \nu \quad , \quad
 E\{ H_{\ar,k} H_{\ar,k'}\}  = 0  \,\, , \,\, 1 \leq k < k' \leq \nu \, .
\end{equation}
The iteration algorithm for computing $\bflambda^{i+1}$ as a function of $\bflambda^i$ is the following,
\begin{align}\label{eqA13}
\bflambda^{i+1} & = \bflambda^i -\alpha_i [\Gamma''(\bflambda^i)]^{-1}\, \bfGamma'(\bflambda^i) \quad , \quad i \geq 0 \, ,\nonumber \\
\bflambda^0         & = \bfzero_\nu \, , \nonumber
\end{align}
in which $\bfGamma'(\bflambda^i) = \bfb - E\{\bfh(\bfH_{\bflambda^i})\}$ and
$[\Gamma''(\bflambda^i)] = [\cov\{\bfh(\bfH_{\bflambda^i})\}]$ (the covariance matrix), and where $\alpha_i$ is a relaxation function (less than $1$) that is introduced for controlling the convergence as a function of iteration number $i$.
For given $i_2 \geq 2$, for given $\beta_1$ and $\beta_2$ such that $0 < \beta_1 < \beta_2 \leq 1$, $\alpha_i$ can be defined by:

\noindent - for  $i \leq i_2$,  $\alpha_i =  \beta_1 +(\beta_2-\beta_1)(i-1)/(i_2-1)$;

\noindent - for $i > i_2$, $\alpha_i = \beta_2$.

\noindent The convergence of the iteration algorithm is controlled by the error function $i\mapsto \err(i)$  defined by
\begin{equation} \label{eqA13}
\err(i) =  \Vert \bfb - E\{\bfh(\bfH_{\bflambda^i})\} \Vert / \Vert \bfb\Vert \, .
\end{equation}
At each iteration $i$, $E\{\bfh(\bfH_{\bflambda^i})\}$ and $[\cov\{\bfh(\bfH_{\bflambda^i})\}]$ are estimated by using the $n_\ar$ learned realizations of $\bfH_{\bflambda^i}$ obtained by reshaping the learned realizations. If $i_{\rm{last}}$ is the last iteration corresponding to convergence, we have $\bfH_\ar =  \bfH_{\bflambda^i}$ with $i = i_{\rm{last}}$.
%
%
%
\section{Estimation of the Kullback-Leibler divergence, the mutual information, and the entropy from a set of realizations}
\label{SectionB}
The definition Kullback-Leibler divergence, the mutual information, and the entropy can be found in
\cite{Cover2006,Kapur1992,Shannon1948}. The estimation of these quantities from a set of independent realizations is carried out using the Gaussian kernel density estimation (GKDE) method \cite{Bowman1997,Gentle2009,Givens2013}.

For $\nu > 1$, let $\bfX=(X_1,\ldots ,X_\nu)$ and $\bfY=(Y_1,\ldots ,Y_\nu)$ be $\RR^\nu$-valued random variables defined on the probability space $(\Theta,\curT,\curP)$, whose probability measures are $P_\bfX(d\bfx) = p_\bfX(\bfx)\, d\bfx$ and $P_\bfY(d\bfy) = p_\bfY(\bfy)\, d\bfy$,
in which  the probability density functions $p_\bfX$ and $p_\bfY$ are assumed to be strictly positive. Let $\{\bfx^\ell,\ell=1,\ldots , N_x\}$ be $N_x$ independent realizations of $\bfX$ and let $\{\bfy^j,j=1,\ldots , N_y\}$ be $N_y$ independent realizations of $\bfY$.
For $k=1,\ldots , \nu$, let $\sigma_{X_k}$ and $\sigma_{Y_k}$ be the standard deviation of $X_k$ and $Y_k$ that are estimated (empirical estimator) with the independent realizations.
Finally, we introduce the Silverman bandwidth for the Gaussian KDE estimation of $p_\bfX$ and $p_\bfY$,
\begin{equation}\label{eqB1}
s_x = \left\{\frac{4}{N_x(2+\nu)} \right\}^{1/(\nu +4)}\quad , \quad s_y = \left\{\frac{4}{N_y(2+\nu)} \right\}^{1/(\nu +4)}\, .
\end{equation}

\subsection{Estimation of the Kullback-Leibler divergence from a set of realizations}
\label{SectionB.1}
The Kullback-Leibler divergence (or the relative entropy) between $p_\bfX$ and $p_\bfY$ is defined by
\begin{equation}\label{eqB2}
D(p_\bfX\Vert p_\bfY) = \int_{\RR^\nu} p_\bfX(\bfx)\,\log \left ( \frac{p_\bfX(\bfx)}{p_\bfY(\bfx)} \right) \, d\bfx =
E\left\{ \log \left( \frac{p_\bfX(\bfX)}{p_\bfY(\bfX)} \right ) \right\} \, .
\end{equation}
The GKDE, $\hat D(p_\bfX\Vert p_\bfY)$, of $D(p_\bfX\Vert p_\bfY)$ yields the formula,
\begin{align}\label{eqB3}
\hat D(p_\bfX\, \Vert\,  p_\bfY) = & \nu\log\left (\frac{s_y}{s_x}\right ) + \log\left (\frac{N_y}{N_x}\right )
            + \log\left (\frac{\sigma_{Y_1}\times\ldots \times\sigma_{Y_\nu}}{\sigma_{X_1}\times\ldots \times\sigma_{X_\nu}}\right ) \nonumber\\
            & + \frac{1}{N_x}
             \sum_{\ell'=1}^{N_x} \log \left \{
             \frac
             {\sum_{\ell=1}^{N_x}\exp\left ( -\frac{1}{2s_x^2} \sum_{k=1}^\nu ( \frac{x_k^{\ell'} - x_k^\ell}{\sigma_{X_k}} )^2\right )}
             {   \sum_{j=1}^{N_y}\exp\left ( -\frac{1}{2s_y^2} \sum_{k=1}^\nu ( \frac{x_k^{\ell'} - y_k^j}   {\sigma_{Y_k}} )^2\right )}
                                      \right \} \, .
\end{align}

\subsection{Estimation of the mutual information from a set of realizations}
\label{SectionB.2}
The mutual information $I(\bfX)$ of $\bfX$ allows to quantify the level of statistical dependencies of the components $X_1,\ldots, X_\nu$ of $\bfX=(X_1,\ldots ,X_\nu)$. Let $p_{X_k}$ be the pdf of real-valued random variable $X_k$,
\begin{equation}\label{eqB4}
p_{X_k}(x_k) = \int_{\RR^{\nu-1}} p_\bfX ( x_1,\ldots ,x_{k-1},x_k,x_{k+1},\ldots,x_\nu ) \,
dx_1 \ldots dx_{k-1} \, dx_{k+1}\ldots\,dx_\nu \, .
\end{equation}
The mutual information $I(\bfX)$ is defined by
\begin{equation}\label{eqB5}
I(\bfX) = D(p_\bfX\, \Vert \,\otimes_{k=1}^\nu p_{X_k}) = E\left\{ \log \left( \frac{p_\bfX(\bfX)}{p_{X_1}(X_1)\times \ldots \times p_{X_\nu}(X_\nu)} \right ) \right\} \, .
\end{equation}
Eq.~\eqref{eqB5} shows that, if the components $X_1,\ldots ,X_\nu$ are statistically independent, then $I(\bfX)=0$.
The GKDE, $\hat I(\bfX)$ of $I(\bfX)$ yields the formula,
\begin{equation}\label{eqB6}
\hat I(\bfX)  = \frac{1}{N_x}
             \sum_{\ell'=1}^{N_x} \log \left \{
             \frac
             {\frac{1}{N_x}\sum_{\ell=1}^{N_x}\exp \left ( -\frac{1}{2s_x^2} \sum_{k=1}^\nu ( \frac{x_k^{\ell'} - x_k^\ell}{\sigma_{X_k}} )^2\right )}
             { \prod_{k=1}^\nu \left ( \frac{1}{N_x} \sum_{\ell''=1}^{N_x}\exp \left ( -\frac{1}{2s_x^2} ( \frac{x_k^{\ell'} - x_k^{\ell''}}{\sigma_{X_k}})^2 \right ) \right )}
                                      \right \} \, .
\end{equation}

\subsection{Estimation of the entropy from a set of realizations}
\label{SectionB.3}
The entropy related to $p_\bfX$ is defined by
\begin{equation}\label{eqB7}
S_\bfX = -\int_{\RR^\nu} p_\bfX(\bfx)\,\log p_\bfX(\bfx) \, d\bfx =
- E\left\{ \log p_\bfX(\bfX)  \right\} \, .
\end{equation}
The GKDE, $\hat S_\bfX$, of $S_\bfX$ yields the formula,
\begin{align}\label{eqB8}
\hat S_\bfX = & \nu \log (s_x\!\!\sqrt{2\pi}\, ) +
             \log (\sigma_{X_1}\times\ldots \times\sigma_{X_\nu} ) \nonumber\\
            & - \frac{1}{N_x}
             \sum_{\ell'=1}^{N_x} \log \left \{  \frac{1}{N_x}
             \sum_{\ell=1}^{N_x}\exp\left ( -\frac{1}{2s_x^2} \sum_{k=1}^\nu ( \frac{x_k^{\ell'} - x_k^\ell}{\sigma_{X_k}} )^2\right )
                                      \right \} \, .
\end{align}
Since $J_S(N_x) = \nu \log (s_x\!\!\sqrt{2\pi}\, )$ is asymptotically for $N_x\rightarrow +\infty$ in $-\log(N_x)$, the entropy decreases when $N_x$ increases. \\

\section*{Conflict of interest}

The author declares that he has no conflict of interest.

%

\end{document}